\documentclass{article}
\usepackage{iclr2027_conference,times}
\usepackage[utf8]{inputenc}
\usepackage[T1]{fontenc}
\usepackage{amsmath,amssymb,mathtools,amsthm}
\usepackage{booktabs,multirow,array,tabularx}
\usepackage{graphicx,xcolor,microtype}
\usepackage{colortbl,tabularx}
\definecolor{qkrow}{HTML}{E5F2E9}
\NewDocumentEnvironment{qktabular}{m m +b}{%
  \sbox0{\begin{tabular}{#2}#3\end{tabular}}%
  \setlength{\tabcolsep}{\dimexpr\tabcolsep+(\textwidth-\wd0)/(#1*2)\relax}%
  \begin{tabular}{#2}#3\end{tabular}%
}{}
\usepackage[section]{placeins}
\usepackage{flafter,float}
\usepackage{algorithm,algpseudocode}
\usepackage{xurl}
\usepackage[hidelinks]{hyperref}
\usepackage[capitalize,noabbrev]{cleveref}
\definecolor{codeurl}{HTML}{8B2439}

\usepackage[textsize=tiny]{todonotes}

\DeclareRobustCommand{\algname}[1]{\textsc{#1}}
\newcommand{\ivan}[1]{\todo[inline]{\textbf{Ivan: }#1}}
\newcommand{\norm}[1]{\left\lVert #1\right\rVert}
\newcommand{\R}{\mathbb R}
\newcommand{\best}[1]{\ensuremath{\mathbf{#1}}}
\newcommand{\resultcell}{\textemdash}
\newtheorem{proposition}{Proposition}[section]
\newtheorem{corollary}[proposition]{Corollary}
\crefname{proposition}{Proposition}{Propositions}
\crefname{corollary}{Corollary}{Corollaries}

\newcolumntype{Y}{>{\raggedright\arraybackslash}X}

\title{\algname{QK-Wanda}: Coupling Queries and Keys\\for Unstructured Pruning}
\author{Ivan Ilin\\
KAUST\\
\texttt{ivan.ilin@kaust.edu.sa}
\And
Peter~Richt\'{a}rik\\
KAUST\\
\texttt{peter.richtarik@kaust.edu.sa}}
\iclrfinalcopy % Show authors and disable review line numbering.
\hypersetup{
 pdftitle={QK-Wanda: Coupling Queries and Keys for Unstructured Pruning},
 pdfauthor={Ivan Ilin and Peter Richt\'{a}rik},
 pdfsubject={Query-key reconstruction for unstructured language model pruning},
 pdfkeywords={language model pruning, causal attention, QK reconstruction, Wanda}
}
\usepackage{needspace}
\begin{document}

\maketitle
\lhead{Preprint} % Override the conference-publication header in the template.
\raggedbottom % Keep paragraph spacing natural beside fixed-position tables and figures.

\begin{abstract}
\algname{Wanda} \citep{sun2024wanda} prunes large language models by scoring weights independently within each linear projection, although queries and keys interact through dot products. We introduce \algname{QK-Wanda}, which scores query and key weights by their individual deletion costs under an unmasked pre-RoPE reconstruction objective. It augments Wanda scores with information from the opposite projection---keys for query weights, and queries for key weights---allowing both projections to share a pruning budget. Its closed-form scores require no gradients or weight updates; full pruning takes 1.3\% longer than Wanda on A100 and 3.1\% longer on H200 with the calibration used in our main experiments. We evaluate QK-only pruning across 15 models from TinyLlama, Llama~2, Llama~3, and Qwen2.5, spanning 0.5B--72B parameters. Relative to Wanda, \algname{QK-Wanda} reduces QK reconstruction error by an average of 60\% at 50\% sparsity and 45\% at 80\%. Downstream gains depend on the model. At 80\% sparsity on Llama~2 70B, WikiText-2 and C4 perplexity decrease by 20.3\% and 13.5\%, while mean zero-shot accuracy rises by 5.94 percentage points. Qwen2.5-72B also improves, but Llama-3.1-70B has substantially higher perplexity despite lower reconstruction error. These results show both the promise of coupled pruning criteria and the limits of local reconstruction as a predictor of model quality.

\end{abstract}

\section{Introduction}
\label{sec:intro}

Large language models (LLMs) perform well across language understanding, generation, and reasoning tasks, but storing and running their billions of parameters remains expensive \citep{touvron2023llama2,grattafiori2024llama3,qwen2024qwen25}. Model compression addresses this cost. Quantization reduces the bits per weight; pruning sets selected weights to zero. Both seek to preserve predictive quality. We focus on unstructured pruning \citep{hoefler2021sparsity}.

\algname{Wanda} selects weights using their magnitudes and calibration activations, without updating retained weights \citep{sun2024wanda}. We focus on this mask-only setting, rather than methods that also compensate for pruning errors, such as \algname{SparseGPT} and \algname{Thanos} \citep{frantar2023sparsegpt,ilin2025thanos}. Our question is whether accounting for query--key interactions can improve these scalar pruning decisions.

Wanda measures each deletion through its effect on an individual projection. In attention, however, a query perturbation changes dot products according to the keys it interacts with, and vice versa \citep{vaswani2017attention}. \algname{QK-Wanda} therefore scores deletions using a common reconstruction objective for all query--key products (\cref{eq:main_unmasked_loss}). The resulting query and key scores measure changes to the same objective, allowing a fixed pruning budget to be allocated across rows and between the two projections.

QK-Wanda changes only how query and key weights are scored, so this paper focuses on pruning query and key weights. Unless stated otherwise, all other weights stay dense and sparsity refers to the query and key weights alone. Across our checkpoints, these weights make up 6--17\% of the linear-layer weights in a transformer block and 4--16\% of all parameters (Appendix~\ref{app:parameter_scope}). The remaining projections can still be pruned with Wanda; \cref{sec:full_block} evaluates this broader pruning scope.

\Needspace{6\baselineskip}
\paragraph{Contributions.}
\begin{enumerate}
\item We derive closed-form scalar deletion costs from the QK reconstruction objective, including grouped-query attention and retained biases (\cref{eq:gqa_single_scores}). These exact individual costs define an additive mask-selection surrogate and allow query and key weights to share a pruning budget without gradients or updates to retained weights.
\item Across 15 checkpoints, we evaluate full QK reconstruction, perplexity, and zero-shot accuracy. The results show gains on models from the Llama~2 and Qwen2.5 families, alongside model- and metric-dependent outcomes on Llama~3.
\item We investigate variations of the QK reconstruction score incorporating causal masking, rotary position embeddings, and row centering of QK products. Combining these three choices gives eight variants, which we evaluate on six checkpoints at 50\% and 80\% QK sparsity. The extensions offer improvements in selected configurations, but no consistent downstream advantage over the base score (Appendix~\ref{app:factorial_selection}).
\end{enumerate}

{\raggedright The implementation and code for reproducing our experiments are available at \mbox{\textcolor{codeurl}{\url{https://github.com/vectozavr/qk-wanda}}}.\par}

\subsection{Notation}
All key notation used in this paper is summarized in a tabular form in Section~\ref{sec:notation}; see Table~\ref{tab:notation}.

\section{Background and Related Work}
\label{sec:background}

\paragraph{Weight importance and reconstruction.}
\algname{Optimal Brain Damage} uses diagonal second-order loss sensitivity; \algname{Optimal Brain Surgeon} includes compensating weight changes \citep{lecun1989obd,hassibi1992obs}. \algname{Optimal Brain Compression} develops layer-reconstruction methods for pruning and quantization \citep{frantar2022obc}. \algname{SparseGPT} scales compensated pruning to LLMs, while \algname{Thanos} solves for several deletions within a weight block together and adapts their allocation \citep{frantar2023sparsegpt,ilin2025thanos}. We study mask selection with surviving weights fixed.

\paragraph{Activation-aware scalar scores.}
\algname{Wanda} ranks weights by their magnitudes and input activation norms \citep{sun2024wanda}. We recall its interpretation as the exact reconstruction cost of deleting a single weight while keeping all others fixed.

Let \(W\in\R^{d_{\mathrm{out}}\times d}\) be a weight matrix and \(X\in\R^{d\times T}\) its calibration inputs, with one token per column. The entry \(W_{ij}\) is the element in row \(i\) and column \(j\) of \(W\). The row vector \(X_{j:}\in\R^{1\times T}\) contains input feature \(j\) across all \(T\) tokens. The linear projection produces \(O=WX\in\R^{d_{\mathrm{out}}\times T}\). After pruning, its output is \(\widehat O=\widehat W X\), and we measure the reconstruction loss
\[
 \mathcal L_W=\norm{\widehat O-O}_F^2
 =\norm{(\widehat W-W)X}_F^2
 =\norm{\Delta X}_F^2,
\]
where \(\Delta=\widehat W-W\in\R^{d_{\mathrm{out}}\times d}\) is the change in weights and \(\norm{\cdot}_F\) is the Frobenius norm. Any retained bias cancels from the output difference.

Suppose we delete only \(W_{ij}\), without changing any other weight. Then \(\Delta_{ij}=-W_{ij}\), and every other entry of \(\Delta\) is zero. Only output row \(i\) changes, by \((\widehat O-O)_{i:}=-W_{ij}X_{j:}\). Its squared reconstruction error gives Wanda's score in squared form:
\begin{equation}
 S^{W}_{ij}=\norm{-W_{ij}X_{j:}}_2^2=W_{ij}^2\norm{X_{j:}}_2^2,
 \label{eq:wanda}
\end{equation}
where \(\norm{\cdot}_2\) is the Euclidean norm and the superscript \(W\) labels Wanda. Squaring its original nonnegative score preserves the ranking. Selecting the smallest score solves the single-deletion reconstruction problem exactly. For multiple deletions, their effects can interact, so individual deletion costs do not generally add up to the joint reconstruction loss. Wanda uses these costs greedily, removing the lowest-scoring weights within each output row until its pruning budget is met. With fixed inputs, deleting one weight leaves the Wanda scores of the remaining weights unchanged.

\algname{RIA} combines relative weight importance with activations and uses channel permutation for N:M sparsity \citep{zhang2024ria}. \algname{SymWanda} adds input- and output-side reconstruction terms for a single matrix \citep{yi2025symmetric}. Our objective instead couples QK projections through their output dot products over all token pairs, making their deletion costs comparable.

\paragraph{Attention-aware quantization.}
\algname{AESPA} uses QK reconstruction quadratics for quantization; \algname{BoA} also accounts for RoPE \citep{kim2024aespa,kim2025boa}. With one projection fixed, the unmasked loss has the form \(\delta^\top G\delta\), where \(\delta\) vectorizes the other projection's perturbation. Deleting scalar weight \(W_p\) gives \(\delta=-W_p e_p\) and cost \(W_p^2G_{pp}\). Our base score is this diagonal specialization of the QK reconstruction geometry used by AESPA. We retain sequence-local products of activation energies, rather than AESPA's product-of-expectations approximation. The contribution is its compensation-free elementwise pruning implementation, GQA-aware statistics, shared QK allocation, and controlled evaluation. Appendix~\ref{app:prior_geometry} details the connection.

\paragraph{Structured QK and attention compression.}
\algname{ThinK} prunes key-cache channels using query-dependent contribution norms \citep{xu2025think}. Its full-sequence channel criterion is the product of the norms of the corresponding query and key channels; its implementation uses a recent query window. This anticipates the opposite-projection energy principle, but removes cache channels rather than scalar projection weights. \algname{CORP} ranks matched QK dimensions by expected contribution energy and applies closed-form compensation \citep{zhang2026corp}. \algname{SNP} selects paired QK neurons using singular components of attention scores \citep{shim2024snp}, while \algname{SparseLoRA} uses a QK norm criterion for channel selection during fine-tuning \citep{khaki2025sparselora}.

Other approaches change the representation or reconstruction target. \citet{liang2025beyond} optimize masks through nonlinear attention reconstruction, including separate QK masks in their LLM experiment. We evaluate closed-form deletions under an unmasked pre-RoPE objective. \algname{CLOVER} prunes orthogonal directions, \algname{\(A^3\)} derives low-rank attention approximations, and \algname{RAP} removes RoPE-aligned channel pairs for KV-cache compression \citep{meng2025clover,wong2025a3,xin2026rap}. \algname{LLM-Pruner}, \algname{SliceGPT}, and \algname{FASP} provide broader structural alternatives; \algname{DIMAP} develops a data-independent module-aware criterion for vision transformers \citep{ma2023llmpruner,ashkboos2024slicegpt,hu2025fasp,he2024dimap}.

\paragraph{Optimizing beyond fixed-weight mask selection.}
\algname{Wanda++} uses regional gradients and decoder-block optimization \citep{yang2025wandaplus}; \algname{SparseLLM} coordinates global pruning subproblems \citep{bai2024sparsellm}. \algname{SNOWS} optimizes deeper nonlinear representations and readjusts retained weights in vision networks \citep{lucas2025snows}. Our experiments isolate forward-only QK mask selection with surviving weights fixed; they do not establish superiority over these broader optimization procedures.

\paragraph{Allocation, calibration, and attention semantics.}
\algname{OWL} studies nonuniform layerwise sparsity \citep{yin2024owl}, while \citet{ji2025calibration} examine sensitivity to calibration data. We study budget allocation between Q and K within a block and vary calibration size and sequence length. Causal masking, shared key heads, and rotary positions affect which QK interactions a score measures \citep{vaswani2017attention,ainslie2023gqa,su2024roformer}. 

Other quantization approaches reconstruct individual layers or optimize compressed representations, including \algname{GPTQ}, \algname{AQLM}, \algname{PV-Tuning}, \algname{HIGGS}, and \algname{QTIP} \citep{frantar2023gptq,egiazarian2024aqlm,malinovskii2024pv,malinovskii2024higgs,tseng2024qtip}.

\section{\algname{QK-Wanda}}
\label{sec:method}

QK-Wanda follows the same single-deletion principle as Wanda, but reconstructs query--key products instead of each projection separately. We first derive the exact cost of deleting one weight while leaving all others fixed, then use these costs greedily to select weights under separate or shared QK budgets. The derivation starts with an unmasked single-head example before extending the same objective to shared key heads.

\begin{figure}[!t]
 \centering
 \includegraphics[width=\linewidth]{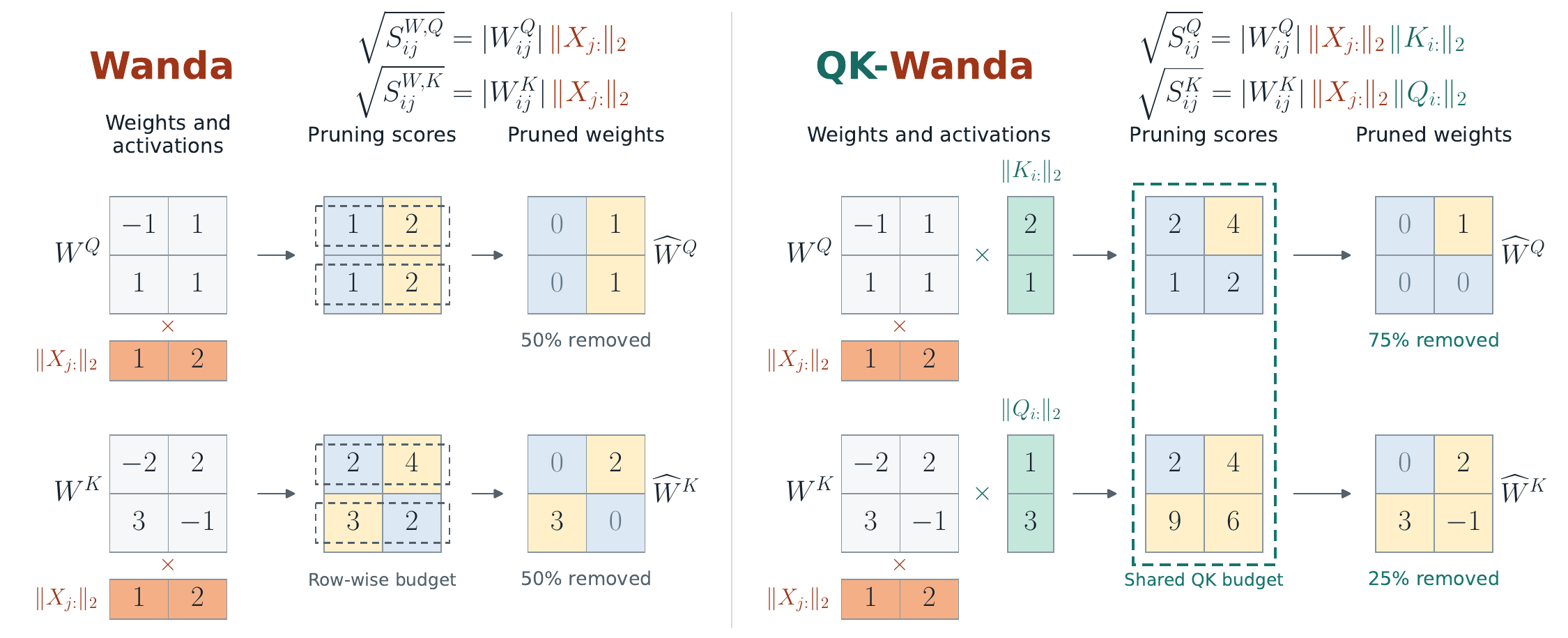}
 \caption{Wanda versus QK-Wanda for one token, \(X=(1,2)^\top\) (\(d=2,T=1\)). Orange/green factors are input/opposite-projection norms; \(\times\) denotes elementwise multiplication with row/column broadcasting. Dashed boxes mark budget pools. Both prune 50\% overall: Wanda removes 50\% per row; QK-Wanda removes 75\% of Q and 25\% of K. Hats denote pruned weights.}
 \label{fig:wanda_qk_illustration}
\end{figure}

\subsection{Single-head derivation: unmasked QK reconstruction}
\label{sec:unmasked_core}

We work within one fixed transformer block. First consider one query head and its own key head, as in ordinary multi-head attention (MHA), using one calibration sequence and omitting the causal mask. This unmasked, pre-RoPE case, \algname{QK-Wanda}, makes the coupling especially simple: a query weight matters through the keys, and a key weight through the queries. \Cref{sec:gqa_method} extends the same objective to several query heads sharing a key head, as in grouped-query attention (GQA).

Let \(X\in\R^{d\times T}\) contain the inputs to the QK linear projections, after any preceding normalization, with one token per column. Let \(W^Q,W^K\in\R^{m\times d}\) be the query/key projection weights, with head width \(m\). Biases \(b^Q,b^K\in\R^{m\times1}\) remain fixed and are zero for bias-free models. With the all-ones column vector \(e\in\{1\}^{T\times1}\), the pre-RoPE activations are

\begin{equation}
 Q=W^{Q}X+b^{Q}e^\top,
 \qquad K=W^{K}X+b^{K}e^\top.
 \label{eq:simple_projections}
\end{equation}
Here \(Q,K\in\mathbb{R}^{m\times T}\); the row \(Q_{i:}\) or \(K_{i:}\) contains coordinate \(i\) across all tokens.

Write \(\widehat Q,\widehat K\) for the outputs of the query/key projections after pruning \(W^Q,W^K\), evaluated on the same input \(X\) with unchanged biases. For the unmasked, pre-RoPE case, define
\begin{equation}
 \mathcal L=\norm{\widehat Q^\top\widehat K-Q^\top K}_F^2.
 \label{eq:simple_unmasked_loss}
\end{equation}
Every token pair contributes, and inputs are fixed during a pruning decision. Attention divides its dot products by \(\sqrt{m}\) before softmax. We omit this scale from the pruning losses and scores derived below: including it would divide every deletion cost by \(m\), leaving the ranking unchanged.

\paragraph{Deleting one query weight.}
Set \(W^{Q}_{ij}\) to zero, keeping keys, biases, and every other weight fixed. The query bias cancels in the output difference; only query row \(i\) changes, with \((\widehat Q-Q)_{i:}\overset{\eqref{eq:simple_projections}}{=}-W^{Q}_{ij}X_{j:}\). Therefore
\begin{equation}
 \widehat Q^\top K-Q^\top K
 =(\widehat Q-Q)^\top K\overset{\eqref{eq:simple_projections}}{=}-W^{Q}_{ij}X_{j:}^\top K_{i:}.
 \label{eq:simple_query_delta}
\end{equation}
This is an outer product, so \(\norm{xy^\top}_F^2=\norm{x}_2^2\norm{y}_2^2\) gives its squared reconstruction cost. By symmetry, deleting \(W^K_{ij}\) with queries fixed yields \(Q^\top(\widehat K-K)\overset{\eqref{eq:simple_projections}}{=}-W^K_{ij}Q_{i:}^\top X_{j:}\). \begin{samepage}
The two exact single-deletion scores are
\begin{equation}
 \boxed{S^Q_{ij}=\underbrace{(W^{Q}_{ij})^2\norm{X_{j:}}_2^2}_{\text{\algname{Wanda}}}\norm{K_{i:}}_2^2,
 \qquad S^K_{ij}=\underbrace{(W^{K}_{ij})^2\norm{X_{j:}}_2^2}_{\text{\algname{Wanda}}}\norm{Q_{i:}}_2^2.}
 \label{eq:simple_wanda_connection}
\end{equation}
\end{samepage}
Each score multiplies the squared Wanda score by the opposite projection's coordinate energy; smaller values indicate less reconstruction damage. These are finite-deletion identities, without a small-change approximation. A bias cancels from its own perturbation but remains in the opposite activation: the key bias affects the query score through \(K\), and the query bias affects the key score through \(Q\).

\Cref{fig:wanda_qk_illustration} illustrates the additional opposite-projection factor and contrasts Wanda's fixed row quotas with a shared QK budget. It shows how the scoring and budget pools change the selected mask; individual-deletion rankings do not guarantee optimal joint reconstruction.

QK-Wanda rescales Wanda's scores within each row by the same factor. For a nonzero opposite-projection activation row, this leaves the ranking unchanged, so identical row budgets and tie-breaking rules produce identical pruning masks. The additional factor therefore affects comparisons across rows and between Q and K. Causal weighting and averaging over independent sequences can also change within-row rankings; Appendix~\ref{app:row_equivalence} gives examples and the zero-energy exception.

\subsection{Extension to grouped-query attention}
\label{sec:gqa_method}

We now extend the same unmasked reconstruction objective \eqref{eq:simple_unmasked_loss} to grouped-query attention (GQA), where several query heads share one key head. GQA reduces KV-cache memory and bandwidth during decoding \citep{ainslie2023gqa}. For example, Llama~3.1-8B uses 32 query heads and 8 key/value heads, so each group of four query heads shares one key head and one value head \citep{grattafiori2024llama3}.

Consider one transformer block and one calibration sequence. Query heads have indices \(h=1,\ldots,H_Q\), and key heads \(g=1,\ldots,H_K\). The set \(\mathcal H_g\) contains the query heads sharing key head \(g\); these sets partition the query heads. For example, \(\mathcal H_1=\{1,2\}\) means that \(Q^{(1)}\) and \(Q^{(2)}\) both interact with \(K^{(1)}\). Each head has width \(m\). The head weights \(W^{Q^{(h)}},W^{K^{(g)}}\in\R^{m\times d}\) produce activations \(Q^{(h)},K^{(g)}\in\R^{m\times T}\) using the affine maps and retained biases in \cref{eq:simple_projections}. Hatted activations use the pruned weights on the same input \(X\).

The unmasked reconstruction objective in \cref{eq:simple_unmasked_loss}, summed over all interacting query--key head pairs, becomes
\begin{equation}
 \mathcal L=\sum_{g=1}^{H_K}\sum_{h\in\mathcal H_g}
 \norm{\widehat Q^{(h)\top}\widehat K^{(g)}-Q^{(h)\top}K^{(g)}}_F^2.
 \label{eq:main_unmasked_loss}
\end{equation}
Every token pair contributes within each head pair. Deleting a query weight changes one query head; deleting a key weight changes every product involving that key head. For \(h\in\mathcal H_g\), the outer-product argument in \cref{sec:unmasked_core} therefore gives
\begin{equation}
 \boxed{S^{Q^{(h)}}_{ij}=(W^{Q^{(h)}}_{ij})^2\norm{X_{j:}}_2^2\norm{K^{(g)}_{i:}}_2^2,
 \quad S^{K^{(g)}}_{ij}=(W^{K^{(g)}}_{ij})^2\norm{X_{j:}}_2^2
 \textstyle\sum_{h\in\mathcal H_g}\norm{Q^{(h)}_{i:}}_2^2.}
 \label{eq:gqa_single_scores}
\end{equation}
The query score uses the energy of its associated key head. The key score sums query energies over the entire group, because removing one key weight perturbs its products with all those query heads. These are the same exact single-deletion costs as in \cref{sec:unmasked_core}, now with explicit head indices; GQA does not introduce a different scoring rule.

The scores \eqref{eq:simple_wanda_connection} for ordinary multi-head attention are the special case of \eqref{eq:gqa_single_scores} with \(H_Q=H_K\) and \(\mathcal H_g=\{g\}\). The key-head sum then has one term, and suppressing head labels recovers \eqref{eq:simple_wanda_connection} exactly.

For \(N\) independent calibration sequences, compute each deletion cost separately and average with equal weight \(1/N\). In particular, average the \emph{products} of input and opposite-projection energies within sequences, rather than multiplying separately averaged energies. No product connects tokens from different sequences. Appendix~\ref{app:unmasked} gives the full formulas and proof. Causal masking and RoPE are optional extensions, derived in Appendices~\ref{app:causal_details} and~\ref{app:rope}; the base scores above use neither.

\subsection{Mask construction and computational cost}
\label{sec:mask_construction}

In the single-sequence setting of \cref{sec:unmasked_core}, Wanda and QK-Wanda produce identical pruning masks under the same row budgets and tie-breaking rules, provided the opposite-projection factors are nonzero. The additional QK factor only rescales scores within each row, so we relax the row quotas to let it affect which weights are removed. We therefore select weights across each projection matrix independently (\emph{separate QK budgets}) or across both projections together (a \emph{shared QK budget}), as illustrated in \cref{fig:main_llama2_70b}. The shared allocation is meaningful because Q and K scores measure changes to the same reconstruction loss \eqref{eq:main_unmasked_loss}.

To distinguish the effect of QK-aware scoring from that of relaxing row quotas, \cref{fig:main_llama2_70b} also includes Wanda with a matrix-wide budget for each projection, matching the separate QK-Wanda allocation. Wanda scores measure reconstruction errors of the Q and K projections separately; their relative scale is not calibrated to a common QK reconstruction objective. We therefore use separate matrix budgets for this control rather than pooling Wanda scores into a shared QK budget.

QK-Wanda scores (\ref{eq:simple_wanda_connection}, \ref{eq:gqa_single_scores}) share a common scale because they measure deletion costs under the same QK reconstruction loss. Multiplying either entire projection, including its weights and bias, by a nonzero constant \(a\), while holding the other fixed, multiplies both query and key scores by \(a^2\). Their relative scale and the ranking in a shared QK pool therefore stay unchanged. Wanda's squared scores scale by \(a^2\) only for the rescaled projection; the other projection's scores stay unchanged. Pooling Wanda scores would thus require a choice of relative scaling between projections, whereas QK-Wanda needs no such calibration.

\paragraph{What thresholding optimizes.}
Let \(\mathcal P\) index the weights competing under one budget, \(S_p\) denote their single-deletion scores, and \(z_p\in\{0,1\}\) indicate whether weight \(p\) is removed (\(z_p=1\)) or retained (\(z_p=0\)). For an integer pruning budget \(M\in\{0,\ldots,|\mathcal P|\}\), selecting the \(M\) lowest scores solves
\begin{equation}
 \min_{z\in\{0,1\}^{|\mathcal P|}}
 \underbrace{\sum_{p\in\mathcal P} S_p z_p}_{\mathcal L_{\mathrm{add}}(z)}
 \qquad\text{subject to}\qquad
 \sum_{p\in\mathcal P}z_p=M.
 \label{eq:additive_mask_selection}
\end{equation}
Thus thresholding minimizes an additive surrogate constructed from exact individual-deletion costs. It does not generally minimize the coupled reconstruction loss: multiple deletions introduce interactions within each projection and between Q and K. Appendix~\ref{app:causal_quadratic} gives the full expansion; retaining all token pairs.

Let \(\rho\in[0,1]\) be the requested fraction to remove. A shared budget uses one pool \(\mathcal P\) containing all query and key weights of the current block. Separate budgets solve \cref{eq:additive_mask_selection} independently for each projection, including all its heads. Row-wise budgets select weights independently within each output row. Neither separate nor shared budgets prescribe how many weights must be removed from each head (Appendix~\ref{app:proof_c}). Further implementation details are given in Appendix~\ref{app:implementation}.

\Cref{alg:qk_wanda} summarizes sequential pruning with these sequence-averaged scores. Each block's QK masks are selected jointly before either projection is changed.

\begin{algorithm}[H]
\caption{\algname{QK-Wanda} with a shared QK budget}
\label{alg:qk_wanda}
\begin{algorithmic}[1]
\State Input: model with \(D\) transformer blocks, \(N\) calibration sequences, sparsity \(\rho\in[0,1]\).
\For{\(\ell=1,\ldots,D\)}
  \State Collect projection inputs: \(X^{(s)}\), \(s=1,\ldots,N\).
  \State Compute activations: \(Q^{(h)},K^{(g)}\) by \eqref{eq:simple_projections} for each calibration sequence.
  \State Compute scores: \(S^{Q^{(h)}},S^{K^{(g)}}\) by \eqref{eq:gqa_single_scores}, averaged over sequences.
  \State Let \(\mathcal P\) index all QK weights \(W_p\) of block \(\ell\).
  \State \(M\gets\lceil\rho|\mathcal P|\rceil\); select \((z_p)_{p\in\mathcal P}\) by \eqref{eq:additive_mask_selection}.
  \State Prune: \(W_p\gets(1-z_p)W_p\) for every QK weight \(W_p\).
  \State Propagate calibration sequences through the pruned block.
\EndFor
\end{algorithmic}
\end{algorithm}

\paragraph{Computational cost.}
Let \(D_Q=H_Qm\) and \(D_K=H_Km\) be the total query and key widths. For \(N\) calibration sequences of length \(T\), a forward pass through the QK projections costs \(O(NTd(D_Q+D_K))\). The QK scores require only per-sequence input and head-coordinate energies, costing \(O(NT(d+D_Q+D_K))\) to accumulate. Combining these energies across sequences costs \(O(NdD_K)\), with query-side factors reused by heads sharing a key; applying the factors to all weights costs \(O(d(D_Q+D_K))\). These additional terms do not change the asymptotic order of QK projection and scoring. Stable mask selection additionally sorts \(P=d(D_Q+D_K)\) scores, with \(O(P\log P)\) comparison-sort complexity for a shared pool. No \(T\times T\) product matrix, gradients, or updates to retained weights are needed.

To measure the practical overhead, we benchmark Wanda and QK-Wanda across four Llama~3 model sizes on A100 and H200 GPUs using the same 524,288 calibration tokens, sequence length 64, and batch size 32 as the main experiments. We include calibration forwards, scoring, GPU mask selection, weight zeroing and propagation through the pruned blocks. QK-Wanda takes 1.3\% longer than Wanda on average on A100 and 3.1\% longer on H200. Table~\ref{tab:method_cost} reports the H200 measurements; Appendix~\ref{app:score_cost} gives results for both GPUs and the timing protocol.

\begin{table}[H]
\centering\small
\setlength{\tabcolsep}{2pt}
\caption{Full QK pruning time on one H200: \(8192\times64\) calibration tokens, batch 32. Includes calibration, scoring, mask selection, zeroing and pruned propagation; excludes loading and setup transfers. Medians of five runs; complexity covers QK projections and scoring per block.}
\label{tab:method_cost}
\renewcommand{\arraystretch}{1.15}
\begin{tabular*}{\linewidth}{@{\extracolsep{\fill}}lcccccc@{}}
\toprule
Method&Complexity&Llama-3.2~1B&Llama-3.2~3B&Llama-3.1~8B&Llama-3.1~70B&Avg.\ ratio\\
\midrule
\algname{Wanda}&\(O(NTd(D_Q+D_K))\)&6.59~s&17.02~s&34.57~s&284.37~s&1.00\\
\rowcolor{qkrow}
\algname{QK-Wanda}&\(O(NTd(D_Q+D_K))\)&7.12~s&17.45~s&34.98~s&285.98~s&1.03\\
\bottomrule
\end{tabular*}
\end{table}

\section{Experiments}
\label{sec:setup}
\label{sec:results}

Our experiments test whether preserving local QK products also preserves the quality of the complete pruned model. We compare reconstruction error, perplexity, and zero-shot accuracy across checkpoints, and use controlled ablations to distinguish the effects of scoring and budget allocation.

\begin{figure}[!t]
\centering
\includegraphics[width=\textwidth]{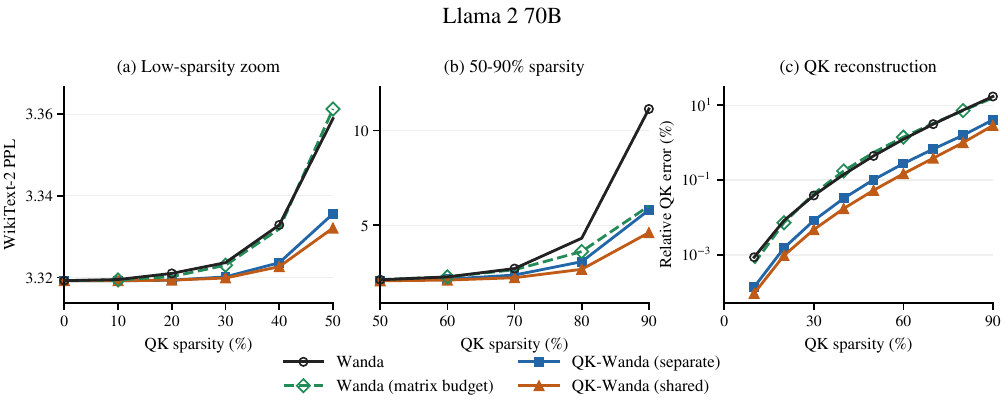}
\caption{QK-only pruning of Llama-2-70B. Left and center: WikiText-2 PPL at low and high sparsity. Right: mean relative QK reconstruction error (\cref{eq:relative_qk_error}). Matrix-budget Wanda and separate-budget QK-Wanda remove the same number of weights from each projection. C4 results: \cref{fig:qk_budget_llama2_70b_c4}; additional model curves: Appendices~\ref{app:budget_curves} and~\ref{app:qk_error_curves}.}
\label{fig:main_llama2_70b}
\end{figure}

\subsection{Experimental setup}

\paragraph{Models and pruning.}
We evaluated 15 base checkpoints: TinyLlama-1.1B; Llama-2-7B/13B/70B; Llama-3.2-1B/3B and Llama-3.1-8B/70B; and Qwen2.5 0.5B/1.5B/3B/7B/14B/32B/72B \citep{zhang2024tinyllama,touvron2023llama2,grattafiori2024llama3,meta2024llama32,qwen2024qwen25}. Llama-2-7B/13B use multi-head attention (MHA); the other checkpoints use grouped-query attention (GQA). Unless stated otherwise, sparsity refers only to \(W^Q,W^K\); all other parameters remain dense. QK weights comprise 4.46--16.11\% of total model parameters across these checkpoints, so 50\% QK sparsity removes 2.23--8.06\% of the total (Appendix~\ref{app:parameter_scope}). Our main tables compare row-wise \algname{Wanda} with \algname{QK-Wanda}, which shares a pruning budget between query and key weights in each block. The curves additionally show \algname{QK-Wanda} with a separate budget for each projection. On the three largest checkpoints, we also tested Wanda with a separate matrix-wide budget for each projection. All methods used sequential, one-shot unstructured pruning without retraining or survivor updates; biases and the original attention computation remained fixed. Integer budgets and reproducibility details are in Appendix~\ref{app:protocol}.

\paragraph{Calibration.}
The main comparisons used calibration seed 0 and 524,288 C4 training tokens \citep{raffel2020c4}: 256 windows of length 2048, segmented into 8192 independent sequences of length 64. We chose this configuration based on the calibration-sensitivity studies in Appendix~\ref{app:calibration}. Increasing the number of windows beyond 256 brought negligible PPL improvements for both Wanda and QK-Wanda. With the total token count fixed, we then tested different sequence lengths on Llama-3.1-8B and Llama-2-7B. Length 64 gave the lowest PPL on Llama-3.1-8B and nearly matched the minima at length 128 on Llama-2-7B, with differences no larger than 0.001 PPL. These small differences across lengths and consistent results across the two model families supported using 256 base windows segmented into length-64 sequences as our common calibration setting. Each checkpoint used its native tokenizer and identical sampled tokens across methods and sparsities. We used FP16 weights except for Qwen2.5-72B, whose dense and pruned models used BF16 to avoid nonfinite FP16 activations. Hardware and software details are in Appendix~\ref{app:implementation}.

\paragraph{Metrics and comparisons.}
We evaluate pruning at two levels. Held-out QK reconstruction error measures how well each block preserves QK products on the same dense-model inputs. Perplexity on WikiText-2 \citep{merity2017wikitext} and C4 and mean accuracy across seven zero-shot tasks measure the quality of the complete pruned model. Comparing these metrics tests whether better local reconstruction translates into downstream gains.

We measure QK reconstruction error on 16 held-out C4 sequences of length 512 and average it across blocks (\cref{eq:relative_qk_error}). This diagnostic differs from pruning calibration, where later blocks receive method-specific sparse-network inputs. A block's 1\% relative squared error corresponds to 10\% relative Frobenius-norm error; the reported average over blocks is taken before any square root. We measured WikiText-2 test and C4 validation PPL at context 2048 over 0--90\% sparsity in 10-point steps. At 30\%, 50\%, and 80\%, we also evaluated saved masks on seven zero-shot tasks, reporting their unweighted mean accuracy with LM Evaluation Harness \citep{evalharness}; tasks, splits, and metrics are specified in Appendix~\ref{app:zero_shot_protocol}. Dense references are shared within each checkpoint. All differences, relative reductions, counts, and bold marks use unrounded measurements; values that tie at the displayed precision can differ in the source data. The 15-model counts compare shared QK-Wanda against original row-wise Wanda, excluding dense points: 135 checkpoint--sparsity pairs per PPL/error metric and 45 for mean accuracy. They describe one calibration seed, not independent statistical trials or statistically significant wins.

% Generated by export_unmasked_main_results.py from audited exact values.
\begin{table}[!t]
\centering
\fontsize{8}{9}\selectfont
\setlength{\tabcolsep}{1.0pt}
\caption{QK-only pruning of the largest checkpoints: WikiText-2/C4 PPL ($\downarrow$), seven-task mean zero-shot accuracy ($\uparrow$), and QK reconstruction error ($\downarrow$). Accuracy and error are percentages. Wanda uses row budgets, Wanda (matrix) separate projection budgets, and both QK-Wanda variants shared budgets. Bold marks the best result at each sparsity, using unrounded values. Full model and task results appear in Appendices~\ref{app:model_results}, \ref{app:zero_shot_tables}, and~\ref{app:matrix_control}.}
\label{tab:large_models}
\begin{tabular*}{\textwidth}{@{\extracolsep{\fill}}lc*{12}{c}@{}}
\toprule
 & & \multicolumn{4}{c}{\textbf{Llama-2-70B}} & \multicolumn{4}{c}{\textbf{Llama-3.1-70B}} & \multicolumn{4}{c}{\textbf{Qwen2.5-72B}}\\
\cmidrule(lr){3-6}\cmidrule(lr){7-10}\cmidrule(lr){11-14}
Method & Sparsity & \shortstack{Wiki2\\PPL} & \shortstack{C4\\PPL} & \shortstack{Zero-shot\\acc.} & \shortstack{Rel. QK\\error} & \shortstack{Wiki2\\PPL} & \shortstack{C4\\PPL} & \shortstack{Zero-shot\\acc.} & \shortstack{Rel. QK\\error} & \shortstack{Wiki2\\PPL} & \shortstack{C4\\PPL} & \shortstack{Zero-shot\\acc.} & \shortstack{Rel. QK\\error}\\
\midrule
Dense & 0\% & 3.32 & 5.71 & 71.50 & 0.00 & 2.81 & 7.11 & 74.21 & 0.00 & 3.87 & 9.26 & 74.75 & 0.00\\
\midrule
\algname{Wanda} & 30\% & 3.32 & 5.71 & 71.50 & 0.04 & \textbf{2.82} & 7.12 & 74.16 & 0.04 & 3.88 & 9.26 & 74.68 & 0.04\\
\algname{Wanda} (matrix) & 30\% & 3.32 & 5.71 & 71.49 & 0.04 & 2.86 & 7.15 & 73.95 & 0.18 & 3.89 & 9.26 & \textbf{75.02} & 0.05\\
\rowcolor{qkrow}
\algname{QK-Wanda} & 30\% & \textbf{3.32} & \textbf{5.71} & 71.54 & \textbf{0.00} & 2.82 & \textbf{7.12} & \textbf{74.23} & \textbf{0.02} & \textbf{3.88} & \textbf{9.26} & 74.66 & \textbf{0.01}\\
\algname{QK-Wanda-MRC} & 30\% & 3.32 & 5.71 & \textbf{71.65} & 0.01 & 2.83 & 7.14 & 74.07 & 0.08 & 3.88 & 9.26 & 74.64 & 0.02\\
\addlinespace[2pt]
\algname{Wanda} & 50\% & 3.36 & 5.74 & 71.37 & 0.44 & \textbf{2.90} & \textbf{7.22} & 73.60 & 0.39 & 3.92 & 9.28 & \textbf{74.72} & 0.19\\
\algname{Wanda} (matrix) & 50\% & 3.36 & 5.74 & 71.80 & 0.54 & 3.35 & 7.61 & 72.54 & 1.65 & 3.92 & 9.29 & 74.59 & 0.29\\
\rowcolor{qkrow}
\algname{QK-Wanda} & 50\% & \textbf{3.33} & \textbf{5.72} & 71.60 & \textbf{0.05} & 2.97 & 7.27 & \textbf{73.77} & \textbf{0.18} & \textbf{3.89} & \textbf{9.27} & 74.44 & \textbf{0.08}\\
\algname{QK-Wanda-MRC} & 50\% & 3.34 & 5.73 & \textbf{71.99} & 0.19 & 3.14 & 7.47 & 73.66 & 0.81 & 3.90 & 9.28 & 74.53 & 0.13\\
\addlinespace[2pt]
\algname{Wanda} & 80\% & 4.55 & 6.91 & 65.85 & 7.40 & \textbf{4.98} & \textbf{9.89} & \textbf{66.64} & 5.27 & 4.55 & 9.87 & 73.44 & 1.30\\
\algname{Wanda} (matrix) & 80\% & 4.13 & 6.53 & 70.10 & 7.19 & 17.69 & 24.21 & 57.74 & 8.85 & 4.31 & 9.64 & 73.22 & 1.73\\
\rowcolor{qkrow}
\algname{QK-Wanda} & 80\% & \textbf{3.63} & \textbf{5.97} & \textbf{71.79} & \textbf{1.00} & 18.89 & 32.59 & 65.93 & \textbf{3.27} & 4.23 & \textbf{9.52} & 74.00 & \textbf{0.62}\\
\algname{QK-Wanda-MRC} & 80\% & 3.85 & 6.27 & 70.84 & 3.36 & 13.67 & 22.61 & 62.61 & 5.34 & \textbf{4.19} & 9.53 & \textbf{74.08} & 0.96\\
\bottomrule
\end{tabular*}
\end{table}

\subsection{Largest checkpoints across model families}
\label{sec:largest_models}

At 80\% QK sparsity, \algname{QK-Wanda} substantially narrows the quality gap to the dense Llama-2-70B model (\cref{tab:large_models}). WikiText-2/C4 PPL improves from 4.55/6.91 with Wanda to 3.63/5.97, compared with 3.32/5.71 for the dense checkpoint. Mean zero-shot accuracy rises from 65.85\% to 71.79\%, slightly above the dense reference of 71.50\%; we do not establish a statistically significant improvement over dense. Qwen2.5-72B also improves on both PPL metrics and mean accuracy. Llama-3.1-70B shows the limitation: QK error decreases from 5.27\% to 3.27\%, but PPL increases from 4.98/9.89 to 18.89/32.59. Mean zero-shot accuracy, however, changes from 66.64\% with Wanda to 65.93\% with QK-Wanda, a decrease of 0.71 percentage points. Thus, the large PPL increase is not accompanied by a comparably large drop in this seven-task mean; this does not imply equivalent performance on individual tasks or other downstream uses.

\paragraph{A more elaborate scoring variant.}
\Cref{tab:large_models} also includes shared-budget \algname{QK-Wanda-MRC}, which combines causal masking, RoPE, and centering (explained in Appendices~\ref{sec:causal_method}, \ref{app:rope}, and~\ref{app:centered}, respectively). At 80\% on Llama-3.1-70B, it reduces WikiText-2/C4 PPL from 18.89/32.59 to 13.67/22.61 relative to the base score, but remains worse than Wanda and decreases mean accuracy from 65.93\% to 62.61\%. Also at 80\%, the base score has lower PPL and higher accuracy on Llama-2-70B, while Qwen2.5-72B has mixed PPL ordering. MRC therefore offers a checkpoint-specific improvement rather than a uniform replacement. Appendix~\ref{app:factorial_selection} compares all eight combinations.

\paragraph{Scoring versus allocation.}
Relaxing Wanda's row quotas has a checkpoint-dependent effect. At 80\%, matrix budgets lower both PPL metrics on Llama-2-70B and Qwen2.5-72B, but worsen them on Llama-3.1-70B (\cref{tab:large_models}). Comparing matrix-budget Wanda with separate-budget QK-Wanda holds per-projection deletion counts fixed. Across the three largest checkpoints and nine sparsities, the QK score lowers QK error in all 27 configurations, WikiText-2 PPL in 22, and C4 PPL in 20. At 80\% on Llama-2-70B, it reduces WikiText-2/C4 PPL from 4.13/6.53 to 3.83/6.14 and QK error from 7.19\% to 1.58\%; a shared budget further reduces PPL to 3.63/5.97. These results show that scoring and budget allocation both contribute. Full curves and family tables in Appendices~\ref{app:budget_curves} and~\ref{app:model_results}.

\Cref{fig:main_llama2_70b} gives the Llama-2-70B sweep. At 50\%, absolute PPL gains are small (WikiText-2/C4 3.36/5.74 to 3.33/5.72). At 90\%, PPL decreases from 11.71/13.92 to 4.74/7.00 and QK error from 17.12\% to 2.87\%; The four main curves isolate row-wise Wanda, matrix-budget Wanda, and the two QK budgets for QK-Wanda; additional scoring variants are shown in the appendix.

\subsection{Results across all model sizes}
\label{sec:other_families}

Across 135 checkpoint--sparsity configurations, shared-budget QK-Wanda produces lower WikiText-2 PPL in 103, lower C4 PPL in 101, and lower full QK error in 133. Mean zero-shot accuracy increases in 30 of 45 evaluated configurations. These descriptive sign counts include arbitrarily small differences; Appendix~\ref{app:unmasked_aggregate_counts} gives detailed accounting.

The two exceptions to lower full QK error are Llama-3.2-3B at 80\% and 90\% sparsity: QK-Wanda gives 10.31\% and 24.87\%, compared with Wanda\textquotesingle s 9.87\% and 21.27\%. Exact single-deletion costs and an additive mask-selection surrogate do not guarantee a smaller coupled error after many deletions. The mean of per-checkpoint relative QK-error reductions is 60.35\% at 50\% sparsity and 44.70\% at 80\%; the corresponding medians are 60.70\% and 50.09\%, and the minimum at 80\% is $-4.42\%$. The diagnostic measures pre-RoPE products on common dense block inputs; it does not measure how RoPE and softmax transform perturbations or how pruning errors accumulate across blocks.

\paragraph{Family-dependent downstream effects.}
At 80\%, mean-accuracy gains are 0.76, 5.28, 7.77, and 5.94 percentage points on TinyLlama and Llama-2-7B/13B/70B. Qwen2.5-72B gains 0.56 points. Llama-3.1-8B illustrates a different tradeoff: WikiText-2/C4 PPL worsens from 8.89/13.54 to 13.47/21.97, yet mean accuracy rises from 59.64\% to 60.61\%. All four Llama~3 sizes have worse PPL on both datasets at 80\%. Neither local reconstruction nor PPL alone predicts every task outcome; Appendix~\ref{app:zero_shot_tables} retains the full task scores.

\subsection{Separate and shared QK budgets}
\label{sec:budget_allocation}

\begin{figure}[!ht]
\centering
\includegraphics[width=\textwidth]{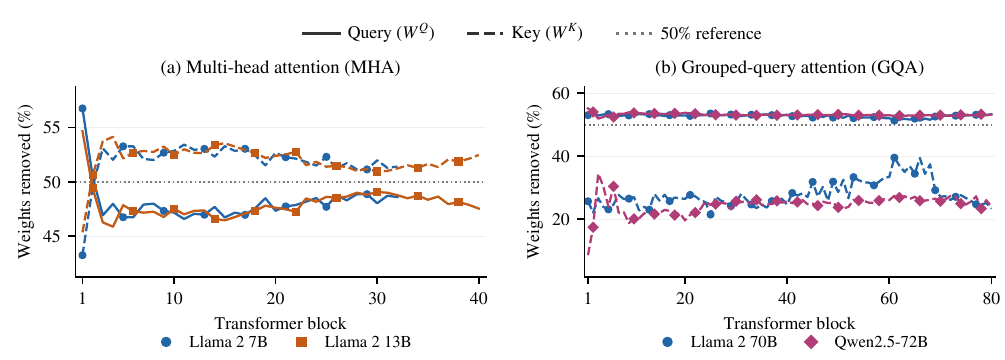}
\caption{QK-Wanda redistributes a shared 50\% pruning budget between query and key weights across transformer blocks. The left panel shows Llama-2-7B/13B (MHA); the right shows Llama-2-70B and Qwen2.5-72B (GQA). Solid and dashed lines show the percentages of query and key weights removed, respectively. Profiles for all 15 checkpoints appear in Appendix~\ref{app:shared_allocation}.}
\label{fig:main_qk_allocation}
\end{figure}

Across the 135 checkpoint--sparsity configurations, shared QK budgets lowered WikiText-2 PPL in 112, C4 PPL in 119, and QK error in 122 relative to separate budgets (Appendix~\ref{app:unmasked_aggregate_counts}). Sharing was not uniformly better: Llama-2-7B has slightly lower QK error under separate budgets at 50\% and 80\%, and at 90\% separate budgets give lower WikiText-2/C4 PPL (29.04/40.18 versus 30.13/41.87).

\Cref{fig:main_qk_allocation} contrasts the resulting allocations in MHA and GQA at a shared 50\% budget. The horizontal axis follows transformer blocks from the input side to the final block; solid and dashed curves show the percentages removed from \(W^Q\) and \(W^K\), respectively. The left panel contains the two MHA checkpoints in our study, Llama-2-7B and 13B. Their allocation is nearly balanced, but consistently favors retaining \(W^Q\) weights after the first few blocks: mean removal rates for \(W^Q\) and \(W^K\) are 48.09\%/51.91\% and 47.98\%/52.02\%, respectively.

Llama-2-70B and Qwen2.5-72B in the right panel show the opposite preference: both preserve a larger fraction of the weights in \(W^K\) throughout their 80 blocks. \(W^Q\) removal stays slightly above 50\%, while \(W^K\) removal is substantially lower. For example, Llama-2-70B removes 52.76\% of \(W^Q\) weights and 27.91\% of \(W^K\) weights on average. This preference persists across depth, although \(W^K\) removal fluctuates more between blocks.

In MHA, \(W^Q\) and \(W^K\) have equal parameter counts, and the observed removal fractions generally stay close to 50\%. In GQA, several query heads share each key head, making \(W^Q\) larger than \(W^K\). A small increase in the fraction removed from \(W^Q\) therefore allows a much larger decrease in the fraction removed from \(W^K\), while exactly half of the combined weights are removed.

The QK scores thus provide an inexpensive way to decide how much of the pruning budget each projection should receive. The resulting sparsities vary from block to block, reflecting differences in the reconstruction costs of removing \(W^Q\) and \(W^K\) weights. This allocation emerges directly from the scores, without specifying a fixed sparsity for each projection or hand-tuning a schedule across blocks. Beyond pruning, these results motivate exploring similar reconstruction-based budget allocation for other compression methods, such as mixed-precision quantization.

\subsection{Pruning all block linear projections}
\label{sec:full_block}

% Generated by scripts/export_unmasked_full_block.py from independent audit.
\begin{table}[htbp]
\centering
\fontsize{8}{9}\selectfont
\setlength{\tabcolsep}{1pt}
\caption{Pruning all seven transformer-block projections at 30\% and 50\% sparsity. Wanda + QK-Wanda combines shared-budget QK pruning with row-wise Wanda on the remaining projections. Accuracy is the seven-task zero-shot mean (\%). Bold marks better unrounded values within each checkpoint and sparsity.}
\label{tab:full_block}
\begin{qktabular}{11}{lcccccccccc}
\toprule
 & & \multicolumn{3}{c}{\textbf{TinyLlama-1.1B}} & \multicolumn{3}{c}{\textbf{Llama-2-7B}} & \multicolumn{3}{c}{\textbf{Llama-2-13B}}\\
\cmidrule(lr){3-5}\cmidrule(lr){6-8}\cmidrule(lr){9-11}
Method & Sparsity & \shortstack{Wiki2\\PPL} & \shortstack{C4\\PPL} & \shortstack{Zero-shot\\acc. (\%)} & \shortstack{Wiki2\\PPL} & \shortstack{C4\\PPL} & \shortstack{Zero-shot\\acc. (\%)} & \shortstack{Wiki2\\PPL} & \shortstack{C4\\PPL} & \shortstack{Zero-shot\\acc. (\%)}\\
\midrule
Dense & 0\% & 7.77 & 9.94 & 49.93 & 5.47 & 7.26 & 64.38 & 4.88 & 6.73 & 67.05\\
\midrule
\algname{Wanda} & 30\% & 8.33 & 10.46 & \textbf{50.22} & 5.66 & 7.50 & 63.88 & 5.06 & 6.90 & 66.14\\
\rowcolor{qkrow}
\algname{Wanda} + \algname{QK-Wanda} & 30\% & \textbf{8.31} & \textbf{10.44} & 49.96 & \textbf{5.66} & \textbf{7.49} & \textbf{64.02} & \textbf{5.05} & \textbf{6.90} & \textbf{66.46}\\
\addlinespace[2pt]
\algname{Wanda} & 50\% & 11.52 & 13.85 & 48.22 & 6.89 & 9.06 & 60.23 & 5.90 & 8.00 & \textbf{64.61}\\
\rowcolor{qkrow}
\algname{Wanda} + \algname{QK-Wanda} & 50\% & \textbf{11.22} & \textbf{13.62} & \textbf{48.56} & \textbf{6.84} & \textbf{9.05} & \textbf{61.02} & \textbf{5.86} & \textbf{7.98} & 64.55\\
\addlinespace[2pt]
\bottomrule
\end{qktabular}
\end{table}

We use QK-Wanda with a shared QK budget while retaining row-wise Wanda on the value, output, and MLP projections. \Cref{tab:full_block} compares this hybrid with row-wise Wanda throughout at 30\% and 50\% full-block sparsity. All seven block projections are pruned; embeddings, biases, normalization parameters and the output head remain dense. We evaluate the same pruned model on all seven zero-shot tasks, without additional pruning or training.

The hybrid lowers PPL in 12 of 12 checkpoint--sparsity--dataset comparisons. At 50\% on TinyLlama, WikiText-2/C4 PPL changes from 11.52/13.85 to 11.22/13.62. At 50\% on Llama-2-7B, WikiText-2/C4 PPL changes from 6.89/9.06 to 6.84/9.05. At 50\% on Llama-2-13B, WikiText-2/C4 PPL changes from 5.90/8.00 to 5.86/7.98. Mean accuracy is higher in 4 of 6 completed comparisons. These results show that QK-Wanda can be combined with pruning the remaining block projections, extending its use beyond query and key pruning alone.

\subsection{QK-Wanda variants}
\label{sec:calibration_studies}

\paragraph{Choice of reconstruction target.}
We tested three extensions to the base unmasked loss in \cref{eq:main_unmasked_loss}. Rotary position embeddings (RoPE, R) rotate queries and keys before their dot products, so reconstruction can instead target post-RoPE products. Centering (C) removes the mean reconstruction error across eligible keys for each query: adding the same constant to every allowed logit in a query row leaves its softmax unchanged, motivating an objective that ignores such shifts. A causal mask (M) restricts reconstruction to keys at or before the query position, matching the pairs accessible to decoder attention. These choices modify the pruning objective, not the model's attention computation.

Appendix~\ref{app:factorial_selection} compares all eight combinations on six checkpoints at 50\% and 80\% shared QK sparsity, with PPL and seven-task accuracy. We found no consistent downstream advantage from the added switches: their effects depend on checkpoint, sparsity, and metric, and some combinations worsen results. Post-RoPE centering with a causal mask (MRC) narrows the PPL gap on Llama-3.1-70B relative to base QK-Wanda, but still trails Wanda and lowers mean accuracy (\cref{sec:largest_models}). We therefore retain the simple unmasked objective as the base method and report the extended objectives, derivations, and comparisons in the appendices.

\paragraph{Progressive re-scoring.}
QK-Wanda's scores for \(W^Q\) depend on key activations, and its scores for \(W^K\) depend on query activations. Removing weights from one projection can therefore change the scores of the remaining weights in the other. Wanda's scores, in contrast, stay unchanged when the projection inputs are fixed. This dependence suggests pruning in small steps and refreshing the scores as weights are removed. We tested this approach by pruning part of \(W^Q\) and \(W^K\) simultaneously, recomputing both sets of scores from the partially pruned projections, and repeating until reaching the target sparsity. The tested schedules did not improve PPL: changes were negligible under row budgets, while progressive pruning worsened PPL under separate matrix budgets. Appendix~\ref{app:iterative} gives the schedules and results. This study and the three-seed ablation used distinct calibration protocols and are excluded from the main aggregate.

\section{Limitations and Scope}
\label{sec:limitations}

Lower full QK error does not guarantee lower PPL or higher task accuracy. At 80\% QK sparsity, all four Llama~3 checkpoints have higher PPL than Wanda, and three have lower QK error. Llama-3.1-70B's mean accuracy gap is only 0.71 percentage points despite much higher PPL; this does not establish statistical equivalence or preserved generative quality. Individual task results and PPL remain complementary.

The scores are exact for a single weight deletion, but do not guarantee an optimal mask when many weights are removed. Minimizing QK reconstruction error also does not guarantee minimizing reconstruction error at the output of the attention layer or transformer block. The held-out QK diagnostic measures the same type of local objective used for scoring; it is not independent evidence of downstream quality. Models with additional nonlinear QK transformations require separate treatment.

Our main results concern QK-only unstructured pruning. Setting weights to zero does not by itself reduce dense-kernel runtime or storage; we do not benchmark inference speedups or hardware-supported N:M sparsity. Full-block evidence covers only TinyLlama-1.1B and Llama-2-7B/13B at 30\% and 50\%.

Compensating updates to surviving weights remain outside this study.

\section{Conclusion}

\algname{QK-Wanda} couples query and key pruning through exact individual weight-deletion costs for QK reconstruction, without gradients or updates to retained weights. With the main calibration protocol, measured pruning takes 1.3\% longer than Wanda on A100 and 3.1\% longer on H200, averaged across the four benchmarked model sizes. Across 15 checkpoints, its shared budget lowers QK reconstruction error in 133 of 135 configurations and often improves perplexity and mean zero-shot accuracy. Llama~3 highlights the limits of this local objective: lower QK error can accompany much higher perplexity, even when mean zero-shot accuracy changes relatively little. These results show the promise of coupled scoring as a practical, inexpensive pruning criterion, while emphasizing that reconstruction quality alone does not determine downstream performance. Accounting for propagated errors and combining the scores with compensating weight updates are directions for future work.

\clearpage
\bibliographystyle{iclr2027_conference}
\bibliography{iclr2027_conference}
\clearpage
\flushbottom
\appendix
% Appendix-only presentation: no row highlighting or stretched vertical gaps.
\raggedbottom
\definecolor{qkrow}{RGB}{255,255,255}
\setlength{\intextsep}{10pt plus 2pt minus 2pt}
\section{Most frequently used notation}\label{sec:notation}

\begin{table}[H]
\caption{Notation used in the main text and appendices. Derivations fix one block; \(h,g\) index its heads and \(s\) its input sequences. The block index \(\ell\) appears in cross-block metrics. Indices are suppressed only when their objects are fixed.}
\label{tab:notation}
\centering
\small
\setlength{\tabcolsep}{3pt}
\renewcommand{\arraystretch}{0.98}
\begin{tabularx}{\linewidth}{@{}l c Y@{}}
\toprule
\multicolumn{3}{@{}l}{\textbf{Dimensions and indices}}\\
\(d,\ d_{\mathrm{out}},\ m\) & -- & Projection input width, output width, and attention-head width, respectively.\\
\(N,\ s,\ T,\ T_s\) & -- & Calibration sequence count/index; common length and length of sequence \(s\).\\
\(i,j;\ t,u\) & -- & Output/input weight coordinates; query/key token positions.\\
\(D,\ell;\ f\) & -- & Number/index of transformer blocks; MLP intermediate width.\\
\(H_Q,H_K;\ h,g\) & -- & Numbers of query/key heads and their respective indices.\\
\(\mathcal H_g\) & -- & Query heads sharing key head \(g\); these groups partition all query heads.\\
\(D_Q,\ D_K\) & -- & Total projection widths: \(H_Qm\) and \(H_Km\).\\
\midrule
\multicolumn{3}{@{}l}{\textbf{Weights, activations, and token-pair products}}\\
\(W,\widehat W,\Delta\) & -- & Projection weights, pruned weights, and change \(\Delta=\widehat W-W\).\\
\(O,\widehat O\) & -- & Linear outputs \(O=WX\) and \(\widehat O=\widehat W X\), in \(\R^{d_{\mathrm{out}}\times T}\).\\
\(W^Q,W^K\) & -- & Full weights in \(\R^{D_Q\times d},\R^{D_K\times d}\); each head matrix is \(m\times d\).\\
\(b^Q,b^K\) & -- & Retained QK biases; parenthesized head labels give \(b^{Q^{(h)}},b^{K^{(g)}}\).\\
\(X,\ Q,\ K\) & -- & Input and pre-RoPE activations: \(X\in\R^{d\times T}\), \(Q,K\in\R^{m\times T}\) for one head.\\
\(Q^{(h)},K^{(g)}\) & -- & Head activations; \(Q^{(h,s)},K^{(g,s)}\in\R^{m\times T_s}\). Hats denote pruned outputs.\\
\(x_t,q_t,k_u\) & -- & Token columns; full labels: \(x_t^{(s)},q_t^{(h,s)},k_u^{(g,s)}\).\\
\(e,\ e^{(s)};\ e_i\) & -- & All-ones columns in \(\{1\}^{T\times1},\{1\}^{T_s\times1}\); coordinate basis vector \(e_i\).\\
\(B=Q^\top K,\ Z\) & -- & Unscaled pre-/post-RoPE token-pair products; \(B^{(h,s)},Z^{(h,s)}\in\R^{T_s\times T_s}\).\\
\(\widehat M,\ \Delta M\) & -- & For a matrix \(M\), its pruned value and change \(\Delta M=\widehat M-M\).\\
\(R_t^{(s)},\ \widetilde q,\widetilde k\) & -- & Fixed rotary map; rotated query/key vectors.\\
\midrule
\multicolumn{3}{@{}l}{\textbf{Losses, scores, and sparsity allocation}}\\
\(C_{tu}=\mathbf1\{u\le t\}\) & -- & Binary causal mask; \(\mathbf1\{\cdot\}\) is the indicator.\\
\(\mathcal L_W\) & -- & Linear reconstruction loss \(\norm{\widehat O-O}_F^2=\norm{\Delta X}_F^2\).\\
\(\mathcal L\) & -- & QK loss over all token pairs; sums head pairs and averages independent sequences.\\
\(\mathcal L_C\) & -- & Causal pre-RoPE reconstruction loss.\\
\(\mathcal L_R\) & -- & Causal post-RoPE reconstruction loss.\\
\shortstack{\(\mathcal L_{\mathrm{cent}}\)\\\(\mathcal L_{R,\mathrm{cent}}\)} & -- & Causal reconstruction losses with row centering, before and after RoPE.\\
\(\mathcal L_{\mathrm{add}}(z)\) & -- & Additive surrogate: sum of the individual costs of selected deletions.\\
\(\mathcal P,\ z_p,\ M\) & -- & Candidate weight indices, binary removal indicators, and integer deletion count.\\
\(S^W_{ij}\) & -- & Squared Wanda score \(W_{ij}^2\norm{X_{j:}}_2^2\).\\
\(S^Q_{ij},\ S^K_{ij};\ S_p\) & -- & Unmasked single-weight deletion costs in \(\mathcal L\); explicit head labels apply to MHA and GQA alike.\\
\(\hat S^Q_{ij},\ \hat S^K_{ij};\ \hat S_p\) & -- & Causal pre-RoPE costs; \(p\) labels a candidate. Hats on \(S\) mark causality.\\
\(\hat S^Q_{ij,R},\ \hat S^K_{ij,R}\) & -- & RoPE-aware costs in \(\mathcal L_R\); head labels use \(Q^{(h)},K^{(g)}\) for all score variants.\\
\(\rho;\ \rho_Q,\rho_K\) & -- & Requested removal fraction; realized fractions within \(W^Q\) and \(W^K\).\\
\(n_Q,n_K;\ P\) & -- & Weight counts \(n_Q=H_Qmd\), \(n_K=H_Kmd\); \(P=n_Q+n_K\) scores per shared pool.\\
\(P_K,\ U_Q,\ X_2\) & -- & Key-prefix/grouped-query-suffix squared energies; \(X_2=X\odot X\).\\
\(\mathcal P_a,\ V_{K,a},\ V_{Q,a}\) & -- & RoPE pair \(a\); key-prefix/grouped-query-suffix \(2\times2\) second-moment sums.\\
\(\nu_t(v),\mathcal V_t(v)\) & -- & Centered prefix energy and unnormalized prefix covariance (Appendix~\ref{app:centered}).\\
\(\delta,\ p,\ W_p\) & -- & Change \(\operatorname{vec}(\Delta W)\); within one head, \(p=i+(j-1)m\) and \(W_p=W_{ij}\).\\
\(G^Q,\ G^K,\ G_U\) & -- & Causal QK and unmasked quadratic coefficients; loss \(\delta^\top G\delta\), Hessian \(2G\).\\
\midrule
\multicolumn{3}{@{}l}{\textbf{Operations and evaluation}}\\
\(M^\top,\ M_{i:},\ M_{:t}\) & -- & Transpose, row, column; \(M_{i,1:t},M_{i,u:T}\) select a row prefix/suffix.\\
\(\norm{x}_2,\ \norm{M}_F\) & -- & Euclidean and Frobenius norms: \((\sum_i x_i^2)^{1/2}\) and \((\sum_{i,j}M_{ij}^2)^{1/2}\).\\
\(\odot,\ \otimes,\ \operatorname{vec}\) & -- & Entrywise product, Kronecker product, and column-wise vectorization.\\
\(\lceil a\rceil,\ \lfloor a\rfloor\) & -- & Smallest integer at least \(a\), and largest integer at most \(a\).\\
\(\langle A,B\rangle_C,\ \norm{A}_C^2\) & -- & Masked product sum and squared seminorm (mean over sequences, sum over heads).\\
\(E_{\mathrm{QK}}\) & -- & Block-mean relative squared QK error, in percent (\cref{eq:relative_qk_error}).\\
\(E_Y\) & -- & Pooled relative squared stage error, reported as \(100E_Y\) (\cref{eq:normalized_error}).\\
\(\overline\Delta_\rho\) & -- & Mean relative QK-error reduction across checkpoints (\%, \cref{eq:cross_checkpoint_reduction}).\\
\(J,r,M_r;\ E,\Delta B\) & -- & Step count/index/cumulative removals; accumulated error/next change (Appendix~\ref{app:iterative}).\\
\bottomrule
\end{tabularx}
\end{table}

\FloatBarrier
\section{Complete Proofs for the Causal Score}
\paragraph{Appendix roadmap.}
Appendices~\ref{app:causal_details}--\ref{app:variants} contain the score derivations; \ref{app:protocol} documents data and implementation, and \ref{app:calibration} covers default selection. Full family tables, task scores, and curves are in Appendices~\ref{app:model_results}--\ref{app:qk_error_curves}. Controlled variants and row budgets are in Appendices~\ref{app:variant_ablation} and~\ref{app:budget_llama2_7b}; progressive re-scoring, timing, and matched matrix budgets are in Appendices~\ref{app:iterative},~\ref{app:score_cost}, and~\ref{app:matrix_control}.
\label{app:causal_details}

\subsection{Causal masking: \algname{QK-Wanda-M}}
\label{sec:causal_method}

\Cref{sec:gqa_method} treats shared key heads with an unmasked reconstruction objective. Here we additionally restrict reconstruction to causally allowed pairs, while retaining the same head notation. The following explanation and scores define \algname{QK-Wanda-M}. A causal mask prevents each query from attending to future tokens \citep{vaswani2017attention}. The reconstruction objective should therefore retain only token pairs with key position \(u\le t\) for query position \(t\). Grouped-query attention (GQA) lets several query heads share one key/value head, reducing KV-cache memory and bandwidth during decoding \citep{ainslie2023gqa}. For example, Llama~3.1-8B uses 32 query heads and 8 key/value heads, so each group of four query heads shares one key head and one value head \citep{grattafiori2024llama3}. We account for both causal masking and shared keys in the reconstruction loss and its single-deletion scores below. This causal, pre-RoPE formulation defines the masked variant, \algname{QK-Wanda-M}.

Consider a single transformer block. Query heads have indices \(h=1,\ldots,H_Q\), and key heads \(g=1,\ldots,H_K\) (e.g., \(H_Q=32\), \(H_K=8\) for Llama~3.1-8B). The set \(\mathcal H_g\) contains the query heads sharing key head \(g\); these sets partition the query heads. For example, \(\mathcal H_1=\{1,2\}\) means that queries \(Q^{(1)},Q^{(2)}\) both interact with \(K^{(1)}\). The projection weights \(W^{Q^{(h)}}\) and \(W^{K^{(g)}}\) produce activations \(Q^{(h)}\) and \(K^{(g)}\), respectively, using the same affine maps and retained biases as in \cref{eq:simple_projections}. The hatted activations \(\widehat Q^{(h)},\widehat K^{(g)}\) are the corresponding outputs after pruning, evaluated on the same inputs. Define the binary causal mask \(C_{tu}=\mathbf1\{u\le t\}\), \(C\in\{0,1\}^{T\times T}\), and the loss
\begin{equation}
 \mathcal L_C=\sum_{g=1}^{H_K}\sum_{h\in\mathcal H_g}
 \norm{C\odot(\widehat Q^{(h)\top}\widehat K^{(g)}-Q^{(h)\top}K^{(g)})}_F^2,
 \label{eq:main_causal_loss}
\end{equation}
where \(\odot\) is entrywise multiplication. Deleting a query weight changes one query head; deleting a key weight changes every product involving that key head. The query score below fixes \(h\in\mathcal H_g\); the key score sums over this group. As before, \(i,j\) index output coordinates and input features. The row slices \(K^{(g)}_{i,1:t}\) and \(Q^{(h)}_{i,u:T}\) select coordinate \(i\) over the key prefix and query suffix. We use \(\hat S\) for causal scores and \(S\) for unmasked scores. With token sums ranging from \(1\) to \(T\), the exact individual-deletion costs are
\begin{equation}
 \boxed{
 \hat S^{Q^{(h)}}_{ij}=(W^{Q^{(h)}}_{ij})^2\sum_t X_{jt}^2\norm{K^{(g)}_{i,1:t}}_2^2,
 \quad
 \hat S^{K^{(g)}}_{ij}=(W^{K^{(g)}}_{ij})^2\sum_u X_{ju}^2\sum_{h\in\mathcal H_g}\norm{Q^{(h)}_{i,u:T}}_2^2.
 }
 \label{eq:causal_single_scores}
\end{equation}
The token ranges in \cref{eq:causal_single_scores} follow the causal mask: a query at position \(t\) interacts only with keys at positions \(1,\ldots,t\), while a key at position \(u\) interacts with queries at positions \(u,\ldots,T\). If several query heads share that key head, deleting one of its weights affects all of them, so the key score adds their contributions.

\Cref{eq:causal_single_scores} recovers the simpler scores in \cref{eq:simple_wanda_connection} when we use ordinary MHA and remove the causal mask. In MHA, each key head serves exactly one query head, so \(\mathcal H_g=\{g\}\) and the head sum reduces to a single contribution. Without the mask, every query token interacts with every key token, so the restricted token ranges become full rows. For a fixed query/key pair, we can then omit the head indices.

For several independent calibration sequences, we average the deletion costs over sequences. \Cref{app:proof_c,app:worked_example} give the full formulas, proofs, and a two-token example. Appendix~\ref{app:unmasked} gives the corresponding formulas without a causal mask, including multiple calibration sequences and shared key heads.

\algname{QK-Wanda-MR} applies the same scalar-deletion construction to a causal target that includes rotary position embeddings (RoPE). \Cref{sec:calibration_studies} summarizes the scoring variants; Appendix~\ref{app:rope} gives the derivation.

\subsection{Full objective, general scores, and their proof}
\label{app:proof_c}

We extend the single-sequence MHA/GQA scores in \cref{sec:causal_method} to \(N\) independent calibration sequences. All identities hold for fixed inputs and real-valued arithmetic in one block; its index \(\ell\) is suppressed throughout the derivation. Sequence \(s\) has \(T_s\) tokens and input \(X^{(s)}\in\R^{d\times T_s}\). On head activations, we append the sequence to the parenthesized head label: \(Q^{(h,s)}\), \(K^{(g,s)}\). These labels retain the main-text convention: \(h,g\) identify heads, \(s\) identifies a sequence, and subscripts select coordinates and token positions. We average sequence losses with equal weight \(1/N\), as in the experiments.

\paragraph{Heads and coordinates.}
All heads have common width \(m\). Query indices are \(h=1,\ldots,H_Q\), and key indices \(g=1,\ldots,H_K\). The disjoint groups \(\mathcal H_g\) partition the query heads by their shared key. In every query/key product or query score below, \(h\in\mathcal H_g\). MHA is the case \(\mathcal H_g=\{g\}\). The common attention scale \(1/\sqrt m\) divides all squared deletion costs by \(m\), preserving rankings.

The matrices \(W^{Q^{(h)}},W^{K^{(g)}}\in\R^{m\times d}\) map input features to head coordinates. In each projection label, the parenthesized superscript identifies the head, not a power. We use the same labels for biases, scores, and quadratic coefficients. Scalar entries are \(W^{Q^{(h)}}_{ij}\) and \(W^{K^{(g)}}_{ij}\), where \(i=1,\ldots,m\) and \(j=1,\ldots,d\). The same coordinate \(i\) is paired in each dot product: changing \(K^{(g,s)}_{iu}\) changes its products with \(Q^{(h,s)}_{it}\) for every \(h\in\mathcal H_g\). The optional bias vectors \(b^{Q^{(h)}},b^{K^{(g)}}\) have \(m\) entries. Let \(e^{(s)}\in\{1\}^{T_s\times1}\) be the all-ones column vector for sequence \(s\). Before rotary position embedding (RoPE), the activations are
\begin{equation}
 Q^{(h,s)}=W^{Q^{(h)}}X^{(s)}+b^{Q^{(h)}}(e^{(s)})^\top,
 \qquad
 K^{(g,s)}=W^{K^{(g)}}X^{(s)}+b^{K^{(g)}}(e^{(s)})^\top.
 \label{eq:head_projections}
\end{equation}
Both activation matrices have \(m\) rows and \(T_s\) columns. In entry notation, \(Q^{(h,s)}_{it}\) is query coordinate \(i\) at token \(t\), and \(K^{(g,s)}_{iu}\) is key coordinate \(i\) at token \(u\). Biases remain fixed and are zero for bias-free models. They cancel from their own projection changes but remain in opposite-side activations.

\paragraph{From heads to complete projections.}
The full block projections \(W^Q\in\R^{(H_Qm)\times d}\) and \(W^K\in\R^{(H_Km)\times d}\) stack their head matrices vertically: each head occupies \(m\) consecutive output rows. Their weight counts are \(n_Q=H_Qmd\) and \(n_K=H_Kmd\). Separate budgets apply to each complete projection, and a shared budget covers both; neither imposes head quotas. For the single-head derivation in \cref{sec:unmasked_core}, these full projections coincide with their sole head matrices.

\paragraph{Which token pairs contribute?}
Use \(t\) for a query position and \(u\) for a key position, both between 1 and \(T_s\). For \(h\in\mathcal H_g\), define the \(T_s\times T_s\) matrices
\begin{equation}
 B^{(h,s)}=Q^{(h,s)\top}K^{(g,s)},
 \qquad C^{(s)}_{tu}=\begin{cases}1,&u\le t,\\0,&u>t.\end{cases}
\end{equation}
Each query head belongs to exactly one group, so \(h\) determines its key head; \(B^{(h,s)}\) needs no additional \(g\) label. The entry \(B^{(h,s)}_{tu}\) sums the products of matching coordinates of query token \(t\) and key token \(u\), following the unscaled convention in \cref{sec:unmasked_core}. The binary matrix \(C^{(s)}\) is the causal mask: a query can use its current and preceding tokens, so only \(u\le t\) contributes. For three tokens it is \(C=\bigl(\begin{smallmatrix}1&0&0\\1&1&0\\1&1&1\end{smallmatrix}\bigr)\).

\paragraph{The reconstruction loss.}
The activations \(\widehat Q^{(h,s)},\widehat K^{(g,s)}\) are computed after pruning on the same \(X^{(s)}\), with biases retained. Their product \(\widehat B^{(h,s)}=\widehat Q^{(h,s)\top}\widehat K^{(g,s)}\), for \(h\in\mathcal H_g\), is compared with \(B^{(h,s)}\). The mask retains allowed token pairs. Using the squared Frobenius norm from \cref{sec:unmasked_core}, the causal loss is
\begin{align}
 \mathcal L_C
 &=\frac1N\sum_{s=1}^{N}\sum_{h=1}^{H_Q}
 \norm{C^{(s)}\odot\bigl(\widehat B^{(h,s)}-B^{(h,s)}\bigr)}_F^2\notag\\
 &=\frac1N\sum_{s=1}^{N}\sum_{h=1}^{H_Q}\sum_{t=1}^{T_s}\sum_{u=1}^{t}
 \bigl(\widehat B^{(h,s)}_{tu}-B^{(h,s)}_{tu}\bigr)^2.
 \label{eq:causal_objective}
\end{align}
As the groups partition the query heads, summing over \(h=1,\ldots,H_Q\) is the same as the group-then-head sum in \cref{eq:main_causal_loss}. We sum squared changes over allowed token pairs and heads, then average sequences. Independent sequences are evaluated separately; no dot product connects tokens from different sequences. The mask selects entries rather than adding the negative infinities used inside softmax. Using a sum instead of the average would multiply all scores by the same positive number \(N\), leaving their ranking unchanged.

Let \(\hat S^{Q^{(h)}}_{ij}\) denote the value of \(\mathcal L_C\) after setting the query weight \(W^{Q^{(h)}}_{ij}\) to zero while every other weight stays fixed. Define \(\hat S^{K^{(g)}}_{ij}\) in the same way for the key weight \(W^{K^{(g)}}_{ij}\). In this causal derivation, a hat on \(S\) labels the causal score. These scores measure the damage from each individual deletion; smaller is better to prune.

\begin{proposition}[Exact causal single-weight deletion]
\label{prop:causal_scores}
Fix the calibration inputs, projections, and retained biases defined above. For \(h\in\mathcal H_g\), the query score is
\begin{equation}
 \boxed{\hat S^{Q^{(h)}}_{ij}=\frac{(W^{Q^{(h)}}_{ij})^2}{N}
 \sum_{s=1}^{N}\sum_{t=1}^{T_s} (X^{(s)}_{jt})^2
 \sum_{u=1}^{t}(K^{(g,s)}_{iu})^2.}
 \label{eq:causal_q}
\end{equation}
For a key weight shared by the query heads in \(\mathcal H_g\), it is
\begin{equation}
 \boxed{\hat S^{K^{(g)}}_{ij}=\frac{(W^{K^{(g)}}_{ij})^2}{N}
 \sum_{s=1}^{N}\sum_{u=1}^{T_s} (X^{(s)}_{ju})^2
 \sum_{h\in\mathcal H_g}\sum_{t=u}^{T_s}(Q^{(h,s)}_{it})^2.}
 \label{eq:causal_k}
\end{equation}
\end{proposition}

Setting \(N=1\) recovers \cref{eq:causal_single_scores}: the inner token sums are precisely the squared prefix/suffix norms. In MHA, each grouped-query sum has one head. The factor \(1/N\) averages independent sequence losses; it does not introduce token pairs between different sequences.

The reference pair has zero reconstruction loss. These scores measure the loss after one finite deletion, not language-model cross-entropy. The proof first allows any fixed binary mask \(C^{(s)}\), then specializes to the causal mask. Sums over \(t,u\) range from 1 to \(T_s\) unless a causal restriction is stated.

\begin{proof}[Proof of \cref{prop:causal_scores}]
For a weight, activation, or product matrix \(M\), write \(\Delta M=\widehat M-M\) for its pruned value minus its reference value. Let \(e_i\in\R^{m}\) be the vector that is one at coordinate \(i\) and zero elsewhere, and define \(e_j\in\R^d\) similarly. These are standard basis vectors, distinct from the all-ones vector \(e^{(s)}\). The outer product \(e_i e_j^\top\) selects exactly one weight coordinate.

\paragraph{Query deletion.}
For \(w=W^{Q^{(h)}}_{ij}\), setting this weight to zero means
\begin{equation}
 \Delta W^{Q^{(h)}}=-w e_i e_j^\top,
 \qquad \Delta Q^{(h,s)}=-w e_i X^{(s)}_{j:}.
 \label{eq:query_weight_delta}
\end{equation}
All keys and other query heads remain fixed. The bias cancels when taking the difference of the two affine projections. Hence
\begin{align}
 \widehat B^{(h,s)}-B^{(h,s)}
 &=(Q^{(h,s)}+\Delta Q^{(h,s)})^\top K^{(g,s)}
             -Q^{(h,s)\top}K^{(g,s)}\notag\\
 &=-wX^{(s)\top}_{j:}K^{(g,s)}_{i:}.
 \label{eq:query_product_delta}
\end{align}
In particular, the \((t,u)\) entry is exactly
\begin{equation}
 \Delta B^{(h,s)}_{tu}
 =-wX^{(s)}_{jt}K^{(g,s)}_{iu}.
 \label{eq:causal_delta}
\end{equation}
Expanding the Frobenius norm entry by entry gives
\begin{align}
 \mathcal L_C
 &=\frac1N\sum_{s=1}^{N}\sum_{t,u}
       \bigl(C^{(s)}_{tu}\Delta B^{(h,s)}_{tu}\bigr)^2\notag\\
 &=\frac{w^2}{N}\sum_{s=1}^{N}\sum_{t,u}
       C^{(s)}_{tu}(X^{(s)}_{jt})^2(K^{(g,s)}_{iu})^2.
 \label{eq:general_mask_query}
\end{align}
We used \(C_{tu}^2=C_{tu}\), which holds because every mask entry is zero or one; all other heads contribute zero. For the causal mask, only \(u\le t\) remains in the inner sum, yielding \cref{eq:causal_q}.

\paragraph{Key deletion and GQA.}
For \(w=W^{K^{(g)}}_{ij}\), the corresponding finite change is
\begin{equation}
 \Delta W^{K^{(g)}}=-w e_i e_j^\top,
 \qquad \Delta K^{(g,s)}=-w e_i X^{(s)}_{j:}.
\end{equation}
For each \(h\in\mathcal H_g\), queries remain fixed, so
\begin{equation}
 \widehat B^{(h,s)}-B^{(h,s)}
 =Q^{(h,s)\top}\Delta K^{(g,s)}
 =-wQ^{(h,s)\top}_{i:}X^{(s)}_{j:}.
\end{equation}
The entry change is \(-wQ^{(h,s)}_{it}X^{(s)}_{ju}\). Heads outside \(\mathcal H_g\) do not change. Therefore
\begin{align}
 \mathcal L_C
 &=\frac1N\sum_{s=1}^{N}\sum_{h\in\mathcal H_g}\sum_{t,u}
       \bigl(C^{(s)}_{tu}\Delta B^{(h,s)}_{tu}\bigr)^2\notag\\
 &=\frac{w^2}{N}\sum_{s=1}^{N}\sum_u (X^{(s)}_{ju})^2
       \sum_{h\in\mathcal H_g}\sum_t C^{(s)}_{tu}(Q^{(h,s)}_{it})^2.
 \label{eq:general_mask_key}
\end{align}
The loss sums errors in distinct head matrices, so there are no cross-products between query heads. For the causal mask, valid query positions satisfy \(t\ge u\), giving \cref{eq:causal_k}. Both derivations are finite algebraic identities, with no Taylor truncation. Biases contribute to the fixed opposite-side activations throughout.
\end{proof}

A general binary mask can also express padding, sliding windows, or document boundaries. The evaluated runner uses fixed-length unpadded windows with ordinary causal attention; the mathematical statement does not imply implementation support for every mask layout.

\subsection{A two-token example}
\label{app:worked_example}

Take one sequence and one head, with \(N=1\), \(T=2\), \(d=2\), \(m=1\), and zero biases. Suppress the sequence/head labels. Let
\begin{equation*}
 X=\begin{pmatrix}1&2\\3&4\end{pmatrix},\qquad
 W^{Q}=\begin{pmatrix}2&1\end{pmatrix},\qquad
 W^{K}=\begin{pmatrix}1&0\end{pmatrix}.
\end{equation*}
Then \(Q=(5,8)\) and \(K=(1,2)\). Delete \(W^{Q}_{11}=2\), the weight from input feature 1 to query coordinate 1. The pruned query is \(\widehat Q=(3,4)\). The product change and its causal restriction are
\begin{equation*}
 \widehat B-B=\begin{pmatrix}-2&-4\\-4&-8\end{pmatrix},\qquad
 C\odot(\widehat B-B)=\begin{pmatrix}-2&0\\-4&-8\end{pmatrix}.
\end{equation*}
The top-right entry is ignored because query token 1 cannot attend to key token 2. The loss is \(\mathcal L_C=(-2)^2+(-4)^2+(-8)^2=84\). The closed-form score gives the same answer:
\begin{equation*}
 \hat S^Q_{11}=2^2\bigl[1^2(1^2)+2^2(1^2+2^2)\bigr]=84.
\end{equation*}
The two terms in brackets use the key prefixes available to query tokens 1 and 2. This example makes explicit what the mask, squared norm, and prefix sum each contribute.

\subsection{When do row-wise masks coincide with Wanda?}
\label{app:row_equivalence}

For one unmasked sequence, \cref{eq:simple_wanda_connection} multiplies every Wanda score in query row \(i\) by the same factor \(\norm{K_{i:}}_2^2\), and every score in key row \(i\) by \(\norm{Q_{i:}}_2^2\). A positive factor preserves rankings and ties, so identical per-row deletion counts and tie handling give identical masks. If the factor is zero, all QK scores in that row tie and equivalence is not guaranteed. The same argument holds for the grouped opposite-head energy in the single-sequence GQA formula (\cref{app:unmasked}). For multiple independent sequences, the exact score averages products of input and opposite-side energies; these generally cannot be factored into a Wanda score times a row constant.

Causal masking can break the equivalence even for one sequence and one head. Set
\begin{equation*}
 X=\begin{pmatrix}1&0\\0&1\end{pmatrix},\qquad
 W^Q=\begin{pmatrix}2&1\end{pmatrix},\qquad
 W^K=\begin{pmatrix}1&3\end{pmatrix},
\end{equation*}
with zero biases. Then \(Q=(2,1)\) and \(K=(1,3)\). For the two query weights, Wanda gives scores \((4,1)\), and the unmasked score gives \((40,10)\): both delete the second weight at 50\% row sparsity. The causal scores are instead \((4,10)\), because the first query sees only the first key whereas the second sees both. They delete the first weight. Thus a row-wise causal control is informative; the unmasked equivalence is not a reason to omit it.

\subsection{Derivation of the prefix/suffix accumulators}

For sequence \(s\), define two \(m\times T_s\) arrays. The notation \([i,t]\) selects the entry at row \(i\), column \(t\):
\begin{align}
 P_K^{(g,s)}[i,t]&=\sum_{u\le t}(K^{(g,s)}_{iu})^2,\\
 U_Q^{(g,s)}[i,u]&=\sum_{h\in\mathcal H_g}\sum_{t\ge u}(Q^{(h,s)}_{it})^2.
\end{align}
The array \(P_K\) stores key-prefix sums; \(U_Q\) stores grouped-query-suffix sums. They are computed by forward and reverse cumulative sums. Let \(X_2^{(s)}=X^{(s)}\odot X^{(s)}\) denote the input matrix with each entry squared. For any matrix \(M\), \([M]_{ij}\) also denotes its row-\(i\), column-\(j\) entry. Thus
\begin{align}
 [P_K^{(g,s)}X_2^{(s)\top}]_{ij}
 &=\sum_t\sum_{u\le t}(K^{(g,s)}_{iu})^2(X^{(s)}_{jt})^2,\\
 [U_Q^{(g,s)}X_2^{(s)\top}]_{ij}
 &=\sum_u\sum_{h\in\mathcal H_g}\sum_{t\ge u}
       (Q^{(h,s)}_{it})^2(X^{(s)}_{ju})^2.
\end{align}
Multiplying by the corresponding squared weight, then averaging over the \(N\) sequences, gives exactly \cref{eq:causal_q,eq:causal_k}. Each resulting interaction matrix has shape \(m\times d\). Query factors are reused within a GQA group; their weight multipliers differ by head. These are exact rearrangements of the pairwise sums.

The implementation accumulates in FP32 and retains the common attention-scale factor \(1/m\) in both score arrays. The formulas report unscaled deletion costs; this common factor preserves their ranking under separate or shared budgets. Within a block, both masks are formed before either is applied, then the sparse block output is propagated. Thus the inputs are fixed during a local scoring decision, even though later blocks receive method-specific sparse inputs.

\subsection{Quadratic form and its diagonal}
\label{app:causal_quadratic}

The notation \(X_{:t}\) selects column \(t\), so write the input, key, and query column vectors as \(x_t^{(s)}=X^{(s)}_{:t}\), \(k^{(g,s)}_{u}=K^{(g,s)}_{:u}\), and \(q^{(h,s)}_{t}=Q^{(h,s)}_{:t}\). The operation \(\operatorname{vec}(M)\) stacks columns in order; \(\otimes\) denotes the Kronecker product.

Fix \(h\in\mathcal H_g\). For a change to that query-head matrix, let \(\delta=\operatorname{vec}(\Delta W^{Q^{(h)}})\in\R^{dm}\). With keys fixed, the change of an unscaled dot product can be written as
\begin{equation}
 k^{(g,s)\top}_{u}\Delta W^{Q^{(h)}}x_t^{(s)}
 =\bigl(x_t^{(s)}\otimes k^{(g,s)}_{u}\bigr)^\top
       \operatorname{vec}(\Delta W^{Q^{(h)}}).
\end{equation}
Squaring this scalar and summing valid entries proves that the exact fixed-key reconstruction loss is \(\delta^\top G^{Q^{(h)}}\delta\), where
\begin{equation}
 G^{Q^{(h)}}=\frac1N\sum_{s=1}^{N}\sum_{t,u} C^{(s)}_{tu}
 (x_t^{(s)}x_t^{(s)\top})\otimes
 (k^{(g,s)}_{u}k^{(g,s)\top}_{u}).
 \label{eq:causal_curvature}
\end{equation}
The matrix \(G^{Q^{(h)}}\in\R^{dm\times dm}\) therefore collects the coefficients of this quadratic loss. For a change to key head \(g\), use \(\delta=\operatorname{vec}(\Delta W^{K^{(g)}})\) instead. Keeping queries fixed and including all affected heads gives the coefficient
\begin{equation}
 G^{K^{(g)}}=\frac1N\sum_{s=1}^{N}\sum_u\sum_{h\in\mathcal H_g}\sum_t
 C^{(s)}_{tu}
 (x_u^{(s)}x_u^{(s)\top})\otimes
 (q^{(h,s)}_{t}q^{(h,s)\top}_{t}).
\end{equation}
Write \(G\) for either coefficient matrix when the side/head is clear, and \(W\) for the corresponding reference head matrix. Each coefficient matrix is positive semidefinite: \(v^\top Gv\ge0\) for every vector \(v\), because it represents a sum of squared dot-product changes. The Hessian with respect to \(\delta\) is \(2G\). Matrix entry \((i,j)\) has vectorized index \(p=i+(j-1)m\), and \(W_p=[\operatorname{vec}(W)]_p=W_{ij}\) is the reference weight, not a component of \(\delta\). Let \(e_p\in\R^{dm}\) select that coordinate. Deleting it means \(\delta=-W_p e_p\), yielding \(\hat S_p=W_p^2[G]_{pp}\). The diagonal coefficient \([G]_{pp}\) determines the cost of perturbing coordinate \(p\); off-diagonal entries describe interactions between weights. The causal restriction generally prevents factorization into separately aggregated input and opposite-projection Gram matrices (products such as \(XX^\top\)).

\subsection{Scaling and the limits of scalar exactness}

Multiply all query weights and biases in one GQA group by a nonzero scalar \(\lambda\), and the shared key weights and bias by its reciprocal \(\lambda^{-1}\). In each Q score, the squared query weight gains \(\lambda^2\) and key energy gains \(\lambda^{-2}\); the factors cancel. The key score has the reciprocal cancellation. Thus products and \algname{QK-Wanda-M} scores are invariant. Including biases is required for complete activations to scale this way; groupwise scalar rescaling also commutes with RoPE.

For several query deletions, \(\delta^\top G\delta\) contains off-diagonal interactions. Summing scalar scores retains only its diagonal terms. For one sequence and one head, suppress their labels. If both projections change, exact expansion gives
\begin{equation}
 \widehat Q^\top\widehat K-Q^\top K
 =\Delta Q^\top K+Q^\top\Delta K+\Delta Q^\top\Delta K.
 \label{eq:joint_change}
\end{equation}
Apply the causal mask to this complete change. Its squared norm also contains cross terms between the two linear paths and the bilinear path. Independent scalar thresholding therefore does not give the exact joint loss or an optimal full mask.

% BEGIN UNMASKED SIX
\subsection{Does causal masking improve perplexity?}
\label{app:unmasked_six}

The causal objective excludes products between a query and future keys, which cannot contribute to decoder attention. This structural match does not by itself establish a downstream advantage over the unmasked objective. We therefore compared the two pre-RoPE scores on six checkpoints spanning Llama~2, Llama~3, and Qwen2.5, including both MHA and GQA. We held the checkpoint revision, calibration samples, numerical precision, pruning procedure, and shared QK budget fixed, and measured WikiText-2 and C4 perplexity at 30\%, 50\%, and 80\% sparsity. \Cref{tab:unmasked_six} reports the comparison.

\begin{table}[htbp]
\centering\small
\renewcommand{\thetable}{B.1}
\renewcommand{\theHtable}{unmasked-six}
\setlength{\tabcolsep}{6pt}
\caption{Causal QK-Wanda-M versus unmasked QK-Wanda with shared QK budgets. Both variants use the same checkpoint and calibration protocol. Lower perplexity is better; bold marks the better value in each pair, including ties at the displayed precision.}
\label{tab:unmasked_six}
\addtocounter{table}{-1}
\begin{qktabular}{6}{llrrrr}
\toprule
& & \multicolumn{2}{c}{WikiText-2 PPL} & \multicolumn{2}{c}{C4 PPL}\\
\cmidrule(lr){3-4}\cmidrule(lr){5-6}
Model & Sparsity & QK-Wanda-M & QK-Wanda & QK-Wanda-M & QK-Wanda\\
\midrule
Llama-2-7B & 30\% & \textbf{5.478} & \textbf{5.478} & \textbf{7.272} & \textbf{7.272}\\
 & 50\% & 5.530 & \textbf{5.529} & 7.358 & \textbf{7.357}\\
 & 80\% & 8.328 & \textbf{8.313} & 11.213 & \textbf{11.197}\\
\addlinespace
Llama-2-13B & 30\% & \textbf{4.888} & \textbf{4.888} & \textbf{6.733} & \textbf{6.733}\\
 & 50\% & \textbf{4.933} & \textbf{4.933} & \textbf{6.792} & \textbf{6.792}\\
 & 80\% & 6.272 & \textbf{6.246} & 8.583 & \textbf{8.546}\\
\addlinespace
Llama-3.1-8B & 30\% & \textbf{6.254} & \textbf{6.254} & \textbf{9.570} & \textbf{9.570}\\
 & 50\% & \textbf{6.405} & \textbf{6.405} & \textbf{9.816} & \textbf{9.816}\\
 & 80\% & 13.814 & \textbf{13.470} & 22.361 & \textbf{21.967}\\
\addlinespace
Llama-3.1-70B & 30\% & \textbf{2.821} & \textbf{2.821} & \textbf{7.124} & \textbf{7.124}\\
 & 50\% & 2.972 & \textbf{2.968} & 7.272 & \textbf{7.268}\\
 & 80\% & 18.984 & \textbf{18.890} & \textbf{32.372} & 32.589\\
\addlinespace
Qwen2.5-7B & 30\% & \textbf{6.854} & 6.855 & \textbf{11.886} & \textbf{11.886}\\
 & 50\% & \textbf{6.886} & \textbf{6.886} & \textbf{11.922} & \textbf{11.922}\\
 & 80\% & \textbf{7.320} & 7.323 & 12.619 & \textbf{12.617}\\
\addlinespace
Qwen2.5-72B & 30\% & \textbf{3.878} & \textbf{3.878} & \textbf{9.258} & \textbf{9.258}\\
 & 50\% & \textbf{3.895} & \textbf{3.895} & \textbf{9.270} & 9.271\\
 & 80\% & 4.244 & \textbf{4.234} & 9.523 & \textbf{9.520}\\
\bottomrule
\end{qktabular}
\end{table}

Neither variant uniformly dominated. At 30\% and 50\% sparsity, all absolute PPL differences were below 0.005. At 80\%, unmasked scoring improved both datasets for Llama-2-7B/13B, Llama-3.1-8B, and Qwen2.5-72B; the other two checkpoints had mixed outcomes. The largest improvement was on Llama-3.1-8B, from 13.814 to 13.470 on WikiText-2 and from 22.361 to 21.967 on C4. Across all 18 model--sparsity pairs, unmasked scoring yielded lower PPL in 11 comparisons on WikiText-2 and 12 on C4, counting unrounded values. These single-calibration-seed measurements do not establish general empirical superiority of the causal variant: its motivation is the restriction to attention-eligible products, while the observed PPL advantage depends on the checkpoint and sparsity. We make no inference about zero-shot accuracy from this PPL-only comparison.

% END UNMASKED SIX

\section{Proofs for the Unmasked and RoPE-Aware Variants}
\label{app:variants}

We give the general unmasked formulas for multiple calibration sequences and shared key heads, then derive the causal RoPE-aware extension. Both use the same single-weight deletion principle as the main text.

\subsection{\algname{QK-Wanda}: multiple sequences and shared keys}
\label{app:unmasked}

\Cref{sec:unmasked_core} derives the unmasked scores for one sequence and one head. Here we extend \cref{eq:gqa_single_scores} to independent calibration sequences using the head notation of \cref{sec:gqa_method}.

\begin{corollary}[Exact unmasked scalar scores]
\label{prop:unmasked}
For fixed inputs, affine QK projections with retained biases, and the head groups in \cref{sec:gqa_method}, every token pair contributes to the reconstruction loss. The loss \(\mathcal L\) sums unmasked head losses and averages over sequences. For \(h\in\mathcal H_g\), its exact single-weight costs, written without a hat, are
\begin{align}
 S^{Q^{(h)}}_{ij}&=\frac{(W^{Q^{(h)}}_{ij})^2}{N}
 \sum_{s=1}^{N}\norm{X^{(s)}_{j:}}_2^2\norm{K^{(g,s)}_{i:}}_2^2,
 \label{eq:unmasked_q}\\
 S^{K^{(g)}}_{ij}&=\frac{(W^{K^{(g)}}_{ij})^2}{N}
 \sum_{s=1}^{N}\norm{X^{(s)}_{j:}}_2^2
 \sum_{h\in\mathcal H_g}\norm{Q^{(h,s)}_{i:}}_2^2.
 \label{eq:unmasked_k}
\end{align}
\end{corollary}
\begin{proof}
Set \(C_{tu}^{(s)}=1\) in \cref{eq:general_mask_query,eq:general_mask_key}. The sum over input positions separates from the sum over opposite-projection positions within each sequence. Equivalently, the query perturbation in \cref{eq:query_product_delta} is an outer product. For any vectors \(x,y\),
\begin{equation}
 \norm{xy^\top}_F^2=\sum_{t,u}x_t^2y_u^2
 =\left(\sum_tx_t^2\right)\left(\sum_uy_u^2\right).
\end{equation}
Applying this identity to the query perturbation gives \cref{eq:unmasked_q}. For a key deletion, apply it to \(Q^{(h,s)\top}_{i:}X^{(s)}_{j:}\) in each sequence and affected head and sum the squared losses over \(\mathcal H_g\), giving \cref{eq:unmasked_k}.
\end{proof}

For one sequence (\(N=1\)) and one query head per key head, suppress sequence/head labels. The scores reduce to
\begin{equation}
 S^Q_{ij}=(W^{Q}_{ij})^2\norm{X_{j:}}_2^2\norm{K_{i:}}_2^2,
 \qquad
 S^K_{ij}=(W^{K}_{ij})^2\norm{X_{j:}}_2^2\norm{Q_{i:}}_2^2.
\end{equation}
Appendix~\ref{app:row_equivalence} explains when these row-wise rankings coincide with Wanda and why multiple sequences or causal masking can break that equivalence.

For \algname{QK-Wanda}, summing all pairs in \cref{eq:causal_curvature} yields
\begin{equation}
 G_U^{Q^{(h)}}=\frac1N\sum_{s=1}^{N}
 (X^{(s)}X^{(s)\top})\otimes
 (K^{(g,s)}K^{(g,s)\top}).
 \label{eq:u_curvature}
\end{equation}
The factorization is sequence-local; a sum of products need not equal a product of sums.

\subsection{\algname{QK-Wanda-MR}: causal reconstruction including rotary embeddings}
\label{app:rope}

The column vectors \(q^{(h,s)}_{t}\) and \(k^{(g,s)}_{u}\), defined in \cref{app:causal_quadratic}, contain all \(m\) coordinates of one query/key token. Let \(R_t^{(s)}\in\R^{m\times m}\) be the model's fixed rotary position map for token \(t\) in sequence \(s\). A tilde marks an activation after this map:
\(\widetilde q^{(h,s)}_{t}=R_t^{(s)}q^{(h,s)}_{t}\) and
\(\widetilde k^{(g,s)}_{u}=R_u^{(s)}k^{(g,s)}_{u}\).
The same maps are used for the reference and pruned weights. Let \(Z^{(h,s)}\) be the \(T_s\times T_s\) matrix of unscaled dot products \emph{after} RoPE; this distinguishes it from the pre-RoPE matrix \(B^{(h,s)}\). For \(h\in\mathcal H_g\), the entries are \(Z^{(h,s)}_{tu}=\widetilde q^{(h,s)\top}_{t}\widetilde k^{(g,s)}_{u}\). Its pruned counterpart \(\widehat Z^{(h,s)}\) uses the same inputs and rotary maps after weight deletion. Summing over all query heads gives
\begin{equation}
 \mathcal L_R=\frac1N\sum_{s=1}^{N}\sum_{h=1}^{H_Q}
 \norm{C^{(s)}\odot(\widehat Z^{(h,s)}-Z^{(h,s)})}_F^2.
 \label{eq:rope_objective}
\end{equation}
The loss \(\mathcal L_R\) uses the same unscaled convention as \(\mathcal L_C\). Token indices range from \(1\) to \(T_s\). The mask \(C^{(s)}\) retains causal validity; \algname{QK-Wanda-MR} adds the rotary transform rather than replacing the causal mask.

\begin{proposition}[Exact RoPE-aware single-weight deletion]
\label{prop:rope_scores}
Under the linear QK assumptions of \cref{prop:causal_scores}, and fixed rotary maps as above, add the suffix \(R\) to the causal scores \(\hat S\) to indicate that RoPE is included. For a query head \(h\in\mathcal H_g\), with \(e_i\) selecting raw coordinate \(i\), the exact query deletion cost and the grouped key deletion cost are
\begin{align}
 \hat S^{Q^{(h)}}_{ij,R}
 &=\frac{(W^{Q^{(h)}}_{ij})^2}{N}\sum_{s=1}^{N}\sum_{t,u} C^{(s)}_{tu}
 (X^{(s)}_{jt})^2
 \left[(R_t^{(s)}e_i)^\top\widetilde k^{(g,s)}_{u}\right]^2,
 \label{eq:rope_q}\\
 \hat S^{K^{(g)}}_{ij,R}
 &=\frac{(W^{K^{(g)}}_{ij})^2}{N}\sum_{s=1}^{N}\sum_u(X^{(s)}_{ju})^2
 \sum_{h\in\mathcal H_g}\sum_t C^{(s)}_{tu}
 \left[\widetilde q^{(h,s)\top}_{t}R_u^{(s)}e_i\right]^2.
 \label{eq:rope_k}
\end{align}
\end{proposition}

\begin{proof}
Deleting a raw query weight \(w=W^{Q^{(h)}}_{ij}\) gives
\(\Delta q^{(h,s)}_{t}=-wX^{(s)}_{jt}e_i\). Since \(R_t^{(s)}\) is fixed and linear,
\begin{equation}
 \Delta\widetilde q^{(h,s)}_{t}
 =R_t^{(s)}\Delta q^{(h,s)}_{t}
 =-wX^{(s)}_{jt}R_t^{(s)}e_i.
\end{equation}
Keys remain fixed. Subtracting the reference dot product therefore yields
\begin{equation}
 \Delta Z^{(h,s)}_{tu}
 =-wX^{(s)}_{jt}
       (R_t^{(s)}e_i)^\top\widetilde k^{(g,s)}_{u}.
\end{equation}
Expand \(\mathcal L_R\) as \(\frac1N\sum_{s=1}^{N}\sum_{t,u}(C_{tu}^{(s)}\Delta Z^{(h,s)}_{tu})^2\). Use \(C_{tu}^2=C_{tu}\), and substitute the preceding identity to obtain \cref{eq:rope_q} exactly.

For a raw key deletion \(w=W^{K^{(g)}}_{ij}\),
\begin{equation}
 \Delta\widetilde k^{(g,s)}_{u}=-wX^{(s)}_{ju}R_u^{(s)}e_i,
 \qquad
 \Delta Z^{(h,s)}_{tu}
 =-wX^{(s)}_{ju}
       \widetilde q^{(h,s)\top}_{t}R_u^{(s)}e_i
\end{equation}
for \(h\in\mathcal H_g\), and zero for other heads. Squaring each valid entry and summing over those heads gives \cref{eq:rope_k}. The proof requires no small-perturbation or softmax linearization. Fixed biases remain in the opposite-side activations.
\end{proof}

For the evaluated Llama architecture, the actual pre-softmax logits on valid positions are \(Z/\sqrt{m}\). Their squared reconstruction loss is therefore \(\mathcal L_R/m\); this common factor leaves the single-deletion ranking unchanged. It does not establish a full-mask optimum. It also makes no claim for additional weight-dependent nonlinear QK transformations.

\subsection{Pairwise second-moment accumulation for \algname{QK-Wanda-MR}}

RoPE acts within disjoint pairs of coordinates. For an even head width \(m\), let \(a=1,\ldots,m/2\) index those pairs. In the evaluated Hugging Face Llama convention, pair \(a\) contains coordinates \(\mathcal P_a=(a,a+m/2)\) in that order. Indexing a vector by \([\mathcal P_a]\) selects those two entries. The \(2\times2\) second-moment matrices of the rotated key/query subvectors are
\begin{align}
 V_{K,a}^{(g,s)}(t)&=\sum_{u\le t}
       \widetilde k^{(g,s)}_{u}[\mathcal P_a]
       \bigl(\widetilde k^{(g,s)}_{u}[\mathcal P_a]\bigr)^\top,\\
 V_{Q,a}^{(g,s)}(u)&=\sum_{h\in\mathcal H_g}\sum_{t\ge u}
       \widetilde q^{(h,s)}_{t}[\mathcal P_a]
       \bigl(\widetilde q^{(h,s)}_{t}[\mathcal P_a]\bigr)^\top.
\end{align}
These are sums of outer products, without subtracting a mean. Their off-diagonal entries retain products between the two coordinates. Let \(R_{t,a}^{(s)}\) be the \(2\times2\) part of the rotary map acting on pair \(a\). If raw coordinate \(i\) is element \(b\in\{1,2\}\) of that pair, let \(e_b\in\R^2\) select it. Its rotated direction is \(R_{t,a}^{(s)}e_b\), so the query interaction is
\begin{equation}
 \sum_{u\le t}\left[(R_{t,a}^{(s)}e_b)^\top
       \widetilde k^{(g,s)}_{u}[\mathcal P_a]\right]^2
 =[R_{t,a}^{(s)\top}V_{K,a}^{(g,s)}(t)R_{t,a}^{(s)}]_{bb}.
\end{equation}
For keys, the same argument gives
\begin{equation}
 \sum_{h\in\mathcal H_g}\sum_{t\ge u}
 \left[\bigl(\widetilde q^{(h,s)}_{t}[\mathcal P_a]\bigr)^\top R_{u,a}^{(s)}e_b\right]^2
 =[R_{u,a}^{(s)\top}V_{Q,a}^{(g,s)}(u)R_{u,a}^{(s)}]_{bb}.
\end{equation}
Multiplying by the input square at \(t\) or \(u\), then averaging over the \(N\) sequences and multiplying by the deleted weight's \(w^2\), recovers \cref{eq:rope_q,eq:rope_k}.

To make the two-coordinate calculation explicit, consider one pair at one position. Write its rotation angle as \(\theta\), with \(c_\theta=\cos\theta\) and \(s_\theta=\sin\theta\). Its rotation matrix is
\begin{equation}
 R=\begin{pmatrix}c_\theta&-s_\theta\\s_\theta&c_\theta\end{pmatrix}.
\end{equation}
Let \(V\) be either symmetric \(2\times2\) second-moment matrix above, with entries \(V_{11},V_{12}=V_{21},V_{22}\). Its two rotated diagonal entries are
\begin{align}
 [R^\top VR]_{11}&=c_\theta^2V_{11}+2c_\theta s_\theta V_{12}+s_\theta^2V_{22},\\
 [R^\top VR]_{22}&=s_\theta^2V_{11}-2c_\theta s_\theta V_{12}+c_\theta^2V_{22}.
\end{align}
The \(V_{12}\) terms show why cross-coordinate products must be retained. Setting every rotary map to the identity recovers \algname{QK-Wanda-M}.

\FloatBarrier
\section{Causal and RoPE-Aware Ablation at 8192 Calibration Tokens}
\label{app:variant_ablation}

We compared causal pre-RoPE QK-Wanda-M with causal-plus-RoPE QK-Wanda-MR on Llama-3.2-1B/3B and Llama-3.1-8B at 50\% and 80\% QK sparsity. Both variants used separate whole-matrix budgets, one-shot masks, and identical calibration inputs for each seed. This protocol used 256 C4 sequences of length 32 and calibration seeds 0, 1, and 2; PPL evaluation used length 2048. The calibration budget and context differ from the main study, so these comparisons are not included in its win/loss counts.

\begin{table}[H]
\centering\small
\setlength{\tabcolsep}{4pt}
\caption{Causal pre-RoPE (M) and causal-plus-RoPE (MR) with separate matrix budgets and \(256\times32\) calibration tokens. PPL is the mean and sample standard deviation over three calibration seeds. Post-RoPE QK error is the pooled relative squared error (\%), averaged over seeds. Bold marks lower displayed means within each checkpoint and sparsity, including ties.}
\label{tab:variant_ablation}
\begin{tabular}{lccccc}
\toprule
Model & Sparsity & Score & WikiText-2 PPL & C4 PPL & \shortstack{Post-RoPE\\QK error (\%)}\\
\midrule
Llama-3.2-1B & 50\% & M & \(\mathbf{11.1933\,\pm\,0.0196}\) & \(\mathbf{17.8518\,\pm\,0.0687}\) & \(0.20\)\\
 &  & MR & \(12.2233\,\pm\,0.0195\) & \(21.0389\,\pm\,0.0728\) & \(\mathbf{0.18}\)\\
\addlinespace[2pt]
Llama-3.2-1B & 80\% & M & \(\mathbf{192.6108\,\pm\,0.4938}\) & \(\mathbf{261.1174\,\pm\,2.7507}\) & \(\mathbf{2.88}\)\\
 &  & MR & \(232.3745\,\pm\,0.8535\) & \(269.6802\,\pm\,1.5493\) & \(3.04\)\\
\addlinespace[2pt]
Llama-3.2-3B & 50\% & M & \(\mathbf{8.0829\,\pm\,0.0012}\) & \(\mathbf{11.8228\,\pm\,0.0010}\) & \(0.12\)\\
 &  & MR & \(8.1256\,\pm\,0.0014\) & \(11.8965\,\pm\,0.0009\) & \(\mathbf{0.11}\)\\
\addlinespace[2pt]
Llama-3.2-3B & 80\% & M & \(\mathbf{40.9509\,\pm\,0.4876}\) & \(\mathbf{51.7862\,\pm\,0.6025}\) & \(\mathbf{1.23}\)\\
 &  & MR & \(44.5742\,\pm\,0.4372\) & \(55.7155\,\pm\,0.1331\) & \(1.25\)\\
\addlinespace[2pt]
Llama-3.1-8B & 50\% & M & \(\mathbf{6.4453\,\pm\,0.0002}\) & \(\mathbf{9.8939\,\pm\,0.0010}\) & \(0.12\)\\
 &  & MR & \(6.4663\,\pm\,0.0008\) & \(9.9308\,\pm\,0.0016\) & \(\mathbf{0.11}\)\\
\addlinespace[2pt]
Llama-3.1-8B & 80\% & M & \(\mathbf{14.0276\,\pm\,0.0957}\) & \(\mathbf{23.4658\,\pm\,0.0695}\) & \(1.18\)\\
 &  & MR & \(15.3611\,\pm\,0.0492\) & \(25.8150\,\pm\,0.0377\) & \(\mathbf{1.17}\)\\
\addlinespace[2pt]
\bottomrule
\end{tabular}
\end{table}

The causal pre-RoPE variant had lower mean PPL on both datasets in all six model--sparsity pairs. The RoPE-aware variant had lower post-RoPE QK error in four of six pairs; its advantage was consistent at 50\% but mixed at 80\%. Thus, including RoPE in the reconstruction target did not improve PPL in this ablation. These results do not establish the same ordering for shared budgets or the main calibration protocol.

The held-out diagnostics used 16 C4 sequences of length 512, seed 20260810, starting from validation document row 1100. Post-RoPE QK and block-output errors use the pooled squared-error ratio in \cref{eq:normalized_error}. Causal pre-RoPE scoring also produced lower mean same-input block-output error in all six pairs and lower rollout block-output error in four of six. As with the main QK-error diagnostic, these stages need not rank masks in the same order as PPL. The unmasked variant is derived in Appendix~\ref{app:unmasked}; this ablation compares only the two causal variants.

\paragraph{Variation across calibration seeds.}
For each of the two QK scoring variants, three checkpoints, and two sparsities, we computed the sample standard deviation of PPL over calibration seeds 0, 1, and 2. Across the six model--variant configurations at 50\%, the standard deviations ranged from 0.00015 to 0.01964 on WikiText-2 and from 0.00089 to 0.07280 on C4. At 80\%, the respective ranges were 0.04917--0.85345 and 0.03766--2.75072. These are within-configuration seed variations, not variation across models. The different calibration and allocation protocol prevents treating these ranges as uncertainty bounds for the main single-seed comparisons.

\FloatBarrier
\section{Experimental Details}
\label{app:protocol}

This appendix specifies the data, evaluation procedures, and implementation conventions used in the experiments. We distinguish pruning calibration from the reconstruction diagnostic and downstream evaluation so that their results can be interpreted under the appropriate input and aggregation rules.

\subsection{Calibration and perplexity data}

We sample calibration text from the English C4 training shard \nolinkurl{en/c4-train.00000-of-01024.json.gz}, using each checkpoint's native tokenizer. Each base sample is a contiguous 2048-token window from one document. The main experiments use seed 0 and 256 windows, giving 524,288 tokens. Splitting these windows into length-64 segments gives 8192 independent calibration sequences. Positions restart in each segment; there is no padding, packing, or attention across segments. Every method and sparsity uses the same sampled token manifest within a checkpoint. Token IDs are not shared between models with different tokenizers.

Appendix~\ref{app:calibration} varies sample count through nested subsets of the base windows and context length through different segmentations of the same 524,288-token pool. Every tested length \(T\) divides 2048, so all segmentations retain exactly the same tokens and use \(524{,}288/T\) independent sequences. Calibration batches contain \(2048/T\) segments, keeping 2048 tokens per forward pass; PPL evaluation stays at length 2048. Causal sums are computed within each sequence before batch aggregation, while QK-Wanda uses all pairs within each segment. The original Llama-3.1-8B \(T=2048\) comparison reuses the identical \(N=256\) sample-count configuration. Appendix~\ref{app:variant_ablation} uses a separate three-seed calibration protocol.

WikiText-2 evaluation used the \texttt{wikitext-2-raw-v1} test split, joining text entries with two newline characters before tokenization. C4 evaluation used the first 1100 documents of \nolinkurl{en/c4-validation.00000-of-00008.json.gz}, joined by spaces and truncated to at most \(256\times2048\) tokens. We evaluated complete, non-overlapping 2048-token windows and discarded the final incomplete window. These same evaluation sets informed the calibration-sensitivity comparisons on Llama-3.1-8B and Llama-2-7B; no document range was reserved exclusively for development. The resulting token sequences were identical across pruning configurations for each checkpoint and tokenizer. We used batch size one, disabled the KV cache, and computed full-vocabulary next-token cross-entropy. To reduce memory use, we evaluated the vocabulary projection in 128-token chunks, preserving FP32 logits and the same cross-entropy calculation.

\subsection{Zero-shot evaluation}
\label{app:zero_shot_protocol}

We evaluated the dense base checkpoints of TinyLlama-1.1B, Llama~2 7B/13B/70B, Llama-3.2-1B/3B, Llama-3.1-8B/70B, and Qwen2.5 0.5B/1.5B/3B/7B/14B/32B/72B, together with QK-only pruning at 30\%, 50\%, and 80\% sparsity. The comparisons include row-wise \algname{Wanda}, \algname{QK-Wanda} with separate and shared budgets, and the labeled scoring variants in the result tables. For TinyLlama and Llama~2 7B/13B, we also evaluated full-block pruning with Wanda and the QK-Wanda-M hybrid at 30\%, 50\%, and 80\% (\cref{tab:full_block_masked_history}). The six QK-Wanda hybrid configurations at 30\% and 50\% appear in \cref{tab:full_block}. We applied the exact saved masks from the corresponding PPL experiments to the same pinned checkpoints, without repeating pruning or updating retained weights. In the QK-only setting, all non-QK parameters remained unchanged. Full-block masks additionally pruned value, output, and MLP projections; embeddings, normalization parameters, biases, and the language-model head remained unchanged. We reused each checkpoint's dense reference across pruning scopes. Appendix~\ref{app:matrix_control} additionally evaluates nine matrix-budget Wanda configurations on the three largest checkpoints under this same task protocol.

We used LM Evaluation Harness v0.4.12 \citep{evalharness}, with no demonstrations or task subsampling, maximum context length 2048, batch size 4, and evaluation seed 1234. We used scaled dot-product attention throughout. The original 0.5B--32B comparisons used FP16 weights on A100 GPUs: Qwen2.5-32B used two GPUs and the smaller models used one. The additional QK-Wanda and scoring-variant evaluations ran on H200 GPUs with the same task protocol. The Llama 70B checkpoints used FP16 weights on H200 GPUs, and Qwen2.5-72B used BF16 weights on two H200 GPUs under its validated numerical protocol. We evaluated the base models directly, without chat templates or an additional beginning-of-sequence token. The suite comprised BoolQ \citep{clark2019boolq}, RTE \citep{wang2018glue}, HellaSwag \citep{zellers2019hellaswag}, WinoGrande \citep{sakaguchi2020winogrande}, ARC-Easy and ARC-Challenge \citep{clark2018arc}, and OpenBookQA \citep{mihaylov2018openbookqa}. We used the validation splits for BoolQ, RTE, HellaSwag, and WinoGrande (3270, 277, 10042, and 1267 examples), and the test splits for ARC-Easy, ARC-Challenge, and OpenBookQA (2376, 1172, and 500 examples). Every configuration used identical evaluation examples and prompts. We used accuracy for BoolQ, RTE, and WinoGrande and length-normalized accuracy for HellaSwag, both ARC tasks, and OpenBookQA. The reported zero-shot score was their unweighted arithmetic mean, expressed as a percentage, computed before rounding individual task scores.

\subsection{Reconstruction metrics}
\label{app:reconstruction_results}

We sampled 16 held-out C4 validation sequences of length 512 with seed 20260810, restricting sampling to document rows at or after 1100. These documents were disjoint from the C4 perplexity subset and the calibration training split. Each checkpoint used its native tokenizer, with identical held-out tokens across methods and sparsities. We evaluated the saved masks from the PPL experiments, using calibration seed 0. A dense forward pass supplied the same reference inputs to the dense and pruned QK projections in every block.

\begin{samepage}
\paragraph{Block-mean relative squared QK error.}
Let \(D\) be the number of transformer blocks, indexed by \(\ell\). We add the subscript \(\ell\) to the head/sequence notation from Appendix~\ref{app:proof_c}: the unscaled pre-RoPE product is
\(B^{(h,s)}_{\ell}=Q^{(h,s)\top}_{\ell}K^{(g,s)}_{\ell}\), where \(h\in\mathcal H_g\) pairs heads within block \(\ell\).
The hatted product \(\widehat B^{(h,s)}_{\ell}\) applies both saved QK masks on the same dense block input, retaining biases. The diagnostic includes all query--key token pairs, without a causal mask or centering. The table column \emph{Rel. QK error (\%)} reports
\begin{equation}
 E_{\mathrm{QK}}=\frac{100}{D}\sum_{\ell=1}^{D}
 \frac{\sum_{s,h}\norm{\widehat B^{(h,s)}_{\ell}-B^{(h,s)}_{\ell}}_F^2}
 {\sum_{s,h}\norm{B^{(h,s)}_{\ell}}_F^2}.
 \label{eq:relative_qk_error}
\end{equation}
\end{samepage}
Here \(\sum_{s,h}\) covers all 16 evaluation sequences and all \(H_Q\) query heads. Within each block, we summed squared errors and dense reference energies over these indices, including all query heads that share a key head under GQA, then divided. We averaged these relative errors equally across blocks, so blocks with larger QK magnitudes did not dominate the aggregate. All block denominators were positive. We took no square root: a value of 1 denotes mean relative \emph{squared} error of 1\%, rather than relative norm error of 1\%. The common attention scale cancels within each block's ratio.

This diagnostic evaluates complete masks, including interactions among deletions, rather than summing individual deletion scores. It isolates each block's QK reconstruction on dense inputs; it does not include error propagated from previously pruned blocks. We retained FP16 projection outputs (BF16 for Qwen2.5-72B), computed QK dot products in FP32 with TF32 disabled, and accumulated squared-error sums in FP64. RoPE remained enabled in the dense forward pass that supplied block inputs; the measured products were taken before RoPE.

\paragraph{Cross-checkpoint aggregation.}
For QK-Wanda, the abstract reports the mean of per-checkpoint relative reductions, not the relative reduction of mean errors. For checkpoint \(c\) at sparsity \(\rho\), let \(E^{\mathrm{QK}}_{c,\rho}\) and \(E^{\mathrm{Wanda}}_{c,\rho}\) be the block-mean errors in \cref{eq:relative_qk_error}. We computed
\begin{equation}
 \overline\Delta_{\rho}=\frac{100}{15}\sum_{c=1}^{15}
 \left(1-\frac{E^{\mathrm{QK}}_{c,\rho}}{E^{\mathrm{Wanda}}_{c,\rho}}\right),
 \label{eq:cross_checkpoint_reduction}
\end{equation}
which gives 60.35\% at 50\% and 44.70\% at 80\% for shared-budget QK-Wanda. All reference errors are positive at these sparsities. \Cref{sec:other_families} discusses the relation to downstream metrics, and Appendix~\ref{app:unmasked_aggregate_counts} gives the base-method descriptive counts; Appendix~\ref{app:aggregate_counts} preserves the separately labeled masked-variant counts.

\paragraph{Additional reconstruction stages.}
For the ablation in Appendix~\ref{app:variant_ablation}, let \(Y_\ell^{(s)}\) be the reference output of a reconstruction stage in block \(\ell\), and \(\widehat Y_\ell^{(s)}\) its pruned counterpart. Pooling over all blocks and evaluation sequences gives
\begin{equation}
 E_Y=\frac{\sum_{s,\ell}\norm{\widehat Y_\ell^{(s)}-Y_\ell^{(s)}}_F^2}
 {\sum_{s,\ell}\norm{Y_\ell^{(s)}}_F^2},
 \label{eq:normalized_error}
\end{equation}
For post-RoPE QK logits, \(Y_\ell^{(s)}\) vertically stacks the matrices \(C^{(s)}\odot Z^{(h,s)}_{\ell}/\sqrt m\) over query heads, and \(\widehat Y_\ell^{(s)}\) uses their pruned counterparts. Here \(Z^{(h,s)}_{\ell}\) adds the block index to the post-RoPE products in \cref{eq:rope_objective}. This definition includes exactly the causally allowed entries of every head. For full block outputs, the \emph{same-input} setting supplies the dense reference input to each pruned block; the \emph{rollout} setting instead supplies outputs from preceding pruned blocks. These stage metrics use \(100E_Y\), without a square root, and are distinct from the equal-block average in \cref{eq:relative_qk_error}.

\subsection{Parameter scope of QK-only pruning}
\label{app:parameter_scope}

We counted unique model parameters from the pinned checkpoint configurations, instantiating model shapes without loading weight values and counting tied embeddings once. QK counts include only projection weights; biases remain dense. Let \(d\) be the hidden width, \(f\) the MLP intermediate width, \(m\) the head width, and \(H_Q,H_K\) the query/key head counts. For one block, the QK weights and all seven projection weights number
\begin{equation*}
 n_{QK}=d\,m(H_Q+H_K),\qquad
 n_{\mathrm{linear}}=2d\,m(H_Q+H_K)+3df.
\end{equation*}
The model-level QK count is \(n_{QK}\) times the block count. The total-parameter denominator additionally includes embeddings, any untied language-model head, norms, and biases. At QK-only sparsity \(\rho\), the removed fraction of total parameters is \(\rho\,n_{QK,\mathrm{total}}/n_{\mathrm{total}}\), up to integer rounding. This measures zeroed weights, not storage or inference speedup.

\begin{table}[H]
\centering\small
\caption{Scope of the QK-only experiments. B and M denote billions and millions of parameters. The last two columns use different denominators: the complete model and its seven block projection weight matrices.}
\label{tab:parameter_scope}
\begin{tabular}{@{}lrrrr@{}}
\toprule
Checkpoint & Total (B) & QK (M) & QK / total (\%) & QK / block linear (\%)\\
\midrule
TinyLlama-1.1B & 1.100 & 103.8 & 9.44 & 10.71\\
Llama-2-7B & 6.738 & 1073.7 & 15.93 & 16.58\\
Llama-2-13B & 13.016 & 2097.2 & 16.11 & 16.53\\
Llama-2-70B & 68.977 & 6039.8 & 8.76 & 8.82\\
Llama-3.2-1B & 1.236 & 83.9 & 6.79 & 8.62\\
Llama-3.2-3B & 3.213 & 352.3 & 10.97 & 12.50\\
Llama-3.1-8B & 8.030 & 671.1 & 8.36 & 9.62\\
Llama-3.1-70B & 70.554 & 6039.8 & 8.56 & 8.82\\
Qwen2.5-0.5B & 0.494 & 22.0 & 4.46 & 6.15\\
Qwen2.5-1.5B & 1.544 & 77.1 & 4.99 & 5.88\\
Qwen2.5-3B & 3.086 & 169.9 & 5.50 & 6.12\\
Qwen2.5-7B & 7.616 & 411.0 & 5.40 & 6.30\\
Qwen2.5-14B & 14.770 & 1509.9 & 10.22 & 11.43\\
Qwen2.5-32B & 32.764 & 2013.3 & 6.14 & 6.45\\
Qwen2.5-72B & 72.706 & 6039.8 & 8.31 & 8.60\\
 
\bottomrule
\end{tabular}
\end{table}

\subsection{Implementation and reproducibility}
\label{app:implementation}

The original A100 experiments used Python 3.10.14, PyTorch 2.4.1 with CUDA 12.1, Transformers 4.45.2, Accelerate 0.34.2, datasets 2.21.0, and LM Evaluation Harness 0.4.12. Subsequent H200 jobs used Python 3.12.3. Smaller original checkpoints ran on one A100, Qwen2.5-32B on two, and retained Llama~70B configurations on four. The additional QK-Wanda and scoring-variant runs used H200 GPUs. All common QK reconstruction measurements used H200, with two GPUs for 32B/70B/72B. Each measurement retains its hardware and frozen-source provenance.

Model weights are FP16 except on Qwen2.5-72B, which uses BF16 for every method and its dense reference to avoid nonfinite FP16 activations. Pruning statistics are accumulated in FP32. \Cref{tab:checkpoint_revisions} lists the pinned model and tokenizer revisions.
\begin{table}[H]
\centering\small
\caption{Hugging Face repository identifiers and pinned revisions used for both model and tokenizer.}
\label{tab:checkpoint_revisions}
\begin{tabular}{@{}ll@{}}
\toprule
\multicolumn{2}{l}{\nolinkurl{meta-llama/Llama-3.1-8B}}\\
Revision & \nolinkurl{d04e592bb4f6aa9cfee91e2e20afa771667e1d4b}\\
\multicolumn{2}{l}{\nolinkurl{meta-llama/Llama-3.2-1B}}\\
Revision & \nolinkurl{4e20de362430cd3b72f300e6b0f18e50e7166e08}\\
\multicolumn{2}{l}{\nolinkurl{meta-llama/Llama-3.2-3B}}\\
Revision & \nolinkurl{13afe5124825b4f3751f836b40dafda64c1ed062}\\
\multicolumn{2}{l}{\nolinkurl{TinyLlama/TinyLlama-1.1B-intermediate-step-1431k-3T}}\\
Revision & \nolinkurl{59f6f375b26bde864a6ca194a9a3044570490064}\\
\multicolumn{2}{l}{\nolinkurl{meta-llama/Llama-2-7b-hf}}\\
Revision & \nolinkurl{01c7f73d771dfac7d292323805ebc428287df4f9}\\
\multicolumn{2}{l}{\nolinkurl{meta-llama/Llama-2-13b-hf}}\\
Revision & \nolinkurl{5c31dfb671ce7cfe2d7bb7c04375e44c55e815b1}\\
\multicolumn{2}{l}{\nolinkurl{meta-llama/Llama-3.1-70B}}\\
Revision & \nolinkurl{349b2ddb53ce8f2849a6c168a81980ab25258dac}\\
\multicolumn{2}{l}{\nolinkurl{meta-llama/Llama-2-70b-hf}}\\
Revision & \nolinkurl{3aba440b59558f995867ba6e1f58f21d0336b5bb}\\
\multicolumn{2}{l}{\nolinkurl{Qwen/Qwen2.5-0.5B}}\\
Revision & \nolinkurl{060db6499f32faf8b98477b0a26969ef7d8b9987}\\
\multicolumn{2}{l}{\nolinkurl{Qwen/Qwen2.5-1.5B}}\\
Revision & \nolinkurl{8faed761d45a263340a0528343f099c05c9a4323}\\
\multicolumn{2}{l}{\nolinkurl{Qwen/Qwen2.5-3B}}\\
Revision & \nolinkurl{3aab1f1954e9cc14eb9509a215f9e5ca08227a9b}\\
\multicolumn{2}{l}{\nolinkurl{Qwen/Qwen2.5-7B}}\\
Revision & \nolinkurl{d149729398750b98c0af14eb82c78cfe92750796}\\
\multicolumn{2}{l}{\nolinkurl{Qwen/Qwen2.5-14B}}\\
Revision & \nolinkurl{97e1e76335b7017d8f67c08a19d103c0504298c9}\\
\multicolumn{2}{l}{\nolinkurl{Qwen/Qwen2.5-32B}}\\
Revision & \nolinkurl{1818d35814b8319459f4bd55ed1ac8709630f003}\\
\multicolumn{2}{l}{\nolinkurl{Qwen/Qwen2.5-72B}}\\
Revision & \nolinkurl{efba10c8e54e91e0d9570ab5f7b51a958474d4cb}\\
\bottomrule
\end{tabular}
\end{table}

\paragraph{Integer counts and ties.}
Our row-wise Wanda implementation removed \(\lceil\rho d\rceil\) weights per output row of width \(d\); the matched row-wise QK-Wanda-M control used the same convention. Separate matrix budgets removed \(\lceil\rho n_Q\rceil\) and \(\lceil\rho n_K\rceil\) weights, and shared budgets removed \(\lceil\rho(n_Q+n_K)\rceil\). This upward rounding is our implementation convention, not a property attributed to the original Wanda algorithm. The retained Llama-3.1-8B QK sweep used floor at the matrix/pair scope; its shared counts coincide with ceiling at the tested sparsities, while separate counts differ by at most one per projection. The 50\% calibration studies also used floor, which equals ceiling for their even row widths and matrix sizes. The 70B/72B matrix-Wanda controls matched separate QK-Wanda-M counts exactly.

Scores were sorted in ascending order with a stable sort. Ties within a row followed increasing column index; matrix ties followed row-major order; shared ties used all Q entries followed by all K entries, each in row-major order. No independent per-projection or per-head score normalization was applied. Frozen source-manifest hashes, rather than the mutable working tree, identify executed code in the archived study plans (\cref{tab:implementation_revisions}). Token manifests, exact deletion counts, masks, and metric files were retained.
\begin{table}[H]
\centering\small
\caption{Frozen source-manifest SHA256 identifiers recorded in study plans; each hash continues on the following line. Retained runs also record mask and result hashes.}
\label{tab:implementation_revisions}
\begin{tabular}{@{}ll@{}}
\toprule
Study bundle & Source-manifest SHA256\\
\midrule
\nolinkurl{qk_budget} & \nolinkurl{91360064ecb2be26de3ffe3450262f6f}\\
 & \nolinkurl{3f5961c4569c460920b14e3c1b6b788f}\\
\nolinkurl{qk_budget_1b} & \nolinkurl{10987a5ad0b19d3fbdea2dc1f83c8d6a}\\
 & \nolinkurl{339c0460ab8e001d4e152e20d62fdc5b}\\
\nolinkurl{qk_budget_3b} & \nolinkurl{291a3cdcb13802ca1a7f9a8e5770eaf8}\\
 & \nolinkurl{90d80f8de1d1b03a54e352da645d43a0}\\
\nolinkurl{qk_budget_llama2} & \nolinkurl{a0119c61deca3691885b1b8e1ca90556}\\
 & \nolinkurl{6e762376e506277279de1a280a078a82}\\
\nolinkurl{qk_budget_qwen25} & \nolinkurl{400df4ddec5a87dbe4bd99c1c6511e03}\\
 & \nolinkurl{957f361f4e07f3f6ececf1e5a3328d13}\\
\nolinkurl{qk_budget_llama70b} & \nolinkurl{5289268b6f2d339bd69bcf257e90a014}\\
 & \nolinkurl{54a4c5ba65966e3f019d5cb3a06e2044}\\
\nolinkurl{qk_budget_qwen72b} & \nolinkurl{10079ee05d763a95fd651cda09416c57}\\
 & \nolinkurl{08d945df5317c1d37f2b00abe6d2bcee}\\
\bottomrule
\end{tabular}
\end{table}

\paragraph{Artifact availability.}
The pruning and evaluation implementation, tests, pinned checkpoint configurations, full-precision results, and scripts for reproducing experiments and regenerating tables and figures are available at \textcolor{codeurl}{\url{https://github.com/vectozavr/qk-wanda}}. The repository README gives environment setup and reproduction commands. Model weights, dataset text, and the large collection of original binary masks are not redistributed; masks can be regenerated, and their recorded hashes and allocation statistics are included. File-integrity manifests document the archived results and their provenance.

\FloatBarrier
\subsection{Shared QK allocation across transformer blocks}
\label{app:shared_allocation}

We examined the shared-budget masks at 50\% combined QK sparsity for all 15 checkpoints. Each run used QK-Wanda with calibration seed 0 and 8192 C4 training sequences of length 64 tokenized by the checkpoint's native tokenizer. Each block had its own shared QK budget: the fractions removed from \(W^Q\) and \(W^K\) could differ, while exactly 50\% of their combined weights were removed. The profiles below use the same masks as the PPL and QK-error comparisons. Lines connect every measured block; markers are shown every fourth block, with different starting blocks for each model to reduce overlap.

\begin{figure}[htbp]
\centering
\includegraphics[width=\textwidth]{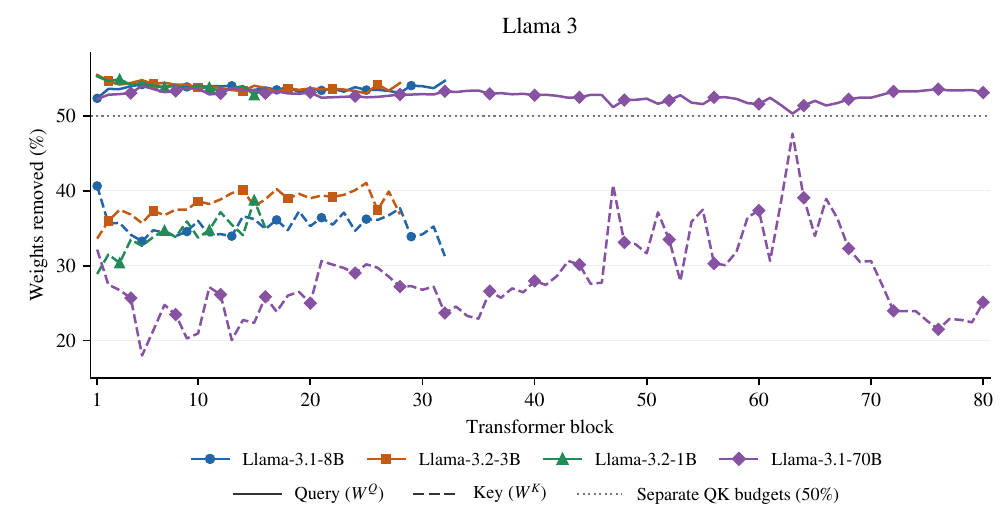}
\caption{QK-Wanda allocation at a shared 50\% budget on Llama~3. Solid lines show query-weight removal and dashed lines key-weight removal; colors and markers identify checkpoints. The dotted line marks 50\% removal in each projection.}
\label{fig:qk_shared_allocation}
\end{figure}

\paragraph{Llama 3: a consistent preference for retaining key weights.}
All four checkpoints removed a larger fraction of query than key weights in every block. Mean query/key removal rates were 53.66\%/35.35\% on Llama-3.1-8B, 53.91\%/38.28\% on Llama-3.2-3B, 54.00\%/34.01\% on Llama-3.2-1B, 52.70\%/28.44\% on Llama-3.1-70B. On 70B, key removal ranged from 18.01\% to 47.61\% across blocks. The allocation varies with depth rather than following a fixed schedule.

\paragraph{Accounting for unequal projection sizes.}
The Q:K parameter-count ratios \(n_Q/n_K\) are 4:1 on 8B and 1B, 3:1 on 3B, and 8:1 on 70B. Let \(n_Q,n_K\) denote the projection weight counts and \(\rho_Q,\rho_K\) their removal fractions. With target sparsity \(\rho\), the ceiling convention in Appendix~\ref{app:implementation} gives
\begin{equation}
 \frac{n_Q\rho_Q+n_K\rho_K}{n_Q+n_K}=\frac{\lceil\rho(n_Q+n_K)\rceil}{n_Q+n_K}.
 \label{eq:shared_budget_fraction}
\end{equation}
At the 50\% budgets used here, the removal count is an integer and this reduces to
\[
 \frac{n_Q\rho_Q+n_K\rho_K}{n_Q+n_K}=50\%.
\]
Consequently, increasing Q sparsity by one percentage point requires decreasing K sparsity by \(n_Q/n_K\) points. This constraint explains the larger amplitude of the K curves; it must be considered when comparing checkpoints. 

\FloatBarrier
\paragraph{Llama 2 and TinyLlama.}
\Cref{fig:qk_shared_allocation_llama2} extends the same analysis to Llama~2 7B/13B/70B and TinyLlama 1.1B. The combined QK removal fraction was exactly 50\% in every block, under the same calibration protocol described above.

\begin{figure}[htbp]
\centering
\includegraphics[width=\textwidth]{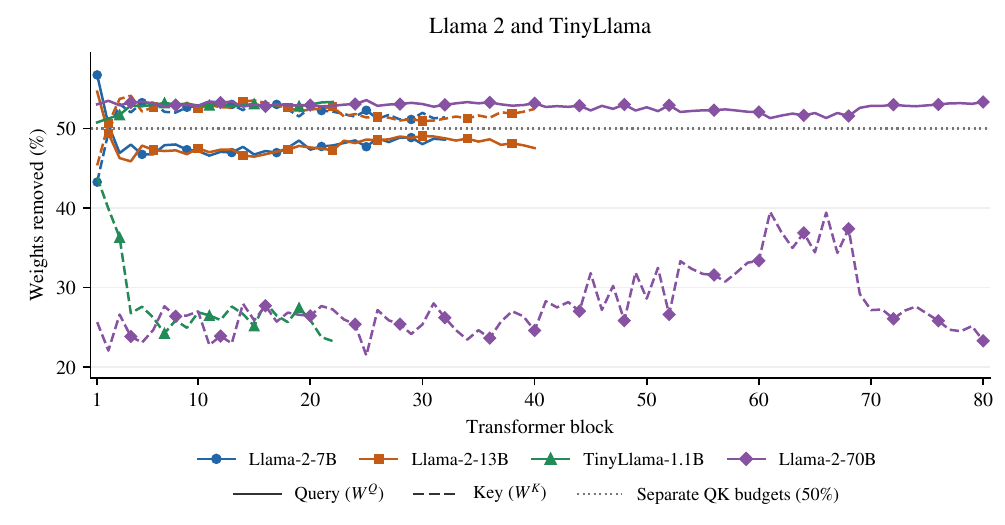}
\caption{QK-Wanda allocation at a shared 50\% budget on TinyLlama and Llama~2. Solid lines show query-weight removal and dashed lines key-weight removal; colors and markers identify checkpoints. The dotted line marks 50\% removal in each projection.}
\label{fig:qk_shared_allocation_llama2}
\end{figure}

\paragraph{Similar allocations on Llama 2 7B and 13B.}
Mean query/key removal rates were 48.09\%/51.91\% on 7B and 47.98\%/52.02\% on 13B. Both retained a larger fraction of query weights on average. Key sparsity exceeded query sparsity in 30 of 32 blocks on 7B and 39 of 40 blocks on 13B. These models have equally sized query and key projections, so their removal fractions average to 50\% within each block.

\paragraph{Stronger K retention on TinyLlama and Llama 2 70B.}
Both checkpoints removed a larger fraction of query than key weights in every block. Mean query/key removal rates were 52.75\%/27.96\% on TinyLlama and 52.76\%/27.91\% on Llama~2 70B. On 70B, key removal varied from 21.45\% to 39.52\% across 80 blocks. Both have an 8:1 query-to-key parameter-count ratio, so a modest increase in query sparsity above 50\% permits a much larger decrease in key sparsity.

\FloatBarrier
\paragraph{Qwen2.5.}
\label{app:shared_allocation_qwen25}
We examined the shared-budget masks from all seven Qwen2.5 checkpoints in \cref{sec:qwen25_comparison}, including 72B, at 50\% combined QK sparsity. The curves report the fraction removed within each projection at the actual transformer block index. Because Q and K have unequal sizes under GQA, their combined sparsity is weighted by their parameter counts; it remained exactly 50\% in every block.

\begin{figure}[htbp]
\centering
\includegraphics[width=\textwidth]{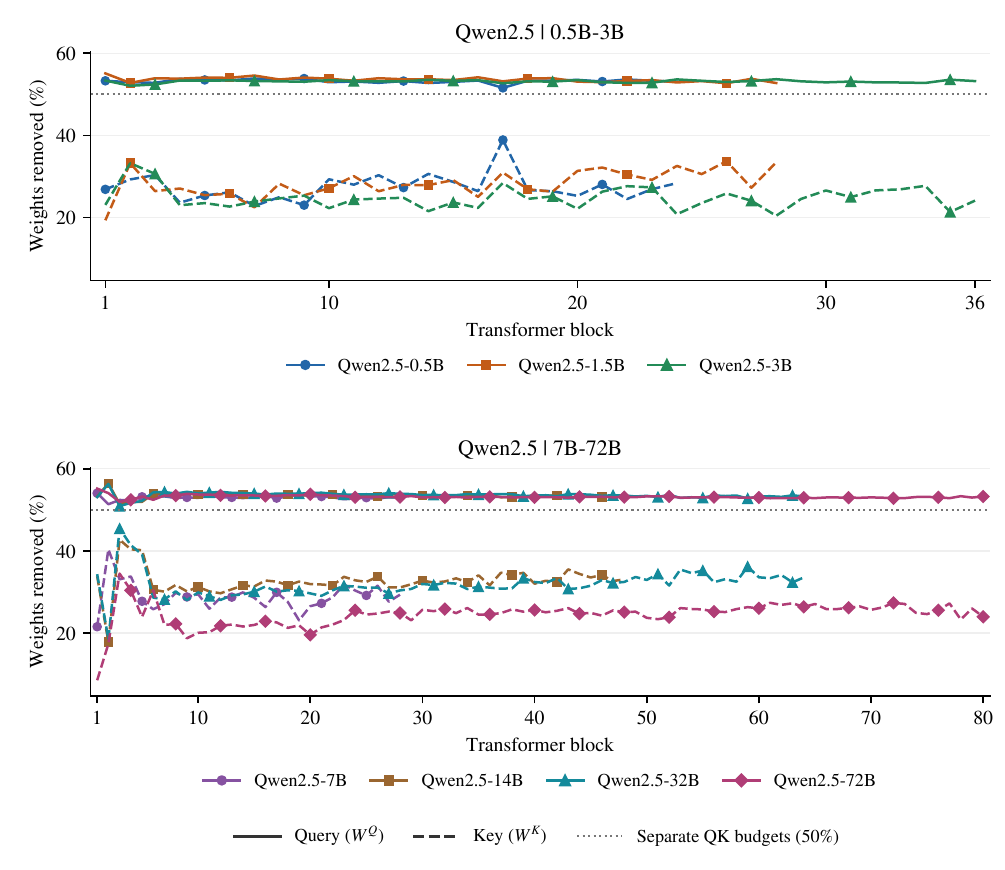}
\caption{QK-Wanda allocation with a shared 50\% budget on Qwen2.5, with 0.5B--3B models in the upper panel and 7B--72B models in the lower panel. Color and marker shape identify the checkpoint within each panel. Solid lines denote \(W^Q\), dashed lines denote \(W^K\), and the dotted line marks separate 50\% budgets. Each curve ends at its checkpoint's final block; 72B extends to block 80.}
\label{fig:qk_shared_allocation_qwen25}
\end{figure}
\paragraph{Observed allocation.} Mean query/key removal fractions were 53.22\%/27.44\% on Qwen2.5-0.5B, 53.62\%/28.30\% on Qwen2.5-1.5B, 53.15\%/24.81\% on Qwen2.5-3B, 53.02\%/28.88\% on Qwen2.5-7B, 53.47\%/32.65\% on Qwen2.5-14B, 53.62\%/31.88\% on Qwen2.5-32B, 53.18\%/24.53\% on Qwen2.5-72B. Every block in all seven checkpoints removed a larger fraction of query than key weights. On 72B, key removal ranged from 8.52\% to 34.49\% across 80 blocks.

\FloatBarrier
\subsection{Historical masked full-block comparison}
\label{app:full_block_masked_history}

These are the original causal QK-Wanda-M measurements, including the 80\% failure regime; they are retained separately from the QK-Wanda comparison in \cref{tab:full_block}.

\begin{table}[htbp]
\centering
\fontsize{8}{9}\selectfont
\setlength{\tabcolsep}{1pt}
\caption{Pruning all seven transformer-block projections (QK/V/O and MLP gate/up/down), with embeddings and the output head excluded from the sparsity denominator. Wanda uses row budgets throughout; Wanda + QK-Wanda-M uses shared QK-Wanda-M only for QK and row-wise Wanda elsewhere. Accuracy is the seven-task mean (\%); seed 0. Bold marks lower displayed PPL or higher accuracy at each sparsity, including ties.}
\label{tab:full_block_masked_history}
% models: tinyllama_1b llama2_7b llama2_13b
\begin{qktabular}{11}{lcccccccccc}
\toprule
 & & \multicolumn{3}{c}{\textbf{TinyLlama-1.1B}} & \multicolumn{3}{c}{\textbf{Llama-2-7B}} & \multicolumn{3}{c}{\textbf{Llama-2-13B}}\\
\cmidrule(lr){3-5}\cmidrule(lr){6-8}\cmidrule(lr){9-11}
Method & Sparsity & \shortstack{Wiki2\\PPL} & \shortstack{C4\\PPL} & \shortstack{Zero-shot\\acc. (\%)} & \shortstack{Wiki2\\PPL} & \shortstack{C4\\PPL} & \shortstack{Zero-shot\\acc. (\%)} & \shortstack{Wiki2\\PPL} & \shortstack{C4\\PPL} & \shortstack{Zero-shot\\acc. (\%)}\\
\midrule
Dense & 0\% & 7.77 & 9.94 & 49.93 & 5.47 & 7.26 & 64.38 & 4.88 & 6.73 & 67.05\\
\midrule
\algname{Wanda} & 30\% & 8.33 & 10.46 & \textbf{50.22} & \textbf{5.66} & 7.50 & 63.88 & 5.06 & 6.90 & 66.14\\
\algname{Wanda} + \algname{QK-Wanda-M} & 30\% & \textbf{8.31} & \textbf{10.44} & 49.94 & \textbf{5.66} & \textbf{7.49} & \textbf{64.03} & \textbf{5.05} & \textbf{6.89} & \textbf{66.39}\\
\addlinespace[2pt]
\algname{Wanda} & 50\% & 11.52 & 13.85 & 48.22 & 6.89 & \textbf{9.06} & 60.23 & 5.90 & 8.00 & 64.61\\
\algname{Wanda} + \algname{QK-Wanda-M} & 50\% & \textbf{11.22} & \textbf{13.62} & \textbf{48.42} & \textbf{6.85} & \textbf{9.06} & \textbf{61.09} & \textbf{5.86} & \textbf{7.98} & \textbf{64.76}\\
\addlinespace[2pt]
\algname{Wanda} & 80\% & 512.49 & 413.72 & \textbf{34.75} & \textbf{1782.14} & \textbf{1214.75} & 33.76 & \textbf{1643.05} & \textbf{1387.63} & 34.65\\
\algname{Wanda} + \algname{QK-Wanda-M} & 80\% & \textbf{410.73} & \textbf{374.82} & 34.71 & 1946.81 & 2011.70 & \textbf{34.63} & 2037.42 & 2360.28 & \textbf{35.24}\\
\bottomrule
\end{qktabular}
\end{table}

To test whether QK-Wanda-M remains useful as a component of broader pruning, we use it for QK while retaining Wanda on the value, output, and MLP projections. \Cref{tab:full_block_masked_history} compares this hybrid with Wanda throughout, using the same calibration and evaluation protocol. Here the requested sparsity applies to all seven block projections; embeddings, biases, normalization parameters, and the language-model head remain dense.

At 30\% and 50\%, the hybrid improved both PPL metrics on all three models (\cref{tab:full_block_masked_history}). At 50\%, TinyLlama WikiText-2/C4 PPL decreased from 11.52/13.85 to 11.22/13.62, with smaller relative gains on 7B/13B; mean accuracy increased by 0.19/0.86/0.15 points on TinyLlama/7B/13B. Thus, the QK-aware hybrid also helped at moderate full-block sparsity.

At 80\% full-block sparsity, both pipelines degraded severely: PPL exceeded 1000 on 7B/13B, and mean accuracy remained below 36\% on all three models. Relative differences in this failure regime do not indicate preservation of model quality.

\FloatBarrier

\section{Calibration Sensitivity}
\label{app:calibration}

The sample-count sensitivity experiment uses Llama-3.1-8B at 50\% QK sparsity and seed 0, comparing row-wise Wanda with causal QK-Wanda-M. The context-length study extends the comparison to unmasked QK-Wanda and Llama-2-7B while holding the 524,288-token calibration pool, 50\% QK sparsity, and seed 0 fixed. WikiText-2 test and C4 validation perplexities are evaluated at context length 2048.

\paragraph{Sample count.}
\Cref{tab:calibration_samples} varies the number of length-2048 windows. Both methods were weakly sensitive to sample count. We selected a common \(N_*=256\): all four PPL values were within 0.047\% of their observed minima at \(N=2048\), which required eight times as many tokens. The WikiText-2/C4 improvements from 256 to 2048 windows were only 0.0020/0.0046 for Wanda and 0.0016/0.0003 for QK-Wanda-M.

\paragraph{Context length.} \Cref{tab:calibration_context,tab:calibration_context_llama2} hold the 524,288-token pool fixed and vary its segmentation. On Llama-3.1-8B, both causal QK-Wanda-M (M) and unmasked QK-Wanda attain their observed minima at \(T=64\) on both datasets. On Llama-2-7B, both variants attain their minima at \(T=128\); at \(T=64\), each is within 0.0010 PPL of those minima. Llama-3.1-8B at \(T=64\): unmasked minus causal PPL is -0.0011 on WikiText-2 and +0.0002 on C4; for Llama-2-7B at \(T=64\): unmasked minus causal PPL is +0.0007 on WikiText-2 and -0.0007 on C4. The direction of the small masked--unmasked difference changes across model and dataset, so neither variant wins uniformly in this sensitivity test. The results support \(T=64\) as a stable operating point across these two checkpoints, but do not establish a universal optimum.

\begin{table}[H]
\centering\small
\setlength{\tabcolsep}{6pt}
\caption{Sample-count sensitivity on Llama-3.1-8B at 50\% QK sparsity, length 2048, and seed 0. Wanda uses row budgets and QK-Wanda-M separate matrix budgets. Evaluation context is 2048. Bold marks within-method, within-dataset minima at the displayed precision; $\dagger$ marks the selected common default, independently of boldface.}
\label{tab:calibration_samples}
\begin{tabular}{rrcc>{\columncolor{qkrow}}c>{\columncolor{qkrow}}c}
\toprule
 & & \multicolumn{2}{c}{Wanda PPL} & \multicolumn{2}{>{\columncolor{qkrow}}c}{QK-Wanda-M PPL}\\
\cmidrule(lr){3-4}\cmidrule(lr){5-6}
Samples & Tokens & WikiText-2 & C4 & WikiText-2 & C4\\
\midrule
32 & 65,536 & 6.4165 & 9.8760 & 6.4585 & 9.9249\\
64 & 131,072 & 6.4163 & 9.8757 & 6.4573 & 9.9199\\
128 & 262,144 & 6.4161 & 9.8753 & 6.4583 & 9.9210\\
256$\dagger$ & 524,288 & 6.4172 & 9.8781 & 6.4580 & 9.9186\\
512 & 1,048,576 & 6.4162 & 9.8759 & 6.4569 & 9.9194\\
1024 & 2,097,152 & 6.4162 & 9.8741 & 6.4565 & 9.9186\\
2048 & 4,194,304 & \textbf{6.4152} & \textbf{9.8734} & \textbf{6.4564} & \textbf{9.9183}\\
\bottomrule
\end{tabular}
\end{table}

\begin{table}[H]
\centering\small
\setlength{\tabcolsep}{4pt}
\caption{Context-length sensitivity on Llama-3.1-8B at 50\% QK sparsity and seed 0, with exactly 524,288 calibration tokens. QK-Wanda-M and QK-Wanda use separate matrix-wide budgets for Q and K; M uses causal pre-RoPE scores, while the unsuffixed method uses unmasked scores. Evaluation context is 2048. Bold marks within-method, within-dataset minima at displayed precision; $\dagger$ marks the selected common default independently of boldface.}
\label{tab:calibration_context}
\begin{tabular}{rrcc>{\columncolor{qkrow}}c>{\columncolor{qkrow}}c>{\columncolor{qkrow}}c>{\columncolor{qkrow}}c}
\toprule
 & & \multicolumn{2}{c}{Wanda} & \multicolumn{2}{c}{QK-Wanda-M} & \multicolumn{2}{c}{QK-Wanda}\\
\cmidrule(lr){3-4}\cmidrule(lr){5-6}\cmidrule(lr){7-8}
Length & Sequences & WikiText-2 & C4 & WikiText-2 & C4 & WikiText-2 & C4\\
\midrule
4 & 131,072 & 6.4076 & 9.8574 & 6.4722 & 9.9447 & 6.4728 & 9.9415\\
8 & 65,536 & 6.4036 & 9.8523 & 6.4599 & 9.9185 & 6.4596 & 9.9169\\
16 & 32,768 & \textbf{6.4029} & 9.8506 & 6.4522 & 9.9028 & 6.4515 & 9.9016\\
32 & 16,384 & 6.4043 & \textbf{9.8494} & 6.4458 & 9.8970 & 6.4464 & 9.8966\\
64$\dagger$ & 8,192 & 6.4038 & 9.8527 & \textbf{6.4437} & \textbf{9.8923} & \textbf{6.4425} & \textbf{9.8925}\\
128 & 4,096 & 6.4044 & 9.8563 & 6.4442 & 9.8925 & 6.4438 & 9.8930\\
256 & 2,048 & 6.4083 & 9.8608 & 6.4470 & 9.8979 & 6.4472 & 9.8989\\
512 & 1,024 & 6.4116 & 9.8669 & 6.4505 & 9.9048 & 6.4502 & 9.9046\\
1024 & 512 & 6.4157 & 9.8744 & 6.4535 & 9.9103 & 6.4555 & 9.9140\\
2048 & 256 & 6.4172 & 9.8781 & 6.4580 & 9.9186 & 6.4568 & 9.9201\\
\bottomrule
\end{tabular}
\end{table}

\begin{table}[H]
\centering\small
\setlength{\tabcolsep}{5pt}
\caption{Context-length sensitivity on Llama-2-7B at 50\% QK sparsity and seed 0, with exactly 524,288 calibration tokens. Both QK-Wanda variants use separate matrix-wide budgets for Q and K; M uses causal pre-RoPE scores and the unsuffixed method uses unmasked scores. Evaluation context is 2048. Bold marks within-method, within-dataset minima at displayed precision; $\dagger$ marks the common default chosen from the Llama-3.1-8B sweep.}
\label{tab:calibration_context_llama2}
\begin{tabular}{rr>{\columncolor{qkrow}}c>{\columncolor{qkrow}}c>{\columncolor{qkrow}}c>{\columncolor{qkrow}}c}
\toprule
 & & \multicolumn{2}{c}{QK-Wanda-M} & \multicolumn{2}{c}{QK-Wanda}\\
\cmidrule(lr){3-4}\cmidrule(lr){5-6}
Length & Sequences & WikiText-2 & C4 & WikiText-2 & C4\\
\midrule
4 & 131,072 & 5.5459 & 7.3785 & 5.5469 & 7.3743\\
8 & 65,536 & 5.5391 & 7.3700 & 5.5388 & 7.3672\\
16 & 32,768 & 5.5356 & 7.3641 & 5.5355 & 7.3621\\
32 & 16,384 & 5.5345 & 7.3616 & 5.5334 & 7.3606\\
64$\dagger$ & 8,192 & 5.5319 & 7.3596 & 5.5326 & 7.3588\\
128 & 4,096 & \textbf{5.5318} & \textbf{7.3586} & \textbf{5.5321} & \textbf{7.3579}\\
256 & 2,048 & 5.5325 & 7.3588 & 5.5328 & 7.3584\\
512 & 1,024 & 5.5337 & 7.3603 & 5.5343 & 7.3587\\
1024 & 512 & 5.5352 & 7.3614 & 5.5351 & 7.3594\\
2048 & 256 & 5.5374 & 7.3619 & 5.5361 & 7.3607\\
\bottomrule
\end{tabular}
\end{table}

\FloatBarrier
\section{Additional Model Comparisons}
\label{app:model_results}

The tables report QK-only pruning under the common protocol in \cref{sec:setup,app:protocol}. They include \algname{QK-Wanda} with separate and shared budgets, row-wise \algname{Wanda}, and the explicitly labeled scoring variants. QK error follows \cref{eq:relative_qk_error}; mean zero-shot accuracy uses the seven tasks in Appendix~\ref{app:zero_shot_protocol}. Full sparsity sweeps appear in Appendices~\ref{app:budget_curves} and~\ref{app:qk_error_curves}.

Across all 15 checkpoints and nine nonzero sparsities, shared-budget QK-Wanda lowers QK error in 133/135 comparisons with Wanda, WikiText-2 PPL in 103/135, and C4 PPL in 101/135. Its mean zero-shot accuracy is higher in 30/45 comparisons. The aggregate tables below separate these results from those of QK-Wanda-M.

\subsection{Llama 3}
\label{sec:model_comparison}
\label{app:llama3_results}

\Cref{tab:main_results} covers Llama-3.2-1B/3B and Llama-3.1-8B/70B. At 80\% sparsity, QK-Wanda has higher PPL than Wanda on all four checkpoints, despite lower QK error on 1B, 8B, and 70B. The 3B checkpoint also has higher QK error at 80\% and 90\%. This family illustrates the limits of local reconstruction as a predictor of downstream quality. QK-Wanda-MC and QK-Wanda-MRC test centering before and after RoPE; their objectives and results are discussed in Appendix~\ref{app:centered}.

\begin{table}[H]
\centering
\fontsize{9}{10}\selectfont
\setlength{\tabcolsep}{1.2pt}
\caption{QK-only pruning on Llama~3. Each checkpoint has WikiText-2 PPL, C4 PPL, seven-task mean accuracy (\%), and relative QK error (\%; \cref{eq:relative_qk_error}). QK-Wanda is shown with shared and separate (sep.) budgets; the labeled variants use shared budgets. Centered variants are defined in Appendix~\ref{app:centered}. Bold marks the best displayed result within each sparsity, including ties; dashes denote variants not evaluated.}
\label{tab:main_results}
\begin{qktabular}{10}{lccccccccc}
\toprule
 & & \multicolumn{4}{c}{\textbf{Llama-3.2-1B}} & \multicolumn{4}{c}{\textbf{Llama-3.2-3B}}\\
\cmidrule(lr){3-6}\cmidrule(lr){7-10}
Method & Sparsity & \shortstack{Wiki2\\PPL} & \shortstack{C4\\PPL} & \shortstack{Zero\\shot\\acc.} & \shortstack{Rel. QK\\error\\(\%)} & \shortstack{Wiki2\\PPL} & \shortstack{C4\\PPL} & \shortstack{Zero\\shot\\acc.} & \shortstack{Rel. QK\\error\\(\%)}\\
\midrule
Dense & 0\% & 9.75 & 14.01 & 54.14 & 0.00 & 7.81 & 11.33 & 61.74 & 0.00\\
\midrule
\algname{Wanda} & 30\% & 9.83 & 14.15 & \textbf{54.78} & 0.23 & 7.85 & 11.41 & 61.35 & 0.15\\
\algname{QK-Wanda-M} & 30\% & \textbf{9.81} & \textbf{14.14} & 54.28 & 0.09 & \textbf{7.84} & \textbf{11.38} & 61.55 & \textbf{0.07}\\
\algname{QK-Wanda (shared)} & 30\% & \textbf{9.81} & \textbf{14.14} & 54.18 & \textbf{0.08} & \textbf{7.84} & \textbf{11.38} & 61.57 & \textbf{0.07}\\
\algname{QK-Wanda (sep.)} & 30\% & 9.82 & 14.17 & 54.17 & 0.11 & \textbf{7.84} & \textbf{11.38} & \textbf{61.64} & 0.08\\
\algname{QK-Wanda-MRC} & 30\% & 9.90 & 14.35 & 54.19 & 0.22 & 7.86 & 11.40 & 61.48 & 0.17\\
\algname{QK-Wanda-MC} & 30\% & 9.91 & 14.36 & 54.43 & 0.24 & 7.90 & 11.47 & 61.35 & 0.31\\
\addlinespace[2pt]
\algname{Wanda} & 50\% & \textbf{10.31} & \textbf{15.03} & \textbf{54.22} & 1.58 & \textbf{8.06} & \textbf{11.79} & 61.11 & 1.07\\
\algname{QK-Wanda-M} & 50\% & 10.65 & 16.24 & 53.51 & 0.76 & 8.07 & 11.80 & 60.99 & 0.64\\
\algname{QK-Wanda (shared)} & 50\% & 10.60 & 16.09 & 53.74 & \textbf{0.74} & 8.07 & \textbf{11.79} & 60.98 & \textbf{0.62}\\
\algname{QK-Wanda (sep.)} & 50\% & 11.05 & 17.34 & 53.09 & 1.16 & 8.08 & 11.82 & 61.04 & 0.77\\
\algname{QK-Wanda-MRC} & 50\% & 22.63 & 46.81 & 51.07 & 2.03 & 8.25 & 12.11 & \textbf{61.72} & 1.49\\
\algname{QK-Wanda-MC} & 50\% & 16.46 & 28.93 & 52.92 & 2.36 & 8.63 & 12.75 & 60.31 & 2.56\\
\addlinespace[2pt]
\algname{Wanda} & 80\% & \textbf{34.74} & \textbf{57.75} & \textbf{43.48} & 15.77 & \textbf{14.22} & \textbf{21.80} & \textbf{53.51} & 9.87\\
\algname{QK-Wanda-M} & 80\% & 202.16 & 260.13 & 42.52 & 15.71 & 42.70 & 44.30 & 50.16 & 10.66\\
\algname{QK-Wanda (shared)} & 80\% & 199.07 & 262.31 & 42.48 & \textbf{15.24} & 42.82 & 44.07 & 50.26 & 10.31\\
\algname{QK-Wanda (sep.)} & 80\% & 217.12 & 314.24 & 40.46 & 21.21 & 38.12 & 47.97 & 49.24 & 12.18\\
\algname{QK-Wanda-MRC} & 80\% & 169.23 & 262.66 & 39.33 & 17.65 & 23.18 & 33.50 & 52.88 & \textbf{9.41}\\
\algname{QK-Wanda-MC} & 80\% & 401.27 & 406.04 & 39.01 & 25.89 & 86.84 & 82.28 & 47.16 & 15.61\\
\bottomrule
\end{qktabular}
\end{table}

\begin{table}[H]
\centering\fontsize{9}{10}\selectfont
\setlength{\tabcolsep}{1.2pt}
{\small Table~\ref{tab:main_results} (continued).}\par\smallskip
\begin{qktabular}{10}{lccccccccc}
\toprule
 & & \multicolumn{4}{c}{\textbf{Llama-3.1-8B}} & \multicolumn{4}{c}{\textbf{Llama-3.1-70B}}\\
\cmidrule(lr){3-6}\cmidrule(lr){7-10}
Method & Sparsity & \shortstack{Wiki2\\PPL} & \shortstack{C4\\PPL} & \shortstack{Zero\\shot\\acc.} & \shortstack{Rel. QK\\error\\(\%)} & \shortstack{Wiki2\\PPL} & \shortstack{C4\\PPL} & \shortstack{Zero\\shot\\acc.} & \shortstack{Rel. QK\\error\\(\%)}\\
\midrule
Dense & 0\% & 6.24 & 9.54 & 69.10 & 0.00 & 2.81 & 7.11 & 74.21 & 0.00\\
\midrule
\algname{Wanda} & 30\% & 6.26 & 9.58 & 68.80 & 0.12 & \textbf{2.82} & \textbf{7.12} & 74.16 & 0.04\\
\algname{QK-Wanda-M} & 30\% & \textbf{6.25} & \textbf{9.57} & 69.40 & \textbf{0.05} & \textbf{2.82} & \textbf{7.12} & 74.18 & \textbf{0.02}\\
\algname{QK-Wanda (shared)} & 30\% & \textbf{6.25} & \textbf{9.57} & 69.29 & \textbf{0.05} & \textbf{2.82} & \textbf{7.12} & \textbf{74.23} & \textbf{0.02}\\
\algname{QK-Wanda (sep.)} & 30\% & 6.26 & 9.58 & 69.31 & 0.06 & \textbf{2.82} & 7.13 & 74.08 & 0.03\\
\algname{QK-Wanda-MRC} & 30\% & 6.27 & 9.59 & 69.18 & 0.11 & 2.83 & 7.14 & 74.07 & 0.08\\
\algname{QK-Wanda-MC} & 30\% & 6.29 & 9.62 & \textbf{69.60} & 0.18 & \resultcell & \resultcell & \resultcell & \resultcell\\
\addlinespace[2pt]
\algname{Wanda} & 50\% & \textbf{6.40} & 9.85 & 68.43 & 0.82 & \textbf{2.90} & \textbf{7.22} & 73.60 & 0.39\\
\algname{QK-Wanda-M} & 50\% & 6.41 & \textbf{9.82} & 69.39 & 0.40 & 2.97 & 7.27 & \textbf{73.79} & \textbf{0.18}\\
\algname{QK-Wanda (shared)} & 50\% & 6.41 & \textbf{9.82} & 69.42 & \textbf{0.39} & 2.97 & 7.27 & 73.77 & \textbf{0.18}\\
\algname{QK-Wanda (sep.)} & 50\% & 6.44 & 9.89 & 69.16 & 0.53 & 3.05 & 7.39 & 73.64 & 0.41\\
\algname{QK-Wanda-MRC} & 50\% & 6.49 & 9.97 & \textbf{69.61} & 0.89 & 3.14 & 7.47 & 73.66 & 0.81\\
\algname{QK-Wanda-MC} & 50\% & 6.66 & 10.21 & 69.01 & 1.58 & \resultcell & \resultcell & \resultcell & \resultcell\\
\addlinespace[2pt]
\algname{Wanda} & 80\% & \textbf{8.89} & \textbf{13.54} & 59.64 & 7.46 & \textbf{4.98} & \textbf{9.89} & \textbf{66.64} & 5.27\\
\algname{QK-Wanda-M} & 80\% & 13.81 & 22.36 & 60.40 & 5.66 & 18.98 & 32.37 & 66.45 & 3.36\\
\algname{QK-Wanda (shared)} & 80\% & 13.47 & 21.97 & 60.61 & \textbf{5.58} & 18.89 & 32.59 & 65.93 & \textbf{3.27}\\
\algname{QK-Wanda (sep.)} & 80\% & 14.01 & 23.06 & 58.95 & 7.34 & 25.96 & 43.27 & 61.06 & 5.84\\
\algname{QK-Wanda-MRC} & 80\% & 12.51 & 19.95 & \textbf{61.81} & 6.99 & 13.67 & 22.61 & 62.61 & 5.34\\
\algname{QK-Wanda-MC} & 80\% & 23.44 & 36.95 & 59.13 & 10.84 & \resultcell & \resultcell & \resultcell & \resultcell\\
\bottomrule
\end{qktabular}
\end{table}

\subsubsection{Zero-shot task performance}
\label{app:zero_shot_llama3}

\Cref{tab:zero_shot_llama_all,tab:zero_shot_llama70b} report the individual task scores underlying \cref{tab:main_results}. At 80\% sparsity, shared-budget QK-Wanda has lower mean accuracy than Wanda on 1B and 3B, but higher mean accuracy on 8B (60.61\% versus 59.64\%). On 70B, the means are close: 65.93\% versus 66.64\%, even though WikiText-2/C4 PPL rises from 4.98/9.89 to 18.89/32.59. Thus, the increase in language-model loss does not translate into an equally large change in task accuracy.

QK-Wanda-MRC improves the high-sparsity PPL of 3B and 8B relative to QK-Wanda, but its effect on task accuracy depends on the checkpoint. Appendix~\ref{app:rope_centered} discusses these comparisons. All displayed variants use the same evaluation examples; differences between PPL and accuracy describe distinct measured outcomes.

\subsection{TinyLlama and Llama 2}
\label{app:tinyllama_results}

TinyLlama-1.1B is an independently trained 3-trillion-token checkpoint \citep{zhang2024tinyllama}, with 32 query heads and 4 key/value heads. Llama~2 7B/13B use MHA, while 70B uses GQA. Each checkpoint uses its native tokenizer; calibration tokens are identical across methods within that checkpoint.

At 80\% sparsity, shared-budget QK-Wanda reduces WikiText-2/C4 PPL from 11.01/13.16 to 8.91/11.36 on TinyLlama, from 12.54/13.94 to 8.31/11.20 on Llama-2-7B, and from 4.55/6.91 to 3.63/5.97 on Llama-2-70B. The 13B checkpoint also improves on both datasets. These gains span MHA and GQA; they do not imply that attention architecture alone determines the benefit.

% models: tinyllama_1b llama2_7b llama2_13b llama2_70b
\begin{table}[H]
\centering
\fontsize{9}{10}\selectfont
\setlength{\tabcolsep}{1.4pt}
\caption{QK-only pruning on TinyLlama and Llama~2. Columns report WikiText-2 PPL, C4 PPL, seven-task mean accuracy (\%), and relative QK error (\%; \cref{eq:relative_qk_error}). QK-Wanda uses shared or separate (sep.) budgets; the labeled variants use shared budgets. Bold marks the best displayed result at each sparsity, including ties.}
\label{tab:llama2_results}
\begin{qktabular}{10}{lccccccccc}
\toprule
 & & \multicolumn{4}{c}{\textbf{TinyLlama-1.1B (GQA)}} & \multicolumn{4}{c}{\textbf{Llama-2-7B (MHA)}}\\
\cmidrule(lr){3-6}\cmidrule(lr){7-10}
Method & Sparsity & \shortstack{Wiki2\\PPL} & \shortstack{C4\\PPL} & \shortstack{Zero\\shot\\acc.} & \shortstack{Rel. QK\\error\\(\%)} & \shortstack{Wiki2\\PPL} & \shortstack{C4\\PPL} & \shortstack{Zero\\shot\\acc.} & \shortstack{Rel. QK\\error\\(\%)}\\
\midrule
Dense & 0\% & 7.77 & 9.94 & 49.93 & 0.00 & 5.47 & 7.26 & 64.38 & 0.00\\
\midrule
\algname{Wanda} & 30\% & 7.79 & 9.97 & 49.89 & 0.08 & 5.49 & 7.28 & 64.00 & 0.09\\
\algname{QK-Wanda-M} & 30\% & \textbf{7.78} & \textbf{9.96} & \textbf{50.02} & \textbf{0.03} & \textbf{5.48} & \textbf{7.27} & 64.24 & \textbf{0.03}\\
\algname{QK-Wanda (shared)} & 30\% & \textbf{7.78} & \textbf{9.96} & 49.95 & \textbf{0.03} & \textbf{5.48} & \textbf{7.27} & 64.15 & \textbf{0.03}\\
\algname{QK-Wanda (sep.)} & 30\% & 7.79 & \textbf{9.96} & 49.96 & \textbf{0.03} & \textbf{5.48} & \textbf{7.27} & \textbf{64.32} & \textbf{0.03}\\
\algname{QK-Wanda-MRC} & 30\% & \textbf{7.78} & 9.97 & 49.83 & 0.05 & 5.49 & \textbf{7.27} & 64.18 & 0.04\\
\algname{QK-Wanda-MC} & 30\% & 7.79 & \textbf{9.96} & 49.87 & 0.05 & 5.49 & \textbf{7.27} & 64.23 & 0.05\\
\addlinespace[2pt]
\algname{Wanda} & 50\% & 7.93 & 10.12 & \textbf{49.71} & 0.41 & 5.62 & 7.42 & 63.43 & 0.68\\
\algname{QK-Wanda-M} & 50\% & \textbf{7.85} & 10.04 & 49.64 & \textbf{0.16} & \textbf{5.53} & \textbf{7.36} & 63.92 & 0.27\\
\algname{QK-Wanda (shared)} & 50\% & \textbf{7.85} & \textbf{10.03} & 49.48 & \textbf{0.16} & \textbf{5.53} & \textbf{7.36} & 63.98 & 0.27\\
\algname{QK-Wanda (sep.)} & 50\% & 7.89 & 10.08 & 49.27 & 0.22 & \textbf{5.53} & \textbf{7.36} & 63.94 & \textbf{0.26}\\
\algname{QK-Wanda-MRC} & 50\% & 7.86 & 10.10 & 49.66 & 0.33 & 5.58 & 7.39 & \textbf{64.00} & 0.36\\
\algname{QK-Wanda-MC} & 50\% & 7.88 & 10.08 & 49.54 & 0.29 & 5.64 & 7.44 & 63.98 & 0.51\\
\addlinespace[2pt]
\algname{Wanda} & 80\% & 11.01 & 13.16 & 46.76 & 2.88 & 12.54 & 13.94 & 52.56 & 7.40\\
\algname{QK-Wanda-M} & 80\% & \textbf{8.90} & \textbf{11.35} & \textbf{47.60} & \textbf{1.31} & 8.33 & 11.21 & 57.81 & 3.16\\
\algname{QK-Wanda (shared)} & 80\% & 8.91 & 11.36 & 47.52 & \textbf{1.31} & 8.31 & 11.20 & 57.84 & 3.15\\
\algname{QK-Wanda (sep.)} & 80\% & 9.74 & 12.39 & 47.23 & 1.77 & 8.28 & 11.13 & 58.04 & \textbf{3.11}\\
\algname{QK-Wanda-MRC} & 80\% & 8.96 & 11.55 & 47.44 & 2.18 & \textbf{8.25} & \textbf{11.11} & \textbf{58.81} & 3.57\\
\algname{QK-Wanda-MC} & 80\% & 9.17 & 11.83 & 46.97 & 2.01 & 9.72 & 13.01 & 57.88 & 4.52\\
\bottomrule
\end{qktabular}
\end{table}

\begin{table}[H]
\centering\fontsize{9}{10}\selectfont
\setlength{\tabcolsep}{1.2pt}
{\small Table~\ref{tab:llama2_results} (continued).}\par\smallskip
\begin{qktabular}{10}{lccccccccc}
\toprule
 & & \multicolumn{4}{c}{\textbf{Llama-2-13B (MHA)}} & \multicolumn{4}{c}{\textbf{Llama-2-70B (GQA)}}\\
\cmidrule(lr){3-6}\cmidrule(lr){7-10}
Method & Sparsity & \shortstack{Wiki2\\PPL} & \shortstack{C4\\PPL} & \shortstack{Zero\\shot\\acc.} & \shortstack{Rel. QK\\error\\(\%)} & \shortstack{Wiki2\\PPL} & \shortstack{C4\\PPL} & \shortstack{Zero\\shot\\acc.} & \shortstack{Rel. QK\\error\\(\%)}\\
\midrule
Dense & 0\% & 4.88 & 6.73 & 67.05 & 0.00 & 3.32 & 5.71 & 71.50 & 0.00\\
\midrule
\algname{Wanda} & 30\% & 4.90 & 6.74 & 66.77 & 0.06 & \textbf{3.32} & \textbf{5.71} & 71.50 & 0.04\\
\algname{QK-Wanda-M} & 30\% & \textbf{4.89} & \textbf{6.73} & 66.92 & \textbf{0.02} & \textbf{3.32} & \textbf{5.71} & 71.62 & \textbf{0.00}\\
\algname{QK-Wanda (shared)} & 30\% & \textbf{4.89} & \textbf{6.73} & 66.84 & \textbf{0.02} & \textbf{3.32} & \textbf{5.71} & 71.54 & \textbf{0.00}\\
\algname{QK-Wanda (sep.)} & 30\% & \textbf{4.89} & \textbf{6.73} & 66.86 & \textbf{0.02} & \textbf{3.32} & \textbf{5.71} & 71.59 & 0.01\\
\algname{QK-Wanda-MRC} & 30\% & \textbf{4.89} & \textbf{6.73} & 66.98 & \textbf{0.02} & \textbf{3.32} & \textbf{5.71} & \textbf{71.65} & 0.01\\
\algname{QK-Wanda-MC} & 30\% & \textbf{4.89} & \textbf{6.73} & \textbf{66.99} & 0.03 & \resultcell & \resultcell & \resultcell & \resultcell\\
\addlinespace[2pt]
\algname{Wanda} & 50\% & 4.97 & 6.83 & 66.33 & 0.45 & 3.36 & 5.74 & 71.37 & 0.44\\
\algname{QK-Wanda-M} & 50\% & \textbf{4.93} & \textbf{6.79} & 66.74 & 0.16 & \textbf{3.33} & \textbf{5.72} & 71.62 & 0.06\\
\algname{QK-Wanda (shared)} & 50\% & \textbf{4.93} & \textbf{6.79} & 66.66 & 0.16 & \textbf{3.33} & \textbf{5.72} & 71.60 & \textbf{0.05}\\
\algname{QK-Wanda (sep.)} & 50\% & \textbf{4.93} & \textbf{6.79} & 66.77 & \textbf{0.15} & 3.34 & \textbf{5.72} & 71.62 & 0.10\\
\algname{QK-Wanda-MRC} & 50\% & \textbf{4.93} & 6.80 & \textbf{67.19} & 0.21 & 3.34 & 5.73 & \textbf{71.99} & 0.19\\
\algname{QK-Wanda-MC} & 50\% & 4.96 & 6.81 & 67.00 & 0.33 & \resultcell & \resultcell & \resultcell & \resultcell\\
\addlinespace[2pt]
\algname{Wanda} & 80\% & 8.06 & 10.20 & 56.35 & 5.86 & 4.55 & 6.91 & 65.85 & 7.40\\
\algname{QK-Wanda-M} & 80\% & 6.27 & 8.58 & 64.09 & 2.04 & \textbf{3.63} & \textbf{5.97} & 71.76 & 1.03\\
\algname{QK-Wanda (shared)} & 80\% & 6.25 & 8.55 & 64.12 & 2.03 & \textbf{3.63} & \textbf{5.97} & 71.79 & \textbf{1.00}\\
\algname{QK-Wanda (sep.)} & 80\% & 6.25 & 8.54 & 64.33 & \textbf{2.00} & 3.83 & 6.14 & \textbf{71.80} & 1.58\\
\algname{QK-Wanda-MRC} & 80\% & \textbf{6.19} & \textbf{8.43} & \textbf{64.94} & 2.12 & 3.85 & 6.27 & 70.84 & 3.36\\
\algname{QK-Wanda-MC} & 80\% & 6.64 & 8.99 & 64.28 & 2.76 & \resultcell & \resultcell & \resultcell & \resultcell\\
\bottomrule
\end{qktabular}
\end{table}

\paragraph{Task-level results.}
\label{app:zero_shot_small}
\Cref{tab:zero_shot_llama_all,tab:zero_shot_llama70b} retain all seven task scores. At 80\%, QK-Wanda raises TinyLlama's mean accuracy from 46.76\% to 47.52\%, and Llama-2-7B's from 52.56\% to 57.84\%. Llama-2-70B improves from 65.85\% to 71.79\%. The individual tasks can move in different directions; the mean does not imply improvement on every task.

\subsection{Qwen2.5}
\label{sec:qwen25_comparison}

We evaluate seven Qwen2.5 base checkpoints from 0.5B to 72B \citep{qwen2024qwen25}. All use GQA. Their QK projections include biases, which enter the score calculations and remain unchanged during pruning. The experiments use native tokenizers and the protocol in \cref{sec:setup,app:protocol}.

\Cref{tab:qwen25_results} reports the 0.5B--32B checkpoints. The PPL differences are generally small at moderate sparsity. At 80\%, outcomes depend on the checkpoint and dataset: for example, QK-Wanda improves WikiText-2 PPL on 7B while slightly increasing C4 PPL, and increases both on 32B. Lower QK error therefore need not coincide with lower PPL. The full sweeps and allocation profiles appear in Appendices~\ref{app:budget_qwen25} and~\ref{app:shared_allocation_qwen25}.

\begin{table}[H]
\centering
\fontsize{8}{9}\selectfont
\setlength{\tabcolsep}{1.5pt}
\caption{QK-only pruning on Qwen2.5 0.5B--32B. Columns report WikiText-2 PPL, C4 PPL, seven-task mean accuracy (\%), and relative QK error (\%; \cref{eq:relative_qk_error}). QK-Wanda uses shared or separate (sep.) budgets; the labeled variants use shared budgets. Biases remain dense. Bold marks the best displayed result at each sparsity, including ties; dashes denote variants not evaluated.}
\label{tab:qwen25_results}
% models: qwen25_0p5b qwen25_1p5b qwen25_3b
\begin{qktabular}{14}{lccccccccccccc}
\toprule
 & & \multicolumn{4}{c}{\textbf{Qwen2.5-0.5B}} & \multicolumn{4}{c}{\textbf{Qwen2.5-1.5B}} & \multicolumn{4}{c}{\textbf{Qwen2.5-3B}}\\
\cmidrule(lr){3-6}\cmidrule(lr){7-10}\cmidrule(lr){11-14}
Method & Sparsity & \shortstack{Wiki2\\PPL} & \shortstack{C4\\PPL} & \shortstack{Zero\\shot\\acc.} & \shortstack{Rel. QK\\error\\(\%)} & \shortstack{Wiki2\\PPL} & \shortstack{C4\\PPL} & \shortstack{Zero\\shot\\acc.} & \shortstack{Rel. QK\\error\\(\%)} & \shortstack{Wiki2\\PPL} & \shortstack{C4\\PPL} & \shortstack{Zero\\shot\\acc.} & \shortstack{Rel. QK\\error\\(\%)}\\
\midrule
Dense & 0\% & 13.07 & 20.38 & 50.86 & 0.00 & 9.26 & 15.12 & 61.64 & 0.00 & 8.03 & 13.35 & 65.29 & 0.00\\
\midrule
\algname{Wanda} & 30\% & 13.12 & 20.43 & 50.48 & 0.07 & 9.30 & 15.15 & 61.45 & 0.08 & 8.06 & \textbf{13.37} & 65.22 & 0.06\\
\algname{QK-Wanda-M} & 30\% & \textbf{13.09} & \textbf{20.40} & 50.97 & \textbf{0.02} & \textbf{9.27} & \textbf{15.13} & 61.69 & \textbf{0.02} & \textbf{8.05} & \textbf{13.37} & 65.31 & \textbf{0.02}\\
\algname{QK-Wanda (shared)} & 30\% & \textbf{13.09} & \textbf{20.40} & \textbf{50.98} & \textbf{0.02} & \textbf{9.27} & \textbf{15.13} & \textbf{61.82} & \textbf{0.02} & \textbf{8.05} & \textbf{13.37} & 65.34 & \textbf{0.02}\\
\algname{QK-Wanda (sep.)} & 30\% & 13.10 & 20.41 & 50.73 & \textbf{0.02} & 9.28 & 15.14 & 61.60 & 0.03 & \textbf{8.05} & \textbf{13.37} & 65.20 & 0.03\\
\algname{QK-Wanda-MRC} & 30\% & \textbf{13.09} & 20.41 & 50.82 & 0.04 & 9.28 & 15.14 & 61.48 & 0.04 & \textbf{8.05} & 13.38 & \textbf{65.45} & 0.04\\
\algname{QK-Wanda-MC} & 30\% & \textbf{13.09} & 20.42 & 50.42 & 0.04 & 9.28 & 15.15 & 61.37 & 0.04 & \textbf{8.05} & 13.38 & 65.30 & 0.05\\
\addlinespace[2pt]
\algname{Wanda} & 50\% & 13.31 & 20.71 & \textbf{50.69} & 0.38 & 9.48 & 15.32 & \textbf{61.57} & 0.40 & 8.20 & 13.50 & 64.72 & 0.31\\
\algname{QK-Wanda-M} & 50\% & \textbf{13.17} & \textbf{20.54} & 50.57 & 0.13 & 9.35 & \textbf{15.26} & 60.58 & \textbf{0.14} & \textbf{8.12} & \textbf{13.45} & 65.13 & \textbf{0.12}\\
\algname{QK-Wanda (shared)} & 50\% & 13.18 & \textbf{20.54} & 50.56 & \textbf{0.12} & 9.35 & \textbf{15.26} & 60.55 & \textbf{0.14} & \textbf{8.12} & \textbf{13.45} & 65.11 & \textbf{0.12}\\
\algname{QK-Wanda (sep.)} & 50\% & 13.21 & 20.64 & 50.61 & 0.16 & 9.41 & 15.31 & 60.76 & 0.18 & 8.15 & 13.49 & \textbf{65.34} & 0.16\\
\algname{QK-Wanda-MRC} & 50\% & 13.20 & 20.59 & 50.39 & 0.23 & \textbf{9.34} & 15.27 & 61.08 & 0.25 & \textbf{8.12} & 13.47 & 65.10 & 0.24\\
\algname{QK-Wanda-MC} & 50\% & 13.19 & 20.60 & 50.18 & 0.23 & 9.36 & 15.32 & 60.61 & 0.29 & \textbf{8.12} & 13.48 & 64.80 & 0.31\\
\addlinespace[2pt]
\algname{Wanda} & 80\% & 16.24 & 25.07 & 47.40 & 2.19 & 12.38 & 18.88 & 55.71 & 2.57 & 10.44 & 15.95 & 62.68 & 1.98\\
\algname{QK-Wanda-M} & 80\% & 15.07 & 23.07 & 47.01 & 1.14 & 11.50 & 18.59 & 57.06 & 1.24 & 10.06 & 16.08 & 61.50 & 1.00\\
\algname{QK-Wanda (shared)} & 80\% & 15.08 & 23.09 & 47.35 & \textbf{1.13} & 11.51 & 18.58 & \textbf{57.31} & \textbf{1.23} & 10.03 & 16.04 & 61.50 & \textbf{0.99}\\
\algname{QK-Wanda (sep.)} & 80\% & 15.89 & 24.35 & 45.32 & 1.39 & 12.00 & 19.02 & 56.40 & 1.50 & 10.59 & 16.76 & 60.27 & 1.34\\
\algname{QK-Wanda-MRC} & 80\% & \textbf{14.59} & \textbf{22.45} & \textbf{48.91} & 1.64 & \textbf{10.92} & \textbf{17.78} & 56.97 & 1.79 & \textbf{9.61} & \textbf{15.62} & 62.50 & 1.67\\
\algname{QK-Wanda-MC} & 80\% & 14.73 & 22.76 & 47.37 & 1.72 & 11.29 & 18.76 & 56.78 & 2.00 & 10.02 & 16.51 & \textbf{62.76} & 1.99\\
\bottomrule
\end{qktabular}
\end{table}

\begin{table}[H]
\centering\fontsize{8}{9}\selectfont
\setlength{\tabcolsep}{1.5pt}
{\small Table~\ref{tab:qwen25_results} (continued).}\par\smallskip
% models: qwen25_7b qwen25_14b qwen25_32b
\begin{qktabular}{14}{lccccccccccccc}
\toprule
 & & \multicolumn{4}{c}{\textbf{Qwen2.5-7B}} & \multicolumn{4}{c}{\textbf{Qwen2.5-14B}} & \multicolumn{4}{c}{\textbf{Qwen2.5-32B}}\\
\cmidrule(lr){3-6}\cmidrule(lr){7-10}\cmidrule(lr){11-14}
Method & Sparsity & \shortstack{Wiki2\\PPL} & \shortstack{C4\\PPL} & \shortstack{Zero\\shot\\acc.} & \shortstack{Rel. QK\\error\\(\%)} & \shortstack{Wiki2\\PPL} & \shortstack{C4\\PPL} & \shortstack{Zero\\shot\\acc.} & \shortstack{Rel. QK\\error\\(\%)} & \shortstack{Wiki2\\PPL} & \shortstack{C4\\PPL} & \shortstack{Zero\\shot\\acc.} & \shortstack{Rel. QK\\error\\(\%)}\\
\midrule
Dense & 0\% & 6.85 & 11.88 & 70.56 & 0.00 & 5.29 & 10.35 & 72.35 & 0.00 & 5.02 & 10.17 & 72.36 & 0.00\\
\midrule
\algname{Wanda} & 30\% & 6.86 & \textbf{11.89} & 70.45 & 0.03 & 5.31 & 10.36 & 72.62 & 0.06 & 5.03 & 10.18 & \textbf{72.48} & 0.05\\
\algname{QK-Wanda-M} & 30\% & \textbf{6.85} & \textbf{11.89} & 70.44 & \textbf{0.01} & \textbf{5.30} & \textbf{10.35} & 72.59 & \textbf{0.02} & \textbf{5.02} & \textbf{10.17} & 72.32 & \textbf{0.02}\\
\algname{QK-Wanda (shared)} & 30\% & \textbf{6.85} & \textbf{11.89} & 70.54 & \textbf{0.01} & \textbf{5.30} & \textbf{10.35} & \textbf{72.66} & \textbf{0.02} & \textbf{5.02} & \textbf{10.17} & 72.33 & \textbf{0.02}\\
\algname{QK-Wanda (sep.)} & 30\% & 6.86 & \textbf{11.89} & 70.48 & \textbf{0.01} & \textbf{5.30} & \textbf{10.35} & 72.58 & 0.03 & \textbf{5.02} & \textbf{10.17} & 72.33 & \textbf{0.02}\\
\algname{QK-Wanda-MRC} & 30\% & \textbf{6.85} & \textbf{11.89} & \textbf{70.55} & 0.02 & \textbf{5.30} & 10.36 & 72.45 & 0.04 & \textbf{5.02} & \textbf{10.17} & 72.37 & 0.03\\
\algname{QK-Wanda-MC} & 30\% & \textbf{6.85} & \textbf{11.89} & 70.45 & 0.02 & \resultcell & \resultcell & \resultcell & \resultcell & \resultcell & \resultcell & \resultcell & \resultcell\\
\addlinespace[2pt]
\algname{Wanda} & 50\% & 6.90 & 11.93 & 70.03 & 0.14 & 5.41 & 10.43 & 72.84 & 0.29 & 5.07 & 10.21 & \textbf{72.69} & 0.24\\
\algname{QK-Wanda-M} & 50\% & 6.89 & \textbf{11.92} & 70.02 & 0.06 & \textbf{5.35} & \textbf{10.40} & 72.95 & \textbf{0.13} & \textbf{5.04} & \textbf{10.20} & 72.40 & 0.11\\
\algname{QK-Wanda (shared)} & 50\% & 6.89 & \textbf{11.92} & 70.07 & \textbf{0.05} & \textbf{5.35} & \textbf{10.40} & 72.93 & \textbf{0.13} & \textbf{5.04} & \textbf{10.20} & 72.33 & \textbf{0.10}\\
\algname{QK-Wanda (sep.)} & 50\% & 6.90 & 11.94 & \textbf{70.08} & 0.07 & \textbf{5.35} & 10.41 & 72.91 & 0.15 & \textbf{5.04} & \textbf{10.20} & 72.36 & 0.12\\
\algname{QK-Wanda-MRC} & 50\% & \textbf{6.88} & \textbf{11.92} & 70.00 & 0.11 & 5.36 & \textbf{10.40} & \textbf{73.05} & 0.21 & 5.05 & \textbf{10.20} & 72.42 & 0.19\\
\algname{QK-Wanda-MC} & 50\% & \textbf{6.88} & 11.93 & 69.78 & 0.12 & \resultcell & \resultcell & \resultcell & \resultcell & \resultcell & \resultcell & \resultcell & \resultcell\\
\addlinespace[2pt]
\algname{Wanda} & 80\% & 7.39 & 12.53 & 67.36 & 0.97 & 6.50 & \textbf{11.53} & 70.64 & 1.69 & \textbf{5.61} & \textbf{10.82} & 70.80 & 1.64\\
\algname{QK-Wanda-M} & 80\% & 7.32 & 12.62 & 68.03 & \textbf{0.45} & 6.33 & 11.79 & 72.43 & \textbf{0.96} & 5.81 & 11.67 & 70.97 & 0.91\\
\algname{QK-Wanda (shared)} & 80\% & 7.32 & 12.62 & 67.98 & \textbf{0.45} & \textbf{6.32} & 11.77 & \textbf{72.56} & \textbf{0.96} & 5.80 & 11.66 & 70.86 & \textbf{0.90}\\
\algname{QK-Wanda (sep.)} & 80\% & 7.39 & 12.75 & 66.98 & 0.52 & 6.50 & 11.88 & 71.73 & 1.04 & 5.89 & 11.70 & 71.02 & 1.01\\
\algname{QK-Wanda-MRC} & 80\% & \textbf{7.26} & \textbf{12.48} & \textbf{68.46} & 0.84 & 6.36 & 11.97 & 71.88 & 1.42 & 5.69 & 11.45 & \textbf{71.39} & 1.40\\
\algname{QK-Wanda-MC} & 80\% & 7.30 & 12.58 & 67.96 & 0.92 & \resultcell & \resultcell & \resultcell & \resultcell & \resultcell & \resultcell & \resultcell & \resultcell\\
\bottomrule
\end{qktabular}
\end{table}

\paragraph{Scaling to 72B.}
All methods and the dense reference use BF16 on Qwen2.5-72B because FP16 produces nonfinite activations even without pruning. At 80\%, shared-budget QK-Wanda lowers WikiText-2/C4 PPL from 4.55/9.87 to 4.23/9.52. Its mean zero-shot accuracy increases from 73.44\% to 74.00\%. \Cref{tab:qwen72b_results} includes the matrix-budget Wanda control and the QK scoring variants.

\begin{table}[htbp]
\centering
\small
\setlength{\tabcolsep}{8pt}
\caption{QK-only pruning on Qwen2.5-72B in BF16. QK-Wanda uses shared or separate (sep.) budgets; the labeled variants use shared budgets. Columns report WikiText-2 PPL, C4 PPL, seven-task mean accuracy (\%), and relative QK error (\%; \cref{eq:relative_qk_error}). Bold marks the best displayed result at each sparsity, including ties.}
\label{tab:qwen72b_results}
\begin{tabular}{lc*{4}{c}}
\toprule
Method & Sparsity & \shortstack{Wiki2\\PPL} & \shortstack{C4\\PPL} & \shortstack{Zero-shot\\acc. (\%)} & \shortstack{Rel. QK\\error (\%)}\\
\midrule
Dense & 0\% & 3.87 & 9.26 & 74.75 & 0.00\\
\midrule
\algname{Wanda} & 30\% & \textbf{3.88} & \textbf{9.26} & \textbf{74.68} & 0.04\\
\algname{QK-Wanda-M} & 30\% & \textbf{3.88} & \textbf{9.26} & 74.61 & \textbf{0.01}\\
\algname{QK-Wanda (shared)} & 30\% & \textbf{3.88} & \textbf{9.26} & 74.66 & \textbf{0.01}\\
\algname{QK-Wanda (sep.)} & 30\% & \textbf{3.88} & \textbf{9.26} & 74.63 & 0.02\\
\algname{QK-Wanda-MRC} & 30\% & \textbf{3.88} & \textbf{9.26} & 74.64 & 0.02\\
\addlinespace[2pt]
\algname{Wanda} & 50\% & 3.92 & 9.28 & 74.72 & 0.19\\
\algname{QK-Wanda-M} & 50\% & 3.90 & \textbf{9.27} & 74.54 & \textbf{0.08}\\
\algname{QK-Wanda (shared)} & 50\% & \textbf{3.89} & \textbf{9.27} & 74.44 & \textbf{0.08}\\
\algname{QK-Wanda (sep.)} & 50\% & 3.90 & 9.28 & \textbf{74.74} & 0.10\\
\algname{QK-Wanda-MRC} & 50\% & 3.90 & 9.28 & 74.53 & 0.13\\
\addlinespace[2pt]
\algname{Wanda} & 80\% & 4.55 & 9.87 & 73.44 & 1.30\\
\algname{QK-Wanda-M} & 80\% & 4.24 & \textbf{9.52} & 73.97 & 0.63\\
\algname{QK-Wanda (shared)} & 80\% & 4.23 & \textbf{9.52} & 74.00 & \textbf{0.62}\\
\algname{QK-Wanda (sep.)} & 80\% & 4.44 & 9.64 & 73.80 & 0.72\\
\algname{QK-Wanda-MRC} & 80\% & \textbf{4.19} & 9.53 & \textbf{74.08} & 0.96\\
\bottomrule
\end{tabular}
\end{table}

\subsubsection{Zero-shot task performance}
\label{app:zero_shot_qwen25}

\Cref{tab:zero_shot_qwen25_all,tab:zero_shot_qwen25_72b} report all seven task scores for the dense checkpoints and the QK-only pruning configurations at 30\%, 50\%, and 80\% sparsity. The weights evaluated for these tasks are the same pruned weights used for PPL; evaluation makes no further changes to them.

At 80\%, shared-budget QK-Wanda raises mean accuracy on 1.5B, 7B, 14B, 32B, and 72B, but lowers it on 0.5B and 3B. The PPL ordering differs in several cases: the 32B mean improves slightly despite higher PPL on both datasets. These results complement the checkpoint summaries rather than providing a single downstream ranking for the family.

\FloatBarrier
\subsection{Descriptive aggregate comparisons}
\label{app:aggregate_counts}
All counts compare unrounded values, exclude dense points, and are descriptive rather than independent statistical trials. Lower PPL/error or higher accuracy determines the favorable direction. The 135 sweep configurations comprise 15 checkpoints and nine sparsities; task evaluation uses the 45 configurations at 30\%, 50\%, and 80\%.
\begin{table}[H]
\centering\small
\caption{Favorable comparisons / evaluated configurations. Shared QK-Wanda-M is compared with row-wise Wanda unless indicated. Accuracy uses the unweighted seven-task mean.}
\label{tab:aggregate_counts}
\begin{tabular}{@{}lrrrr@{}}
\toprule
Subset / comparison & WikiText-2 & C4 & QK error & Accuracy\\
\midrule
All sparsities & 102/135 & 98/135 & 133/135 & 27/45\\
Tabulated sparsities & 35/45 & 33/45 & 44/45 & 27/45\\
TinyLlama / Llama 2 & 35/36 & 35/36 & 36/36 & 11/12\\
Llama 3 & 13/36 & 13/36 & 34/36 & 6/12\\
Qwen2.5 & 54/63 & 50/63 & 63/63 & 10/21\\
30\% sparsity & 14/15 & 14/15 & 15/15 & 10/15\\
50\% sparsity & 11/15 & 12/15 & 15/15 & 7/15\\
80\% sparsity & 10/15 & 7/15 & 14/15 & 10/15\\
Shared vs. separate QK & 111/135 & 117/135 & 122/135 & \textemdash\\
Separate QK vs. row Wanda & 94/135 & 85/135 & 124/135 & \textemdash\\
\bottomrule
\end{tabular}
\end{table}
Of the 135 shared-QK/row-Wanda comparisons, both PPL values decreased in 91, both increased in 26, and 18 had mixed signs. In 24 configurations, both PPL values increased despite lower QK error. Across the 315 individual task scores, 190 increased, 109 decreased, and 16 tied; tasks within a configuration are not independent replications.

\subsection{Aggregate comparisons for QK-Wanda}
\label{app:unmasked_aggregate_counts}
The following counts use the QK-Wanda results and the common QK reconstruction metric. They compare unrounded values, exclude dense references, and are descriptive rather than independent statistical trials.
\begin{table}[H]
\centering\small
\caption{Favorable comparisons for QK-Wanda. Shared QK budgets are compared with row-wise Wanda unless indicated; accuracy is the unweighted seven-task mean. Historical masked counts above remain separately labeled.}
\label{tab:unmasked_aggregate_counts}
\begin{tabular}{@{}lrrrr@{}}
\toprule
Comparison & WikiText-2 & C4 & Full QK error & Accuracy\\
\midrule
All sparsities & 103/135 & 101/135 & 133/135 & 30/45\\
Tabulated sparsities & 35/45 & 35/45 & 44/45 & 30/45\\
Shared vs. separate QK & 112/135 & 119/135 & 122/135 & \textemdash\\
Separate QK vs. row Wanda & 95/135 & 85/135 & 124/135 & \textemdash\\
\bottomrule
\end{tabular}
\end{table}

\FloatBarrier
\section{Zero-Shot Task Results}
\label{app:zero_shot_tables}

% Checkpoint-sized continuations; data retained from the original exporter.
\begin{table}[H]
\centering\fontsize{8}{9}\selectfont\setlength{\tabcolsep}{3pt}\renewcommand{\arraystretch}{1}
\caption{Seven-task zero-shot accuracy (\%) for TinyLlama, Llama~2, and Llama~3 with QK-only pruning. Each model includes its dense reference. \algname{QK-Wanda} uses shared or separate (sep.) budgets; other QK variants use shared budgets. Centered variants are described in Appendix~\ref{app:centered}. Metrics follow Appendix~\ref{app:zero_shot_protocol}; Avg. is their unweighted mean. Bold marks the best accuracy within each model and sparsity, including ties. Dashes mark variants not evaluated in this comparison.}
\label{tab:zero_shot_llama_all}
\begin{qktabular}{10}{lccccccccc}
\toprule
Method & Sparsity & BoolQ & RTE & \shortstack{Hella\\Swag} & \shortstack{Wino\\Grande} & \shortstack{ARC-\\Easy} & \shortstack{ARC-\\Challenge} & \shortstack{OpenBook\\QA} & Avg.\\
\midrule
% model: tinyllama_1b
\multicolumn{10}{l}{\textbf{TinyLlama-1.1B}}\\
\addlinespace[1pt]
Dense & 0\% & 57.77 & 51.99 & 59.16 & 59.04 & 55.35 & 30.20 & 36.00 & 49.93\\
\cmidrule(lr){1-10}
\algname{Wanda} & 30\% & 57.89 & 51.99 & 58.96 & 58.56 & 55.26 & 30.38 & \textbf{36.20} & 49.89\\
\algname{QK-Wanda-M} & 30\% & 58.50 & 52.35 & 59.11 & 58.41 & 55.30 & 30.46 & 36.00 & \textbf{50.02}\\
\algname{QK-Wanda (shared)} & 30\% & \textbf{58.53} & 51.62 & \textbf{59.12} & 58.48 & \textbf{55.43} & 30.46 & 36.00 & 49.95\\
\algname{QK-Wanda (sep.)} & 30\% & 58.01 & 51.99 & 59.06 & 58.64 & 55.30 & \textbf{30.72} & 36.00 & 49.96\\
\algname{QK-Wanda-MRC} & 30\% & 57.98 & 52.35 & 59.07 & \textbf{58.72} & 55.26 & 30.20 & 35.20 & 49.83\\
\algname{QK-Wanda-MC} & 30\% & 57.92 & \textbf{52.71} & 58.95 & 58.25 & 54.97 & 30.12 & \textbf{36.20} & 49.87\\
\addlinespace[1pt]
\algname{Wanda} & 50\% & 56.94 & \textbf{53.79} & 58.24 & 57.70 & 55.47 & \textbf{30.80} & 35.00 & \textbf{49.71}\\
\algname{QK-Wanda-M} & 50\% & 57.68 & 50.90 & \textbf{58.76} & \textbf{58.41} & 55.56 & 30.38 & 35.80 & 49.64\\
\algname{QK-Wanda (shared)} & 50\% & 57.34 & 50.54 & 58.73 & 58.17 & 55.47 & 30.12 & 36.00 & 49.48\\
\algname{QK-Wanda (sep.)} & 50\% & 57.09 & 49.46 & 58.40 & 58.25 & \textbf{55.85} & 30.63 & 35.20 & 49.27\\
\algname{QK-Wanda-MRC} & 50\% & \textbf{58.10} & 51.62 & 58.28 & 57.85 & 54.59 & 30.55 & \textbf{36.60} & 49.66\\
\algname{QK-Wanda-MC} & 50\% & 57.77 & 50.90 & 58.46 & 57.77 & 55.43 & 30.46 & 36.00 & 49.54\\
\addlinespace[1pt]
\algname{Wanda} & 80\% & 54.68 & \textbf{55.23} & 51.16 & 53.28 & 51.14 & 28.07 & 33.80 & 46.76\\
\algname{QK-Wanda-M} & 80\% & 55.17 & 52.35 & \textbf{54.73} & 55.64 & 52.65 & 29.27 & 33.40 & \textbf{47.60}\\
\algname{QK-Wanda (shared)} & 80\% & 54.92 & 51.99 & 54.70 & \textbf{55.88} & 52.78 & 29.18 & 33.20 & 47.52\\
\algname{QK-Wanda (sep.)} & 80\% & 55.75 & 53.07 & 52.31 & 54.93 & \textbf{53.03} & 28.50 & 33.00 & 47.23\\
\algname{QK-Wanda-MRC} & 80\% & 57.13 & 51.62 & 53.93 & 55.33 & 50.42 & 29.44 & \textbf{34.20} & 47.44\\
\algname{QK-Wanda-MC} & 80\% & \textbf{57.98} & 49.10 & 53.71 & 54.14 & 51.35 & \textbf{29.52} & 33.00 & 46.97\\
\bottomrule\end{qktabular}
\end{table}

\begin{table}[H]
\centering\fontsize{8}{9}\selectfont\setlength{\tabcolsep}{3pt}\renewcommand{\arraystretch}{1}
{\small Table~\ref{tab:zero_shot_llama_all} (continued).}\par\smallskip
\begin{qktabular}{10}{lccccccccc}
\toprule
Method & Sparsity & BoolQ & RTE & \shortstack{Hella\\Swag} & \shortstack{Wino\\Grande} & \shortstack{ARC-\\Easy} & \shortstack{ARC-\\Challenge} & \shortstack{OpenBook\\QA} & Avg.\\
\midrule
% model: llama2_7b
\multicolumn{10}{l}{\textbf{Llama-2-7B}}\\
\addlinespace[1pt]
Dense & 0\% & 77.74 & 62.82 & 75.98 & 69.06 & 74.54 & 46.33 & 44.20 & 64.38\\
\cmidrule(lr){1-10}
\algname{Wanda} & 30\% & 76.73 & 61.73 & 75.76 & 68.75 & \textbf{74.79} & 45.82 & 44.40 & 64.00\\
\algname{QK-Wanda-M} & 30\% & 77.46 & 62.45 & 75.99 & 68.67 & 74.75 & 46.16 & 44.20 & 64.24\\
\algname{QK-Wanda (shared)} & 30\% & 77.43 & 61.73 & 76.01 & \textbf{68.82} & \textbf{74.79} & 46.25 & 44.00 & 64.15\\
\algname{QK-Wanda (sep.)} & 30\% & 77.49 & \textbf{62.82} & 75.92 & 68.75 & 74.75 & 46.33 & 44.20 & \textbf{64.32}\\
\algname{QK-Wanda-MRC} & 30\% & 77.58 & 61.37 & \textbf{76.05} & 68.59 & 74.75 & 46.33 & \textbf{44.60} & 64.18\\
\algname{QK-Wanda-MC} & 30\% & \textbf{77.68} & 61.73 & 76.00 & 68.51 & 74.71 & \textbf{46.59} & 44.40 & 64.23\\
\addlinespace[1pt]
\algname{Wanda} & 50\% & 75.96 & \textbf{61.37} & 74.80 & 68.43 & 74.45 & 44.80 & 44.20 & 63.43\\
\algname{QK-Wanda-M} & 50\% & 76.57 & 59.57 & \textbf{75.99} & \textbf{68.90} & 75.17 & \textbf{46.42} & 44.80 & 63.92\\
\algname{QK-Wanda (shared)} & 50\% & 76.51 & 59.93 & 75.97 & 68.75 & 75.34 & 46.16 & \textbf{45.20} & 63.98\\
\algname{QK-Wanda (sep.)} & 50\% & 76.73 & 59.57 & 75.89 & 68.67 & 75.38 & 46.16 & \textbf{45.20} & 63.94\\
\algname{QK-Wanda-MRC} & 50\% & 76.73 & 60.65 & 75.87 & 68.82 & 75.17 & 45.56 & \textbf{45.20} & \textbf{64.00}\\
\algname{QK-Wanda-MC} & 50\% & \textbf{76.82} & 59.93 & 75.72 & \textbf{68.90} & \textbf{75.42} & 46.25 & 44.80 & 63.98\\
\addlinespace[1pt]
\algname{Wanda} & 80\% & 64.13 & 54.87 & 60.34 & 56.12 & 61.41 & 33.28 & 37.80 & 52.56\\
\algname{QK-Wanda-M} & 80\% & 72.48 & 53.79 & 69.23 & 63.61 & 66.33 & 37.80 & 41.40 & 57.81\\
\algname{QK-Wanda (shared)} & 80\% & 72.42 & 53.79 & 69.29 & 63.85 & 66.54 & 37.80 & 41.20 & 57.84\\
\algname{QK-Wanda (sep.)} & 80\% & \textbf{72.66} & 54.51 & \textbf{69.55} & 63.38 & 66.58 & 38.40 & 41.20 & 58.04\\
\algname{QK-Wanda-MRC} & 80\% & 72.11 & \textbf{55.96} & 69.04 & \textbf{64.17} & \textbf{67.97} & \textbf{39.85} & \textbf{42.60} & \textbf{58.81}\\
\algname{QK-Wanda-MC} & 80\% & 71.80 & 55.23 & 66.67 & 63.38 & 67.47 & 38.82 & 41.80 & 57.88\\
\bottomrule\end{qktabular}
\end{table}

\begin{table}[H]
\centering\fontsize{8}{9}\selectfont\setlength{\tabcolsep}{3pt}\renewcommand{\arraystretch}{1}
{\small Table~\ref{tab:zero_shot_llama_all} (continued).}\par\smallskip
\begin{qktabular}{10}{lccccccccc}
\toprule
Method & Sparsity & BoolQ & RTE & \shortstack{Hella\\Swag} & \shortstack{Wino\\Grande} & \shortstack{ARC-\\Easy} & \shortstack{ARC-\\Challenge} & \shortstack{OpenBook\\QA} & Avg.\\
\midrule
% model: llama2_13b
\multicolumn{10}{l}{\textbf{Llama-2-13B}}\\
\addlinespace[1pt]
Dense & 0\% & 80.61 & 65.34 & 79.38 & 72.14 & 77.53 & 49.15 & 45.20 & 67.05\\
\cmidrule(lr){1-10}
\algname{Wanda} & 30\% & 80.37 & 63.18 & 79.33 & 71.98 & 77.48 & \textbf{49.49} & 45.60 & 66.77\\
\algname{QK-Wanda-M} & 30\% & 80.73 & 63.90 & 79.42 & 72.22 & \textbf{77.57} & 48.98 & 45.60 & 66.92\\
\algname{QK-Wanda (shared)} & 30\% & 80.64 & 63.90 & 79.41 & 71.90 & 77.48 & 48.98 & 45.60 & 66.84\\
\algname{QK-Wanda (sep.)} & 30\% & 80.73 & 63.54 & \textbf{79.43} & 72.14 & \textbf{77.57} & 49.23 & 45.40 & 66.86\\
\algname{QK-Wanda-MRC} & 30\% & 80.89 & \textbf{64.26} & 79.33 & \textbf{72.38} & 77.53 & 49.06 & 45.40 & 66.98\\
\algname{QK-Wanda-MC} & 30\% & \textbf{80.92} & 63.90 & 79.29 & \textbf{72.38} & 77.48 & 49.15 & \textbf{45.80} & \textbf{66.99}\\
\addlinespace[1pt]
\algname{Wanda} & 50\% & 80.12 & 63.18 & 78.69 & 71.43 & 77.40 & 48.46 & 45.00 & 66.33\\
\algname{QK-Wanda-M} & 50\% & 80.98 & 63.90 & 79.39 & 71.03 & 76.98 & 48.89 & 46.00 & 66.74\\
\algname{QK-Wanda (shared)} & 50\% & 80.98 & 63.90 & 79.37 & 71.19 & 77.06 & 48.72 & 45.40 & 66.66\\
\algname{QK-Wanda (sep.)} & 50\% & 81.07 & 64.26 & \textbf{79.44} & 71.19 & 77.23 & 48.63 & 45.60 & 66.77\\
\algname{QK-Wanda-MRC} & 50\% & \textbf{81.22} & \textbf{65.34} & 79.34 & \textbf{72.22} & \textbf{77.44} & \textbf{48.98} & 45.80 & \textbf{67.19}\\
\algname{QK-Wanda-MC} & 50\% & \textbf{81.22} & 64.26 & 79.19 & \textbf{72.22} & 77.23 & 48.29 & \textbf{46.60} & 67.00\\
\addlinespace[1pt]
\algname{Wanda} & 80\% & 67.00 & 53.43 & 67.60 & 57.93 & 67.55 & 39.76 & 41.20 & 56.35\\
\algname{QK-Wanda-M} & 80\% & 79.17 & 62.82 & 75.35 & 71.35 & 71.59 & 43.52 & 44.80 & 64.09\\
\algname{QK-Wanda (shared)} & 80\% & 79.20 & 62.82 & 75.43 & 71.19 & 71.55 & 43.86 & 44.80 & 64.12\\
\algname{QK-Wanda (sep.)} & 80\% & 79.36 & \textbf{63.90} & \textbf{75.65} & 71.59 & 71.89 & 43.34 & 44.60 & 64.33\\
\algname{QK-Wanda-MRC} & 80\% & \textbf{79.88} & 63.18 & 75.20 & \textbf{71.74} & \textbf{72.98} & \textbf{44.80} & \textbf{46.80} & \textbf{64.94}\\
\algname{QK-Wanda-MC} & 80\% & 79.48 & 62.45 & 73.89 & 70.48 & 72.81 & 44.03 & \textbf{46.80} & 64.28\\
\bottomrule\end{qktabular}
\end{table}

\begin{table}[H]
\centering\fontsize{8}{9}\selectfont\setlength{\tabcolsep}{3pt}\renewcommand{\arraystretch}{1}
{\small Table~\ref{tab:zero_shot_llama_all} (continued).}\par\smallskip
\begin{qktabular}{10}{lccccccccc}
\toprule
Method & Sparsity & BoolQ & RTE & \shortstack{Hella\\Swag} & \shortstack{Wino\\Grande} & \shortstack{ARC-\\Easy} & \shortstack{ARC-\\Challenge} & \shortstack{OpenBook\\QA} & Avg.\\
\midrule
% model: llama3_1b
\multicolumn{10}{l}{\textbf{Llama-3.2-1B}}\\
\addlinespace[1pt]
Dense & 0\% & 64.01 & 56.68 & 63.67 & 60.62 & 60.44 & 36.35 & 37.20 & 54.14\\
\cmidrule(lr){1-10}
\algname{Wanda} & 30\% & 64.40 & \textbf{60.29} & 63.49 & \textbf{61.09} & 60.31 & 36.86 & 37.00 & \textbf{54.78}\\
\algname{QK-Wanda-M} & 30\% & 64.37 & 57.40 & 63.65 & 60.30 & 60.40 & 36.43 & \textbf{37.40} & 54.28\\
\algname{QK-Wanda (shared)} & 30\% & 64.22 & 56.68 & \textbf{63.66} & 60.14 & \textbf{60.65} & 36.52 & \textbf{37.40} & 54.18\\
\algname{QK-Wanda (sep.)} & 30\% & 64.10 & 57.40 & 63.49 & 59.83 & 60.61 & 36.77 & 37.00 & 54.17\\
\algname{QK-Wanda-MRC} & 30\% & 64.01 & 57.40 & 63.60 & 60.38 & 59.93 & 36.60 & \textbf{37.40} & 54.19\\
\algname{QK-Wanda-MC} & 30\% & \textbf{64.92} & 57.76 & 63.47 & 60.38 & 60.06 & \textbf{37.03} & \textbf{37.40} & 54.43\\
\addlinespace[1pt]
\algname{Wanda} & 50\% & 63.09 & \textbf{63.54} & 61.78 & \textbf{59.91} & \textbf{59.22} & 35.84 & 36.20 & \textbf{54.22}\\
\algname{QK-Wanda-M} & 50\% & 63.85 & 57.76 & 62.78 & 59.35 & 58.67 & 35.58 & 36.60 & 53.51\\
\algname{QK-Wanda (shared)} & 50\% & 63.82 & 58.84 & \textbf{62.79} & 59.19 & 58.67 & 35.84 & \textbf{37.00} & 53.74\\
\algname{QK-Wanda (sep.)} & 50\% & \textbf{64.13} & 55.60 & 61.53 & 58.96 & 58.71 & \textbf{35.92} & 36.80 & 53.09\\
\algname{QK-Wanda-MRC} & 50\% & 59.82 & 54.51 & 58.94 & 57.70 & 56.90 & 33.02 & 36.60 & 51.07\\
\algname{QK-Wanda-MC} & 50\% & 62.84 & 59.57 & 60.65 & 58.48 & 57.53 & 34.98 & 36.40 & 52.92\\
\addlinespace[1pt]
\algname{Wanda} & 80\% & 59.60 & 51.99 & \textbf{41.29} & \textbf{53.12} & \textbf{45.37} & 25.43 & 27.60 & \textbf{43.48}\\
\algname{QK-Wanda-M} & 80\% & 61.07 & 51.62 & 37.29 & 51.22 & 41.71 & 25.34 & \textbf{29.40} & 42.52\\
\algname{QK-Wanda (shared)} & 80\% & 61.22 & 51.62 & 37.08 & 51.46 & 41.41 & \textbf{25.94} & 28.60 & 42.48\\
\algname{QK-Wanda (sep.)} & 80\% & \textbf{61.83} & 51.26 & 33.40 & 49.96 & 38.55 & 23.04 & 25.20 & 40.46\\
\algname{QK-Wanda-MRC} & 80\% & 60.28 & \textbf{53.07} & 29.84 & 51.78 & 33.88 & 21.25 & 25.20 & 39.33\\
\algname{QK-Wanda-MC} & 80\% & 59.79 & 46.21 & 30.06 & 50.59 & 37.50 & 21.50 & 27.40 & 39.01\\
\bottomrule\end{qktabular}
\end{table}

\begin{table}[H]
\centering\fontsize{8}{9}\selectfont\setlength{\tabcolsep}{3pt}\renewcommand{\arraystretch}{1}
{\small Table~\ref{tab:zero_shot_llama_all} (continued).}\par\smallskip
\begin{qktabular}{10}{lccccccccc}
\toprule
Method & Sparsity & BoolQ & RTE & \shortstack{Hella\\Swag} & \shortstack{Wino\\Grande} & \shortstack{ARC-\\Easy} & \shortstack{ARC-\\Challenge} & \shortstack{OpenBook\\QA} & Avg.\\
\midrule
% model: llama3_3b
\multicolumn{10}{l}{\textbf{Llama-3.2-3B}}\\
\addlinespace[1pt]
Dense & 0\% & 73.24 & 54.51 & 73.59 & 70.01 & 71.72 & 46.08 & 43.00 & 61.74\\
\cmidrule(lr){1-10}
\algname{Wanda} & 30\% & \textbf{73.49} & 53.43 & 73.40 & 69.22 & 71.63 & 45.65 & 42.60 & 61.35\\
\algname{QK-Wanda-M} & 30\% & 73.21 & 54.15 & 73.56 & 69.30 & 71.63 & 46.42 & 42.60 & 61.55\\
\algname{QK-Wanda (shared)} & 30\% & 73.30 & 54.15 & \textbf{73.59} & \textbf{69.38} & 71.63 & 46.16 & \textbf{42.80} & 61.57\\
\algname{QK-Wanda (sep.)} & 30\% & 73.21 & \textbf{54.51} & 73.49 & 69.14 & \textbf{71.80} & \textbf{46.50} & \textbf{42.80} & \textbf{61.64}\\
\algname{QK-Wanda-MRC} & 30\% & 73.09 & \textbf{54.51} & 73.57 & 69.22 & 71.51 & 46.08 & 42.40 & 61.48\\
\algname{QK-Wanda-MC} & 30\% & 73.30 & 53.43 & 73.28 & \textbf{69.38} & 71.76 & 46.08 & 42.20 & 61.35\\
\addlinespace[1pt]
\algname{Wanda} & 50\% & 72.72 & 57.04 & 72.25 & 69.06 & 69.53 & 44.80 & \textbf{42.40} & 61.11\\
\algname{QK-Wanda-M} & 50\% & 72.78 & 54.87 & \textbf{73.31} & 68.59 & 70.16 & 45.82 & 41.40 & 60.99\\
\algname{QK-Wanda (shared)} & 50\% & 72.54 & 54.51 & 73.28 & 68.82 & 70.12 & 45.82 & 41.80 & 60.98\\
\algname{QK-Wanda (sep.)} & 50\% & \textbf{73.06} & 53.79 & 72.77 & 69.06 & \textbf{70.50} & \textbf{45.90} & 42.20 & 61.04\\
\algname{QK-Wanda-MRC} & 50\% & 71.59 & \textbf{62.09} & 72.82 & \textbf{69.22} & 70.03 & 44.88 & 41.40 & \textbf{61.72}\\
\algname{QK-Wanda-MC} & 50\% & 72.20 & 54.87 & 72.21 & 68.27 & 69.36 & 43.86 & 41.40 & 60.31\\
\addlinespace[1pt]
\algname{Wanda} & 80\% & 63.00 & \textbf{60.29} & 60.12 & \textbf{58.88} & 57.83 & \textbf{37.29} & \textbf{37.20} & \textbf{53.51}\\
\algname{QK-Wanda-M} & 80\% & 65.66 & 48.01 & 58.25 & 57.30 & 53.41 & 33.11 & 35.40 & 50.16\\
\algname{QK-Wanda (shared)} & 80\% & \textbf{66.09} & 48.01 & 58.79 & 57.30 & 53.24 & 33.19 & 35.20 & 50.26\\
\algname{QK-Wanda (sep.)} & 80\% & 64.01 & 48.01 & 56.78 & 55.49 & 54.67 & 33.11 & 32.60 & 49.24\\
\algname{QK-Wanda-MRC} & 80\% & 63.55 & 57.40 & \textbf{60.61} & 58.64 & \textbf{59.47} & 35.49 & 35.00 & 52.88\\
\algname{QK-Wanda-MC} & 80\% & 60.64 & 49.10 & 47.81 & 54.46 & 53.45 & 31.23 & 33.40 & 47.16\\
\bottomrule\end{qktabular}
\end{table}

\begin{table}[H]
\centering\fontsize{8}{9}\selectfont\setlength{\tabcolsep}{3pt}\renewcommand{\arraystretch}{1}
{\small Table~\ref{tab:zero_shot_llama_all} (continued).}\par\smallskip
\begin{qktabular}{10}{lccccccccc}
\toprule
Method & Sparsity & BoolQ & RTE & \shortstack{Hella\\Swag} & \shortstack{Wino\\Grande} & \shortstack{ARC-\\Easy} & \shortstack{ARC-\\Challenge} & \shortstack{OpenBook\\QA} & Avg.\\
\midrule
% model: llama3_8b
\multicolumn{10}{l}{\textbf{Llama-3.1-8B}}\\
\addlinespace[1pt]
Dense & 0\% & 82.11 & 69.68 & 78.88 & 73.64 & 81.10 & 53.50 & 44.80 & 69.10\\
\cmidrule(lr){1-10}
\algname{Wanda} & 30\% & 81.80 & 68.23 & 78.60 & 73.09 & 81.10 & 53.58 & \textbf{45.20} & 68.80\\
\algname{QK-Wanda-M} & 30\% & 82.32 & 70.76 & 79.05 & 73.64 & \textbf{81.19} & \textbf{53.67} & \textbf{45.20} & 69.40\\
\algname{QK-Wanda (shared)} & 30\% & 82.29 & 70.04 & 79.02 & 73.80 & 81.10 & 53.58 & \textbf{45.20} & 69.29\\
\algname{QK-Wanda (sep.)} & 30\% & 82.48 & 70.40 & 79.03 & 73.64 & 80.93 & 53.50 & \textbf{45.20} & 69.31\\
\algname{QK-Wanda-MRC} & 30\% & 82.39 & 70.40 & 79.01 & 73.48 & 80.77 & 53.41 & 44.80 & 69.18\\
\algname{QK-Wanda-MC} & 30\% & \textbf{82.75} & \textbf{71.48} & \textbf{79.08} & \textbf{74.35} & 81.06 & 53.50 & 45.00 & \textbf{69.60}\\
\addlinespace[1pt]
\algname{Wanda} & 50\% & 81.19 & 69.68 & 77.76 & 72.06 & 80.09 & 53.41 & 44.80 & 68.43\\
\algname{QK-Wanda-M} & 50\% & 82.54 & 71.84 & \textbf{79.13} & \textbf{73.48} & 80.56 & 53.58 & 44.60 & 69.39\\
\algname{QK-Wanda (shared)} & 50\% & 82.54 & 72.20 & 79.05 & 73.32 & \textbf{80.89} & 53.33 & 44.60 & 69.42\\
\algname{QK-Wanda (sep.)} & 50\% & 82.42 & 71.12 & 78.98 & 73.01 & 80.68 & 53.33 & 44.60 & 69.16\\
\algname{QK-Wanda-MRC} & 50\% & \textbf{82.66} & \textbf{72.92} & 78.91 & \textbf{73.48} & 79.80 & \textbf{54.52} & 45.00 & \textbf{69.61}\\
\algname{QK-Wanda-MC} & 50\% & 82.26 & 70.04 & 78.68 & 72.85 & 80.60 & 53.24 & \textbf{45.40} & 69.01\\
\addlinespace[1pt]
\algname{Wanda} & 80\% & 68.26 & 61.37 & 68.65 & 66.14 & 70.24 & 41.81 & 41.00 & 59.64\\
\algname{QK-Wanda-M} & 80\% & 66.64 & 59.93 & \textbf{73.47} & 66.22 & 70.37 & 43.94 & 42.20 & 60.40\\
\algname{QK-Wanda (shared)} & 80\% & 67.49 & 59.93 & 73.31 & 66.30 & \textbf{70.83} & 44.03 & \textbf{42.40} & 60.61\\
\algname{QK-Wanda (sep.)} & 80\% & 67.00 & 58.48 & 71.15 & 64.25 & 69.11 & 41.64 & 41.00 & 58.95\\
\algname{QK-Wanda-MRC} & 80\% & \textbf{70.86} & \textbf{64.98} & 72.89 & \textbf{66.85} & 70.62 & \textbf{44.28} & 42.20 & \textbf{61.81}\\
\algname{QK-Wanda-MC} & 80\% & 66.33 & 59.21 & 69.15 & 65.04 & 70.71 & 43.09 & 40.40 & 59.13\\
\bottomrule\end{qktabular}
\end{table}

\FloatBarrier
% Checkpoint-sized continuations; data retained from the original exporter.
\begin{table}[H]
\centering\fontsize{8}{9}\selectfont\setlength{\tabcolsep}{3pt}\renewcommand{\arraystretch}{1}
\caption{\algname{QK-Wanda} is shown with shared and separate (sep.) budgets; dashes mark variants not evaluated in this comparison. Zero-shot accuracy (\%) on Qwen2.5 base checkpoints with QK-only unstructured pruning. Each model section contains its dense reference and original row-wise \algname{Wanda} versus shared-budget \algname{QK-Wanda-M}. Centered \algname{QK-Wanda-MC} is additionally reported for 0.5B/1.5B/3B/7B (Appendix~\ref{app:centered}). Task metrics follow Appendix~\ref{app:zero_shot_protocol}; Avg. is their unweighted mean. Bold marks the highest displayed accuracy within each model and sparsity, including ties.}
\label{tab:zero_shot_qwen25_all}
\begin{qktabular}{10}{lccccccccc}
\toprule
Method & Sparsity & BoolQ & RTE & \shortstack{Hella\\Swag} & \shortstack{Wino\\Grande} & \shortstack{ARC-\\Easy} & \shortstack{ARC-\\Challenge} & \shortstack{OpenBook\\QA} & Avg.\\
\midrule
% model: qwen25_0p5b
\multicolumn{10}{l}{\textbf{Qwen2.5-0.5B}}\\
\addlinespace[1pt]
Dense & 0\% & 62.42 & 58.84 & 52.15 & 56.27 & 58.71 & 32.42 & 35.20 & 50.86\\
\cmidrule(lr){1-10}
\algname{Wanda} & 30\% & 61.10 & 58.48 & 51.95 & 56.04 & 58.29 & \textbf{32.51} & 35.00 & 50.48\\
\algname{QK-Wanda-M} & 30\% & 61.71 & 59.21 & 52.16 & 57.46 & 58.50 & 32.34 & 35.40 & 50.97\\
\algname{QK-Wanda (shared)} & 30\% & \textbf{62.11} & \textbf{59.57} & 52.13 & 56.99 & 58.59 & 32.25 & 35.20 & \textbf{50.98}\\
\algname{QK-Wanda (sep.)} & 30\% & 61.65 & 58.12 & 52.10 & \textbf{57.54} & 58.42 & 31.91 & 35.40 & 50.73\\
\algname{QK-Wanda-MRC} & 30\% & 61.93 & 58.84 & \textbf{52.21} & 56.67 & 58.46 & 32.00 & \textbf{35.60} & 50.82\\
\algname{QK-Wanda-MC} & 30\% & 61.07 & 56.32 & 52.09 & \textbf{57.54} & \textbf{58.80} & 31.74 & 35.40 & 50.42\\
\addlinespace[1pt]
\algname{Wanda} & 50\% & 60.86 & 56.32 & 51.63 & \textbf{56.99} & \textbf{61.36} & \textbf{33.11} & 34.60 & \textbf{50.69}\\
\algname{QK-Wanda-M} & 50\% & \textbf{62.51} & 57.76 & 52.00 & 55.56 & 59.13 & 32.00 & 35.00 & 50.57\\
\algname{QK-Wanda (shared)} & 50\% & 62.48 & 57.40 & 51.90 & 55.64 & 59.47 & 32.25 & 34.80 & 50.56\\
\algname{QK-Wanda (sep.)} & 50\% & 61.35 & \textbf{58.12} & 51.81 & 55.33 & 59.89 & 32.34 & \textbf{35.40} & 50.61\\
\algname{QK-Wanda-MRC} & 50\% & 60.86 & 57.76 & \textbf{52.02} & 55.49 & 58.80 & 32.59 & 35.20 & 50.39\\
\algname{QK-Wanda-MC} & 50\% & 61.44 & 56.68 & 51.95 & 55.80 & 58.29 & 32.08 & 35.00 & 50.18\\
\addlinespace[1pt]
\algname{Wanda} & 80\% & 56.54 & \textbf{60.29} & 47.49 & 53.67 & 54.17 & 28.24 & 31.40 & 47.40\\
\algname{QK-Wanda-M} & 80\% & 60.24 & 49.10 & 49.10 & 53.59 & 54.67 & 29.95 & 32.40 & 47.01\\
\algname{QK-Wanda (shared)} & 80\% & 60.61 & 50.90 & 49.11 & 53.83 & 54.42 & 29.95 & 32.60 & 47.35\\
\algname{QK-Wanda (sep.)} & 80\% & 49.51 & 49.46 & 47.55 & \textbf{54.62} & 54.00 & 28.33 & \textbf{33.80} & 45.32\\
\algname{QK-Wanda-MRC} & 80\% & \textbf{62.32} & 53.43 & \textbf{50.02} & 52.64 & \textbf{58.80} & \textbf{31.57} & 33.60 & \textbf{48.91}\\
\algname{QK-Wanda-MC} & 80\% & 57.43 & 51.62 & 49.27 & 54.06 & 55.09 & 31.48 & 32.60 & 47.37\\
\bottomrule\end{qktabular}
\end{table}

\begin{table}[H]
\centering\fontsize{8}{9}\selectfont\setlength{\tabcolsep}{3pt}\renewcommand{\arraystretch}{1}
{\small Table~\ref{tab:zero_shot_qwen25_all} (continued).}\par\smallskip
\begin{qktabular}{10}{lccccccccc}
\toprule
Method & Sparsity & BoolQ & RTE & \shortstack{Hella\\Swag} & \shortstack{Wino\\Grande} & \shortstack{ARC-\\Easy} & \shortstack{ARC-\\Challenge} & \shortstack{OpenBook\\QA} & Avg.\\
\midrule
% model: qwen25_1p5b
\multicolumn{10}{l}{\textbf{Qwen2.5-1.5B}}\\
\addlinespace[1pt]
Dense & 0\% & 72.94 & 70.04 & 67.77 & 63.30 & 71.51 & 45.14 & 40.80 & 61.64\\
\cmidrule(lr){1-10}
\algname{Wanda} & 30\% & 73.09 & 68.95 & 67.82 & 63.54 & \textbf{71.13} & 44.62 & \textbf{41.00} & 61.45\\
\algname{QK-Wanda-M} & 30\% & 73.09 & 70.76 & 67.75 & 63.85 & 70.62 & \textbf{44.97} & 40.80 & 61.69\\
\algname{QK-Wanda (shared)} & 30\% & 72.94 & \textbf{71.48} & 67.72 & \textbf{64.17} & 70.88 & 44.80 & 40.80 & \textbf{61.82}\\
\algname{QK-Wanda (sep.)} & 30\% & 72.48 & 70.40 & \textbf{67.86} & 63.85 & 70.79 & 44.80 & \textbf{41.00} & 61.60\\
\algname{QK-Wanda-MRC} & 30\% & \textbf{73.55} & 68.59 & \textbf{67.86} & 63.61 & 70.88 & 44.88 & \textbf{41.00} & 61.48\\
\algname{QK-Wanda-MC} & 30\% & 73.06 & 70.04 & 67.69 & 63.85 & 70.24 & 44.28 & 40.40 & 61.37\\
\addlinespace[1pt]
\algname{Wanda} & 50\% & \textbf{73.27} & \textbf{69.68} & 67.07 & 62.27 & \textbf{73.32} & 44.80 & 40.60 & \textbf{61.57}\\
\algname{QK-Wanda-M} & 50\% & 71.50 & 67.15 & 67.71 & 63.54 & 69.44 & 43.94 & \textbf{40.80} & 60.58\\
\algname{QK-Wanda (shared)} & 50\% & 71.38 & 67.15 & 67.68 & 63.85 & 69.32 & 43.69 & \textbf{40.80} & 60.55\\
\algname{QK-Wanda (sep.)} & 50\% & 71.10 & 68.95 & 67.50 & 63.14 & 70.20 & 44.03 & 40.40 & 60.76\\
\algname{QK-Wanda-MRC} & 50\% & 71.71 & 68.95 & \textbf{67.73} & 64.01 & 70.24 & \textbf{44.88} & 40.00 & 61.08\\
\algname{QK-Wanda-MC} & 50\% & 71.10 & 68.59 & 67.46 & \textbf{64.33} & 69.23 & 43.34 & 40.20 & 60.61\\
\addlinespace[1pt]
\algname{Wanda} & 80\% & 58.81 & 60.29 & 61.47 & 61.25 & \textbf{70.54} & 39.42 & 38.20 & 55.71\\
\algname{QK-Wanda-M} & 80\% & \textbf{68.59} & 61.01 & 63.19 & 60.93 & 67.26 & 39.85 & 38.60 & 57.06\\
\algname{QK-Wanda (shared)} & 80\% & 67.71 & 62.45 & 63.08 & 60.93 & 68.06 & \textbf{40.36} & 38.60 & \textbf{57.31}\\
\algname{QK-Wanda (sep.)} & 80\% & 64.40 & 62.09 & 61.04 & 59.67 & 68.64 & 39.76 & \textbf{39.20} & 56.40\\
\algname{QK-Wanda-MRC} & 80\% & 67.09 & 63.90 & \textbf{64.51} & \textbf{61.33} & 64.73 & 39.25 & 38.00 & 56.97\\
\algname{QK-Wanda-MC} & 80\% & 67.55 & \textbf{65.34} & 63.90 & 60.22 & 63.30 & 38.74 & 38.40 & 56.78\\
\bottomrule\end{qktabular}
\end{table}

\begin{table}[H]
\centering\fontsize{8}{9}\selectfont\setlength{\tabcolsep}{3pt}\renewcommand{\arraystretch}{1}
{\small Table~\ref{tab:zero_shot_qwen25_all} (continued).}\par\smallskip
\begin{qktabular}{10}{lccccccccc}
\toprule
Method & Sparsity & BoolQ & RTE & \shortstack{Hella\\Swag} & \shortstack{Wino\\Grande} & \shortstack{ARC-\\Easy} & \shortstack{ARC-\\Challenge} & \shortstack{OpenBook\\QA} & Avg.\\
\midrule
% model: qwen25_3b
\multicolumn{10}{l}{\textbf{Qwen2.5-3B}}\\
\addlinespace[1pt]
Dense & 0\% & 77.22 & 75.45 & 73.53 & 68.35 & 73.11 & 47.35 & 42.00 & 65.29\\
\cmidrule(lr){1-10}
\algname{Wanda} & 30\% & 77.13 & 73.65 & 73.46 & \textbf{69.38} & 72.90 & \textbf{47.61} & 42.40 & 65.22\\
\algname{QK-Wanda-M} & 30\% & \textbf{77.58} & 74.01 & 73.60 & 68.59 & 73.82 & 47.35 & 42.20 & 65.31\\
\algname{QK-Wanda (shared)} & 30\% & 77.37 & \textbf{74.73} & 73.57 & 68.27 & 73.78 & 47.44 & 42.20 & 65.34\\
\algname{QK-Wanda (sep.)} & 30\% & 77.49 & 74.01 & 73.57 & 68.67 & 73.06 & 47.01 & \textbf{42.60} & 65.20\\
\algname{QK-Wanda-MRC} & 30\% & 76.85 & 74.37 & \textbf{73.67} & 69.14 & \textbf{73.91} & \textbf{47.61} & \textbf{42.60} & \textbf{65.45}\\
\algname{QK-Wanda-MC} & 30\% & 76.97 & \textbf{74.73} & 73.58 & 68.35 & 73.40 & 47.44 & \textbf{42.60} & 65.30\\
\addlinespace[1pt]
\algname{Wanda} & 50\% & 77.06 & 72.20 & 72.88 & 68.59 & 73.15 & 46.93 & 42.20 & 64.72\\
\algname{QK-Wanda-M} & 50\% & 77.95 & 71.48 & 73.28 & 68.59 & \textbf{74.71} & 47.53 & 42.40 & 65.13\\
\algname{QK-Wanda (shared)} & 50\% & 77.80 & 71.12 & 73.34 & 68.75 & 74.62 & \textbf{47.95} & 42.20 & 65.11\\
\algname{QK-Wanda (sep.)} & 50\% & \textbf{78.10} & \textbf{72.92} & 73.24 & \textbf{68.82} & 74.20 & 47.70 & 42.40 & \textbf{65.34}\\
\algname{QK-Wanda-MRC} & 50\% & 77.40 & \textbf{72.92} & \textbf{73.60} & 68.67 & 73.02 & 47.27 & \textbf{42.80} & 65.10\\
\algname{QK-Wanda-MC} & 50\% & 77.28 & 72.20 & 73.25 & 68.51 & 72.69 & 47.27 & 42.40 & 64.80\\
\addlinespace[1pt]
\algname{Wanda} & 80\% & 71.04 & \textbf{77.26} & 68.01 & \textbf{65.43} & \textbf{73.06} & 43.77 & 40.20 & 62.68\\
\algname{QK-Wanda-M} & 80\% & 73.73 & 68.59 & 70.24 & 64.01 & 69.53 & 44.20 & 40.20 & 61.50\\
\algname{QK-Wanda (shared)} & 80\% & 74.65 & 67.87 & 70.25 & 63.61 & 69.23 & 43.86 & \textbf{41.00} & 61.50\\
\algname{QK-Wanda (sep.)} & 80\% & 72.32 & 67.15 & 67.53 & 61.33 & 70.45 & 43.09 & 40.00 & 60.27\\
\algname{QK-Wanda-MRC} & 80\% & 75.32 & 68.95 & \textbf{70.76} & 65.04 & 71.30 & \textbf{45.14} & \textbf{41.00} & 62.50\\
\algname{QK-Wanda-MC} & 80\% & \textbf{76.15} & 72.92 & 70.09 & 64.48 & 71.21 & 43.86 & 40.60 & \textbf{62.76}\\
\bottomrule\end{qktabular}
\end{table}

\begin{table}[H]
\centering\fontsize{8}{9}\selectfont\setlength{\tabcolsep}{3pt}\renewcommand{\arraystretch}{1}
{\small Table~\ref{tab:zero_shot_qwen25_all} (continued).}\par\smallskip
\begin{qktabular}{10}{lccccccccc}
\toprule
Method & Sparsity & BoolQ & RTE & \shortstack{Hella\\Swag} & \shortstack{Wino\\Grande} & \shortstack{ARC-\\Easy} & \shortstack{ARC-\\Challenge} & \shortstack{OpenBook\\QA} & Avg.\\
\midrule
% model: qwen25_7b
\multicolumn{10}{l}{\textbf{Qwen2.5-7B}}\\
\addlinespace[1pt]
Dense & 0\% & 84.65 & 81.59 & 78.94 & 73.09 & 77.36 & 51.11 & 47.20 & 70.56\\
\cmidrule(lr){1-10}
\algname{Wanda} & 30\% & \textbf{84.53} & \textbf{81.95} & 78.90 & 73.16 & 77.27 & 50.77 & 46.60 & 70.45\\
\algname{QK-Wanda-M} & 30\% & 84.46 & 81.59 & 78.99 & 72.93 & 77.31 & 50.77 & 47.00 & 70.44\\
\algname{QK-Wanda (shared)} & 30\% & \textbf{84.53} & 81.59 & 79.00 & 72.93 & \textbf{77.40} & \textbf{51.11} & 47.20 & 70.54\\
\algname{QK-Wanda (sep.)} & 30\% & 84.37 & 81.59 & 78.96 & 72.85 & 77.36 & 51.02 & 47.20 & 70.48\\
\algname{QK-Wanda-MRC} & 30\% & 84.28 & 81.59 & \textbf{79.01} & \textbf{73.48} & 77.36 & \textbf{51.11} & 47.00 & \textbf{70.55}\\
\algname{QK-Wanda-MC} & 30\% & 84.37 & 81.59 & 78.93 & 73.01 & 77.06 & 50.77 & \textbf{47.40} & 70.45\\
\addlinespace[1pt]
\algname{Wanda} & 50\% & 84.74 & \textbf{80.87} & 78.85 & 71.59 & 76.68 & 50.51 & \textbf{47.00} & 70.03\\
\algname{QK-Wanda-M} & 50\% & 84.62 & 80.51 & 79.07 & 71.51 & 76.98 & \textbf{51.28} & 46.20 & 70.02\\
\algname{QK-Wanda (shared)} & 50\% & 84.46 & 80.51 & 79.06 & 71.74 & \textbf{77.53} & 51.19 & 46.00 & 70.07\\
\algname{QK-Wanda (sep.)} & 50\% & \textbf{84.80} & \textbf{80.87} & \textbf{79.10} & 71.67 & 76.98 & 50.77 & 46.40 & \textbf{70.08}\\
\algname{QK-Wanda-MRC} & 50\% & 84.19 & 80.14 & 78.91 & 72.38 & 77.19 & 51.19 & 46.00 & 70.00\\
\algname{QK-Wanda-MC} & 50\% & 84.37 & 80.51 & 78.94 & \textbf{72.61} & 75.80 & 50.26 & 46.00 & 69.78\\
\addlinespace[1pt]
\algname{Wanda} & 80\% & 82.94 & 76.17 & 77.66 & 69.46 & 72.47 & 47.18 & 45.60 & 67.36\\
\algname{QK-Wanda-M} & 80\% & 82.91 & \textbf{77.98} & \textbf{78.67} & 68.90 & 73.27 & 49.06 & 45.40 & 68.03\\
\algname{QK-Wanda (shared)} & 80\% & 83.00 & \textbf{77.98} & 78.52 & 69.22 & 72.98 & 48.98 & 45.20 & 67.98\\
\algname{QK-Wanda (sep.)} & 80\% & 82.20 & 75.81 & 78.18 & 66.54 & 71.51 & 48.63 & \textbf{46.00} & 66.98\\
\algname{QK-Wanda-MRC} & 80\% & 83.00 & 77.26 & 78.62 & \textbf{70.48} & \textbf{74.96} & \textbf{49.32} & 45.60 & \textbf{68.46}\\
\algname{QK-Wanda-MC} & 80\% & \textbf{83.03} & 77.62 & 78.45 & 68.43 & 74.12 & 49.06 & 45.00 & 67.96\\
\bottomrule\end{qktabular}
\end{table}

\begin{table}[H]
\centering\fontsize{8}{9}\selectfont\setlength{\tabcolsep}{3pt}\renewcommand{\arraystretch}{1}
{\small Table~\ref{tab:zero_shot_qwen25_all} (continued).}\par\smallskip
\begin{qktabular}{10}{lccccccccc}
\toprule
Method & Sparsity & BoolQ & RTE & \shortstack{Hella\\Swag} & \shortstack{Wino\\Grande} & \shortstack{ARC-\\Easy} & \shortstack{ARC-\\Challenge} & \shortstack{OpenBook\\QA} & Avg.\\
\midrule
% model: qwen25_14b
\multicolumn{10}{l}{\textbf{Qwen2.5-14B}}\\
\addlinespace[1pt]
Dense & 0\% & 85.26 & 80.14 & 82.88 & 74.98 & 79.12 & 58.87 & 45.20 & 72.35\\
\cmidrule(lr){1-10}
\algname{Wanda} & 30\% & 85.41 & 80.14 & 82.88 & \textbf{76.40} & 79.21 & 58.70 & \textbf{45.60} & 72.62\\
\algname{QK-Wanda-M} & 30\% & 85.54 & 80.14 & 82.96 & 75.06 & 79.59 & 59.47 & 45.40 & 72.59\\
\algname{QK-Wanda (shared)} & 30\% & 85.44 & 80.14 & 82.93 & 75.61 & 79.50 & \textbf{59.56} & 45.40 & \textbf{72.66}\\
\algname{QK-Wanda (sep.)} & 30\% & \textbf{85.72} & \textbf{80.51} & 82.90 & 75.61 & 79.25 & 58.70 & 45.40 & 72.58\\
\algname{QK-Wanda-MRC} & 30\% & 85.66 & 80.14 & \textbf{83.00} & 74.59 & \textbf{79.80} & 58.79 & 45.20 & 72.45\\
\addlinespace[1pt]
\algname{Wanda} & 50\% & 84.19 & \textbf{81.59} & 82.82 & 75.69 & 80.68 & 58.70 & 46.20 & 72.84\\
\algname{QK-Wanda-M} & 50\% & \textbf{85.93} & 79.42 & 82.92 & 76.09 & 80.89 & \textbf{59.22} & 46.20 & 72.95\\
\algname{QK-Wanda (shared)} & 50\% & 85.66 & 80.14 & 82.99 & 76.09 & 80.64 & 58.96 & 46.00 & 72.93\\
\algname{QK-Wanda (sep.)} & 50\% & 85.38 & 80.87 & 82.95 & 76.09 & 79.67 & 58.79 & \textbf{46.60} & 72.91\\
\algname{QK-Wanda-MRC} & 50\% & 85.41 & 79.78 & \textbf{83.04} & \textbf{76.56} & \textbf{81.14} & \textbf{59.22} & 46.20 & \textbf{73.05}\\
\addlinespace[1pt]
\algname{Wanda} & 80\% & 77.49 & \textbf{80.87} & 79.94 & 72.93 & 82.20 & 55.03 & 46.00 & 70.64\\
\algname{QK-Wanda-M} & 80\% & \textbf{85.84} & 78.34 & 81.71 & 74.19 & 81.48 & 57.42 & 48.00 & 72.43\\
\algname{QK-Wanda (shared)} & 80\% & 85.63 & 78.34 & \textbf{81.73} & 73.95 & 82.37 & \textbf{57.51} & \textbf{48.40} & \textbf{72.56}\\
\algname{QK-Wanda (sep.)} & 80\% & 83.85 & 77.26 & 81.16 & 73.09 & \textbf{82.83} & 56.74 & 47.20 & 71.73\\
\algname{QK-Wanda-MRC} & 80\% & 84.46 & 78.70 & 81.63 & \textbf{75.06} & 80.56 & 57.17 & 45.60 & 71.88\\
\bottomrule\end{qktabular}
\end{table}

\begin{table}[H]
\centering\fontsize{8}{9}\selectfont\setlength{\tabcolsep}{3pt}\renewcommand{\arraystretch}{1}
{\small Table~\ref{tab:zero_shot_qwen25_all} (continued).}\par\smallskip
\begin{qktabular}{10}{lccccccccc}
\toprule
Method & Sparsity & BoolQ & RTE & \shortstack{Hella\\Swag} & \shortstack{Wino\\Grande} & \shortstack{ARC-\\Easy} & \shortstack{ARC-\\Challenge} & \shortstack{OpenBook\\QA} & Avg.\\
\midrule
% model: qwen25_32b
\multicolumn{10}{l}{\textbf{Qwen2.5-32B}}\\
\addlinespace[1pt]
Dense & 0\% & 87.16 & 81.59 & 84.07 & 75.30 & 78.28 & 55.72 & 44.40 & 72.36\\
\cmidrule(lr){1-10}
\algname{Wanda} & 30\% & 87.31 & \textbf{81.23} & 84.06 & \textbf{75.61} & \textbf{78.28} & \textbf{56.48} & 44.40 & \textbf{72.48}\\
\algname{QK-Wanda-M} & 30\% & 87.40 & 80.87 & 84.02 & 75.30 & 78.20 & 56.23 & 44.20 & 72.32\\
\algname{QK-Wanda (shared)} & 30\% & 87.43 & 80.87 & 84.03 & 75.22 & 78.16 & 56.23 & 44.40 & 72.33\\
\algname{QK-Wanda (sep.)} & 30\% & 87.34 & 80.87 & 84.05 & 75.45 & 78.03 & 56.14 & 44.40 & 72.33\\
\algname{QK-Wanda-MRC} & 30\% & \textbf{87.65} & 80.87 & \textbf{84.11} & 74.82 & 78.20 & 56.14 & \textbf{44.80} & 72.37\\
\addlinespace[1pt]
\algname{Wanda} & 50\% & 87.13 & \textbf{82.31} & 83.80 & 75.85 & \textbf{78.07} & \textbf{56.48} & 45.20 & \textbf{72.69}\\
\algname{QK-Wanda-M} & 50\% & 87.19 & 80.14 & 83.89 & 75.85 & 77.69 & 56.23 & \textbf{45.80} & 72.40\\
\algname{QK-Wanda (shared)} & 50\% & 86.97 & 80.51 & 83.93 & 75.69 & 77.61 & 55.97 & 45.60 & 72.33\\
\algname{QK-Wanda (sep.)} & 50\% & 86.54 & 80.51 & \textbf{83.95} & \textbf{76.40} & 77.78 & 56.14 & 45.20 & 72.36\\
\algname{QK-Wanda-MRC} & 50\% & \textbf{87.49} & 80.87 & \textbf{83.95} & 75.77 & 78.03 & 56.06 & 44.80 & 72.42\\
\addlinespace[1pt]
\algname{Wanda} & 80\% & 84.19 & 80.51 & 81.75 & 73.95 & 75.55 & \textbf{54.27} & 45.40 & 70.80\\
\algname{QK-Wanda-M} & 80\% & 84.68 & \textbf{80.87} & 82.75 & 74.66 & 74.33 & 53.92 & 45.60 & 70.97\\
\algname{QK-Wanda (shared)} & 80\% & 84.50 & 80.51 & 82.75 & 74.66 & 74.16 & 53.67 & 45.80 & 70.86\\
\algname{QK-Wanda (sep.)} & 80\% & 84.40 & 77.98 & 82.54 & \textbf{76.09} & \textbf{75.88} & \textbf{54.27} & 46.00 & 71.02\\
\algname{QK-Wanda-MRC} & 80\% & \textbf{84.74} & 80.14 & \textbf{83.11} & 75.14 & 75.42 & 54.01 & \textbf{47.20} & \textbf{71.39}\\
\bottomrule\end{qktabular}
\end{table}

\FloatBarrier
% Generated by scripts/export_zero_shot_grouped.py.
\begin{table}[H]
\centering
\fontsize{8}{9}\selectfont
\setlength{\tabcolsep}{3pt}
\renewcommand{\arraystretch}{1}
\caption{Unmasked \algname{QK-Wanda} is shown with shared and separate (sep.) budgets; dashes mark variants not evaluated in this comparison. Zero-shot accuracy (\%) on Llama~2 and Llama~3.1 70B with QK-only unstructured pruning. Each model section contains its dense reference and original row-wise \algname{Wanda} versus shared-budget \algname{QK-Wanda-M}. Task metrics follow Appendix~\ref{app:zero_shot_protocol}; Avg. is their unweighted mean. Bold marks the highest displayed accuracy within each model and sparsity, including ties.}
\label{tab:zero_shot_llama70b}
\begin{qktabular}{10}{lccccccccc}
\toprule
Method & Sparsity & BoolQ & RTE & \shortstack{Hella\\Swag} & \shortstack{Wino\\Grande} & \shortstack{ARC-\\Easy} & \shortstack{ARC-\\Challenge} & \shortstack{OpenBook\\QA} & Avg.\\
\midrule
% model: llama2_70b
\multicolumn{10}{l}{\textbf{Llama-2-70B}}\\
\addlinespace[1pt]
Dense & 0\% & 83.76 & 67.87 & 83.81 & 77.98 & 80.98 & 57.34 & 48.80 & 71.50\\
\cmidrule(lr){1-10}
\algname{Wanda} & 30\% & 83.43 & 68.95 & 83.72 & 77.82 & 80.72 & \textbf{57.42} & 48.40 & 71.50\\
\algname{QK-Wanda-M} & 30\% & 83.64 & 68.95 & 83.79 & 77.90 & \textbf{81.02} & \textbf{57.42} & \textbf{48.60} & 71.62\\
\algname{QK-Wanda (shared)} & 30\% & 83.64 & 68.59 & 83.81 & 77.90 & \textbf{81.02} & \textbf{57.42} & 48.40 & 71.54\\
\algname{QK-Wanda (sep.)} & 30\% & \textbf{83.73} & 68.95 & 83.78 & 77.82 & \textbf{81.02} & \textbf{57.42} & 48.40 & 71.59\\
\algname{QK-Wanda-MRC} & 30\% & 83.64 & \textbf{69.31} & \textbf{83.84} & \textbf{77.98} & 80.98 & \textbf{57.42} & 48.40 & \textbf{71.65}\\
\addlinespace[1pt]
\algname{Wanda} & 50\% & 82.63 & 69.31 & 83.38 & \textbf{77.98} & 80.89 & 57.00 & 48.40 & 71.37\\
\algname{QK-Wanda-M} & 50\% & 83.94 & 68.95 & \textbf{83.79} & 77.90 & 81.02 & 57.17 & \textbf{48.60} & 71.62\\
\algname{QK-Wanda (shared)} & 50\% & 83.94 & 68.59 & 83.74 & 77.82 & 81.06 & 57.42 & \textbf{48.60} & 71.60\\
\algname{QK-Wanda (sep.)} & 50\% & 83.73 & 69.31 & 83.73 & 77.58 & 81.14 & 57.25 & \textbf{48.60} & 71.62\\
\algname{QK-Wanda-MRC} & 50\% & \textbf{84.10} & \textbf{71.48} & 83.42 & \textbf{77.98} & \textbf{81.31} & \textbf{57.85} & 47.80 & \textbf{71.99}\\
\addlinespace[1pt]
\algname{Wanda} & 80\% & 76.51 & 65.70 & 76.66 & 68.03 & 78.11 & 50.34 & 45.60 & 65.85\\
\algname{QK-Wanda-M} & 80\% & \textbf{84.04} & 70.40 & \textbf{82.54} & 77.27 & 81.44 & \textbf{56.66} & \textbf{50.00} & 71.76\\
\algname{QK-Wanda (shared)} & 80\% & 83.91 & 71.12 & \textbf{82.54} & 77.19 & 81.48 & 56.48 & 49.80 & 71.79\\
\algname{QK-Wanda (sep.)} & 80\% & 83.70 & 71.48 & 81.96 & 77.19 & \textbf{81.65} & \textbf{56.66} & \textbf{50.00} & \textbf{71.80}\\
\algname{QK-Wanda-MRC} & 80\% & 83.18 & \textbf{72.56} & 79.67 & \textbf{78.85} & 79.59 & 55.03 & 47.00 & 70.84\\
\midrule
% model: llama3_70b
\multicolumn{10}{l}{\textbf{Llama-3.1-70B}}\\
\addlinespace[1pt]
Dense & 0\% & 85.44 & 69.68 & 85.03 & 79.95 & 86.66 & 64.93 & 47.80 & 74.21\\
\cmidrule(lr){1-10}
\algname{Wanda} & 30\% & 85.35 & 70.04 & 84.86 & 79.56 & 86.36 & \textbf{64.76} & \textbf{48.20} & 74.16\\
\algname{QK-Wanda-M} & 30\% & 85.26 & \textbf{70.76} & 84.85 & 79.64 & \textbf{86.62} & 64.33 & 47.80 & 74.18\\
\algname{QK-Wanda (shared)} & 30\% & 85.38 & \textbf{70.76} & 84.82 & \textbf{79.87} & 86.53 & 64.42 & 47.80 & \textbf{74.23}\\
\algname{QK-Wanda (sep.)} & 30\% & 85.41 & 70.04 & \textbf{84.92} & 79.72 & 86.53 & 64.16 & 47.80 & 74.08\\
\algname{QK-Wanda-MRC} & 30\% & \textbf{85.72} & 70.04 & 84.75 & 79.48 & 86.28 & 64.42 & 47.80 & 74.07\\
\addlinespace[1pt]
\algname{Wanda} & 50\% & 84.74 & 70.04 & 84.32 & 79.01 & 85.94 & 62.97 & \textbf{48.20} & 73.60\\
\algname{QK-Wanda-M} & 50\% & 85.11 & 70.04 & 84.54 & 79.24 & \textbf{86.20} & 63.82 & 47.60 & \textbf{73.79}\\
\algname{QK-Wanda (shared)} & 50\% & \textbf{85.20} & 69.68 & \textbf{84.67} & 79.24 & 86.15 & 63.82 & 47.60 & 73.77\\
\algname{QK-Wanda (sep.)} & 50\% & 84.65 & \textbf{70.76} & 84.41 & 78.85 & 85.94 & \textbf{64.08} & 46.80 & 73.64\\
\algname{QK-Wanda-MRC} & 50\% & 84.80 & 70.40 & 83.82 & \textbf{79.40} & 85.77 & 63.40 & 48.00 & 73.66\\
\addlinespace[1pt]
\algname{Wanda} & 80\% & 76.06 & 68.59 & \textbf{76.62} & \textbf{69.14} & 80.89 & 54.01 & \textbf{41.20} & \textbf{66.64}\\
\algname{QK-Wanda-M} & 80\% & 77.80 & \textbf{71.12} & 73.85 & 64.88 & \textbf{81.61} & \textbf{55.72} & 40.20 & 66.45\\
\algname{QK-Wanda (shared)} & 80\% & \textbf{78.07} & 68.59 & 73.83 & 64.25 & 81.27 & 55.12 & 40.40 & 65.93\\
\algname{QK-Wanda (sep.)} & 80\% & 73.73 & 67.15 & 69.12 & 58.25 & 76.43 & 49.91 & 32.80 & 61.06\\
\algname{QK-Wanda-MRC} & 80\% & 74.59 & 70.40 & 70.17 & 64.40 & 74.58 & 46.93 & 37.20 & 62.61\\
\bottomrule
\end{qktabular}
\end{table}

% Generated by scripts/export_zero_shot_llama2_13b.py.
\begin{table}[htbp]
\centering
\fontsize{8}{9}\selectfont
\setlength{\tabcolsep}{2pt}
\caption{Unmasked \algname{QK-Wanda} is shown with shared and separate (sep.) budgets; dashes mark variants not evaluated in this comparison. Zero-shot accuracy (\%) on Qwen2.5-72B with QK-only unstructured pruning. Dense is the unpruned reference; \algname{QK-Wanda-M} uses shared QK budgets. We report accuracy for BoolQ, RTE, and WinoGrande, and length-normalized accuracy for HellaSwag, ARC-Easy/Challenge, and OpenBookQA. Avg. is the unweighted mean of the seven unrounded scores. Higher is better; bold marks the best displayed result at each sparsity, including ties.}
\label{tab:zero_shot_qwen25_72b}
\begin{qktabular}{10}{lccccccccc}
\toprule
Method & Sparsity & BoolQ & RTE & \shortstack{Hella\\Swag} & \shortstack{Wino\\Grande} & \shortstack{ARC-\\Easy} & \shortstack{ARC-\\Challenge} & \shortstack{OpenBook\\QA} & Avg.\\
\midrule
Dense & 0\% & 89.17 & 77.26 & 86.06 & 78.14 & 83.38 & 62.63 & 46.60 & 74.75\\
\midrule
\algname{Wanda} & 30\% & \textbf{89.27} & 77.26 & \textbf{86.13} & 77.51 & \textbf{83.33} & 62.46 & 46.80 & \textbf{74.68}\\
\algname{QK-Wanda-M} & 30\% & 89.17 & 77.26 & 86.04 & 77.58 & 83.16 & 62.46 & 46.60 & 74.61\\
\algname{QK-Wanda (shared)} & 30\% & 89.05 & \textbf{77.62} & 86.05 & 77.58 & 83.04 & 62.46 & 46.80 & 74.66\\
\algname{QK-Wanda (sep.)} & 30\% & 89.14 & \textbf{77.62} & 86.04 & 77.19 & 83.21 & \textbf{62.63} & 46.60 & 74.63\\
\algname{QK-Wanda-MRC} & 30\% & 89.05 & 77.26 & 86.08 & \textbf{77.66} & 83.29 & 62.12 & \textbf{47.00} & 74.64\\
\addlinespace[2pt]
\algname{Wanda} & 50\% & 89.11 & \textbf{76.90} & 85.90 & 77.82 & 83.42 & 62.12 & \textbf{47.80} & 74.72\\
\algname{QK-Wanda-M} & 50\% & \textbf{89.42} & 76.53 & 85.96 & 77.58 & 83.84 & 61.86 & 46.60 & 74.54\\
\algname{QK-Wanda (shared)} & 50\% & 89.33 & 76.17 & 85.85 & 77.90 & 83.71 & 61.69 & 46.40 & 74.44\\
\algname{QK-Wanda (sep.)} & 50\% & \textbf{89.42} & 76.17 & 85.97 & \textbf{78.06} & \textbf{84.05} & \textbf{62.29} & 47.20 & \textbf{74.74}\\
\algname{QK-Wanda-MRC} & 50\% & 89.14 & 76.53 & \textbf{85.99} & 77.66 & 83.33 & 62.03 & 47.00 & 74.53\\
\addlinespace[2pt]
\algname{Wanda} & 80\% & 87.55 & \textbf{77.26} & 83.90 & 76.40 & 83.25 & 60.15 & \textbf{45.60} & 73.44\\
\algname{QK-Wanda-M} & 80\% & 89.27 & 75.09 & 85.06 & 76.24 & \textbf{85.23} & 62.29 & 44.60 & 73.97\\
\algname{QK-Wanda (shared)} & 80\% & 89.17 & 75.45 & 85.09 & 76.09 & 84.89 & 62.12 & 45.20 & 74.00\\
\algname{QK-Wanda (sep.)} & 80\% & \textbf{89.42} & 74.37 & 84.65 & 75.93 & 84.39 & \textbf{62.46} & 45.40 & 73.80\\
\algname{QK-Wanda-MRC} & 80\% & 88.84 & 76.17 & \textbf{85.26} & \textbf{76.80} & 84.30 & 61.77 & 45.40 & \textbf{74.08}\\
\bottomrule
\end{qktabular}
\end{table}

\FloatBarrier

\FloatBarrier
\section{Perplexity Across Sparsity Levels}
\label{app:budget_curves}

Each checkpoint has WikiText-2 and C4 curves under the protocol in \cref{sec:setup}. The left panel covers 0--50\% sparsity and the right 50--90\%, with independent logarithmic PPL ranges. Dense points are shared references. Every curve contains all nine measured nonzero sparsities.

The figures include row-wise Wanda, QK-Wanda and QK-Wanda-M with separate and shared budgets, and shared-budget QK-Wanda-MRC. QK-Wanda-MC is included where evaluated (Appendix~\ref{app:centered}); the largest checkpoints also include matrix-budget Wanda (Appendix~\ref{app:matrix_control}). Llama-2-7B additionally includes row-wise QK-Wanda-M and separate/shared QK-Wanda-MR. The legends identify the methods in each figure.

\subsection{Llama 2 13B}
\label{app:budget_llama2_13b}

\begin{figure}[H]
\centering
\includegraphics[width=\textwidth]{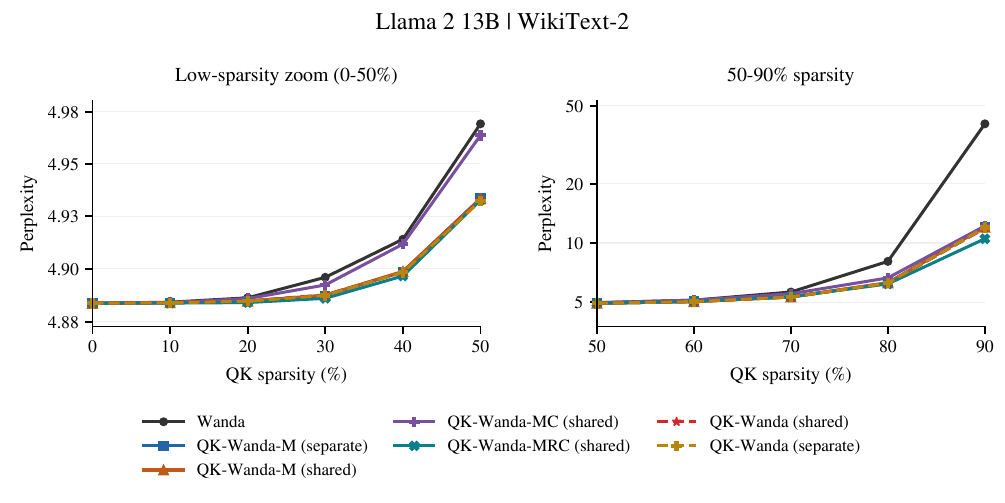}
\caption{Llama-2-13B: WikiText-2 PPL across QK sparsity.}
\label{fig:qk_budget_llama2_13b}
\end{figure}

\begin{figure}[H]
\centering
\includegraphics[width=\textwidth]{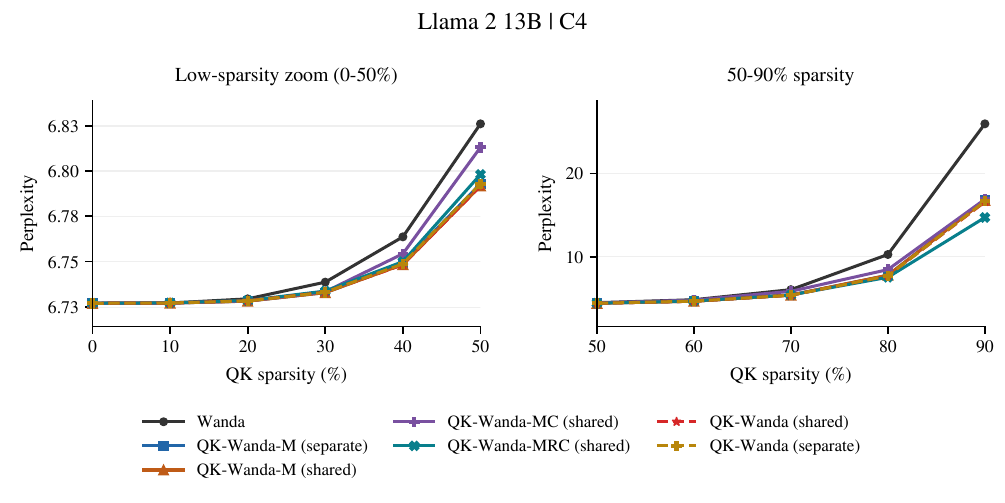}
\caption{Llama-2-13B: C4 PPL across QK sparsity.}
\label{fig:qk_budget_llama2_13b_c4}
\end{figure}

\FloatBarrier
\subsection{Llama 2 7B}
\label{app:budget_llama2_7b}

\begin{figure}[H]
\centering
\includegraphics[width=\textwidth]{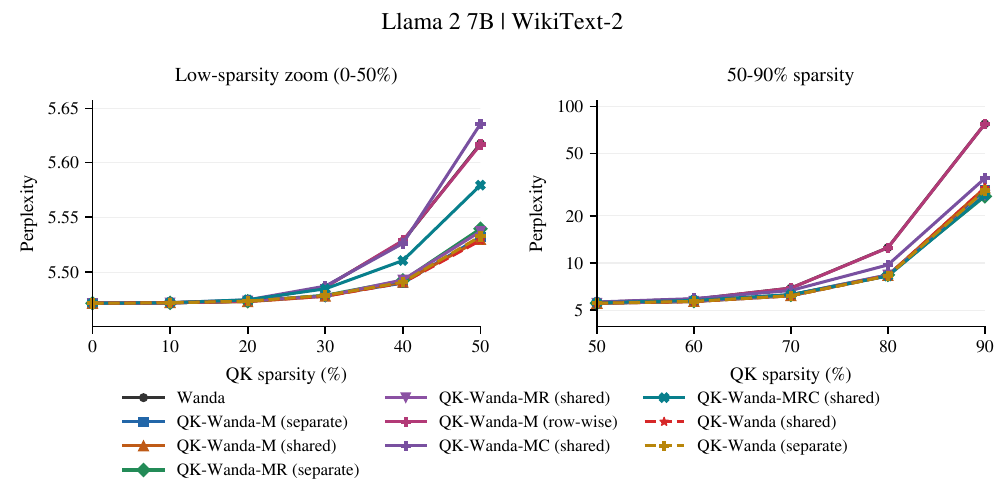}
\caption{Llama-2-7B: WikiText-2 PPL across QK sparsity.}
\label{fig:qk_budget_llama2_7b}
\end{figure}

\begin{figure}[H]
\centering
\includegraphics[width=\textwidth]{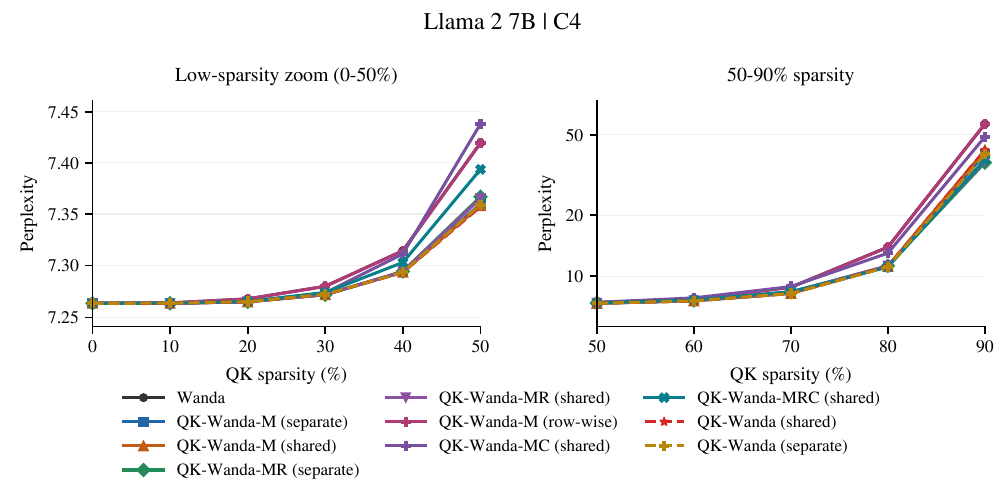}
\caption{Llama-2-7B: C4 PPL across QK sparsity.}
\label{fig:qk_budget_llama2_7b_c4}
\end{figure}

\FloatBarrier
\paragraph{Scoring and allocation ablation on Llama-2-7B.}
We evaluated causal pre-RoPE scores with row-wise, separate matrix, and shared QK budgets, and the RoPE-aware scores from Appendix~\ref{app:variants} with separate and shared budgets. All runs used the same checkpoint, frozen calibration tokens, context lengths, seed, FP16 precision, and ceiling rule for deletion counts. At sparsity $\rho$, row-wise QK-Wanda-M removed exactly $\lceil4096\rho\rceil$ weights from every Q and K row, matching Wanda's integer quotas. This control fixes allocation granularity while changing the score; causal scores need not preserve within-row rankings (Appendix~\ref{app:row_equivalence}). Row-wise pruning and its PPL and zero-shot evaluations ran on NVIDIA H200 GPUs. The existing Wanda, separate/shared causal, and RoPE pruning sweeps ran on A100 GPUs.

At 30\%, 50\%, and 80\%, we replayed the saved masks for the seven zero-shot tasks in Appendix~\ref{app:zero_shot_protocol}, including the original separate-budget causal masks. This table retains the causal QK diagnostic on common dense inputs: each block pools squared errors and reference energies over the allowed pairs in \cref{eq:causal_objective}, and the block ratios are averaged as in \cref{eq:relative_qk_error}; it does not measure the RoPE-aware objective itself. These errors were evaluated on the original A100 hardware, with a repeated shared-budget control reproducing the published block-level errors. \Cref{tab:llama2_rope_ablation} includes all six methods and the dense reference.

\begin{table}[htbp]
\centering\small
\setlength{\tabcolsep}{4pt}
\caption{Scoring and budget ablation on Llama-2-7B with QK-only unstructured pruning. QK-Wanda-M uses causal pre-RoPE scores with row-wise, separate matrix, or shared QK budgets. QK-Wanda-MR additionally includes RoPE, with separate or shared budgets. Wanda uses row-wise budgets. This specialized ablation retains the historical causal pre-RoPE error: the block-relative squared norm uses only causally allowed pairs in both numerator and denominator. It is distinct from the full-product diagnostic in \cref{eq:relative_qk_error}. Bold selects the best displayed result among six methods at each sparsity, including ties.}
\label{tab:llama2_rope_ablation}
\begin{qktabular}{6}{lccccc}
\toprule Method & Sparsity & WikiText-2 PPL & C4 PPL & Avg. acc. (\%) & Causal QK error (\%)\\
\midrule
Dense & 0\% & 5.47 & 7.26 & 64.38 & 0.00\\
\midrule
\algname{Wanda} & 30\% & 5.49 & 7.28 & 64.00 & 0.10\\
\algname{QK-Wanda-M} (row-wise) & 30\% & 5.49 & 7.28 & 63.95 & 0.10\\
\algname{QK-Wanda-M} (separate) & 30\% & \textbf{5.48} & \textbf{7.27} & \textbf{64.32} & \textbf{0.03}\\
\algname{QK-Wanda-M} (shared) & 30\% & \textbf{5.48} & \textbf{7.27} & 64.24 & \textbf{0.03}\\
\algname{QK-Wanda-MR} (separate) & 30\% & \textbf{5.48} & \textbf{7.27} & 64.15 & 0.04\\
\algname{QK-Wanda-MR} (shared) & 30\% & \textbf{5.48} & \textbf{7.27} & 64.21 & 0.04\\
\addlinespace
\algname{Wanda} & 50\% & 5.62 & 7.42 & 63.43 & 0.72\\
\algname{QK-Wanda-M} (row-wise) & 50\% & 5.62 & 7.42 & 63.42 & 0.72\\
\algname{QK-Wanda-M} (separate) & 50\% & \textbf{5.53} & \textbf{7.36} & 63.91 & \textbf{0.28}\\
\algname{QK-Wanda-M} (shared) & 50\% & \textbf{5.53} & \textbf{7.36} & 63.92 & 0.29\\
\algname{QK-Wanda-MR} (separate) & 50\% & 5.54 & 7.37 & 63.98 & 0.34\\
\algname{QK-Wanda-MR} (shared) & 50\% & 5.54 & 7.37 & \textbf{64.02} & 0.35\\
\addlinespace
\algname{Wanda} & 80\% & 12.54 & 13.94 & 52.56 & 7.76\\
\algname{QK-Wanda-M} (row-wise) & 80\% & 12.51 & 13.90 & 52.51 & 7.78\\
\algname{QK-Wanda-M} (separate) & 80\% & \textbf{8.32} & \textbf{11.17} & 58.03 & \textbf{3.40}\\
\algname{QK-Wanda-M} (shared) & 80\% & 8.33 & 11.21 & 57.81 & 3.44\\
\algname{QK-Wanda-MR} (separate) & 80\% & 8.38 & 11.25 & \textbf{58.13} & 4.29\\
\algname{QK-Wanda-MR} (shared) & 80\% & 8.39 & 11.27 & 57.82 & 4.33\\
\bottomrule
\end{qktabular}
\end{table}

\FloatBarrier
\paragraph{Effect of allocation.}
Row-wise causal QK-Wanda-M closely tracked Wanda: absolute PPL differences were below 0.001 on both datasets through 60\% sparsity. At 80\%, its WikiText-2/C4 PPL was 12.51/13.90, compared with Wanda's 12.54/13.94; at 90\%, the differences remained below 0.10 and had opposite signs across datasets. Its mean zero-shot accuracy was lower by 0.05, 0.01, and 0.05 percentage points at 30\%, 50\%, and 80\%, and its QK error was slightly higher at all three levels before rounding. Given the single seed and different GPU platforms for PPL and task evaluation, these small changes do not establish a reliable advantage from changing the score under fixed row quotas.

Allowing the causal QK scores to redistribute deletions across rows had a much larger effect on this checkpoint. At 80\%, separate and shared budgets achieved WikiText-2/C4 PPL of 8.32/11.17 and 8.33/11.21, respectively, and mean accuracy of 58.03\% and 57.81\%, versus 52.51\% for row-wise QK-Wanda-M. Separate budgets slightly outperformed shared budgets at this sparsity and had lower QK error at all three tabulated sparsities; shared budgets achieved lower PPL at 30\% and 50\%. Thus, freeing the per-row quota mattered substantially here, while sharing the budget between Q and K was not uniformly better than separate matrix budgets.

\paragraph{Effect of RoPE.}
Adding RoPE did not consistently improve perplexity. At matched budgets, it lowered WikiText-2 PPL at 3 of 9 nonzero sparsities for separate budgets and 2 of 9 for shared budgets; on C4, it improved 3 of 9 for each budget. Below 40\% sparsity, all absolute changes were smaller than 0.001 PPL. From 40\% through 80\%, RoPE increased PPL on both datasets for both budgets. At 90\%, however, it reduced WikiText-2/C4 PPL from 29.23/40.29 to 26.66/36.40 with separate budgets and from 30.14/41.90 to 27.19/37.14 with a shared budget. Thus, RoPE helped at the most aggressive pruning level, despite being slightly worse over the intermediate range.

Zero-shot changes were small and mixed. At a matched shared budget, adding RoPE changed mean accuracy by $-0.03$, $+0.10$, and $+0.01$ percentage points at 30\%, 50\%, and 80\% sparsity, respectively. At 80\%, RoPE with separate budgets achieved 58.13\%, compared with 57.82\% for RoPE with a shared budget. Both RoPE variants had higher pre-RoPE QK error than the shared causal method at all three tabulated sparsities, while retaining lower error than Wanda. This comparison concerns a common pre-RoPE diagnostic, not whether the RoPE-aware objective was better preserved. Within this causal ablation, RoPE helps at 90\% sparsity but does not provide a consistent improvement across the sweep.

\FloatBarrier
\subsection{TinyLlama 1.1B}
\label{app:budget_tinyllama_1b}

\begin{figure}[H]
\centering
\includegraphics[width=\textwidth]{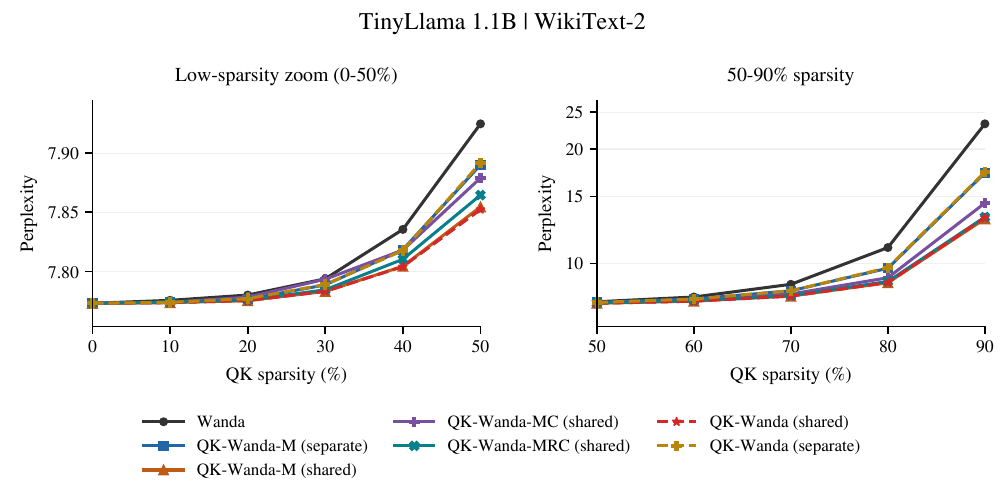}
\caption{TinyLlama-1.1B: WikiText-2 PPL across QK sparsity.}
\label{fig:qk_budget_tinyllama_1b}
\end{figure}

\begin{figure}[H]
\centering
\includegraphics[width=\textwidth]{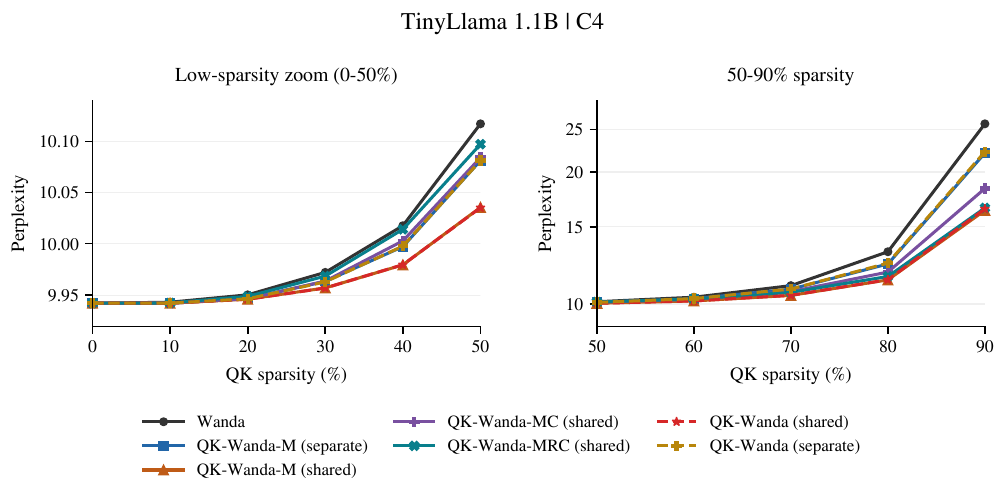}
\caption{TinyLlama-1.1B: C4 PPL across QK sparsity.}
\label{fig:qk_budget_tinyllama_1b_c4}
\end{figure}

\FloatBarrier
\subsection{Llama 3.2 1B}
\label{app:budget_llama3_1b}

\begin{figure}[H]
\centering
\includegraphics[width=\textwidth]{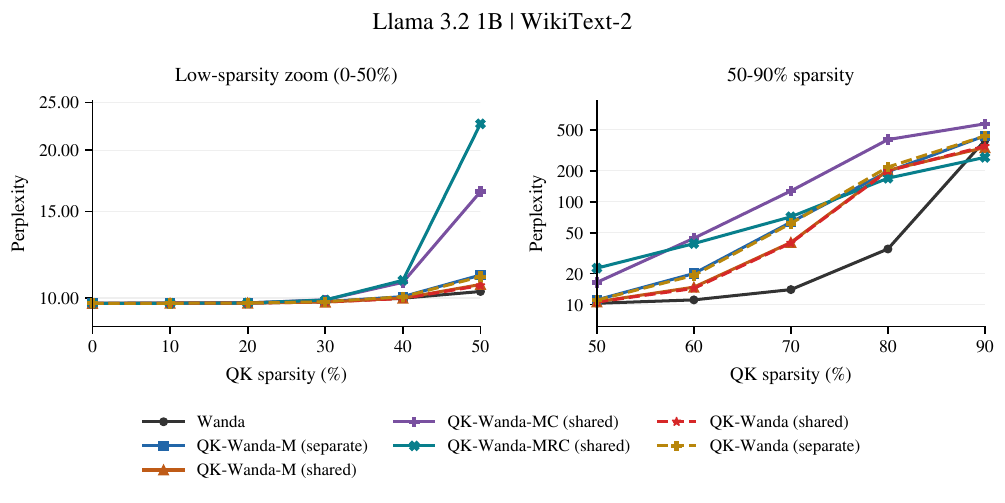}
\caption{Llama-3.2-1B: WikiText-2 PPL across QK sparsity.}
\label{fig:qk_budget_1b}
\end{figure}

\begin{figure}[H]
\centering
\includegraphics[width=\textwidth]{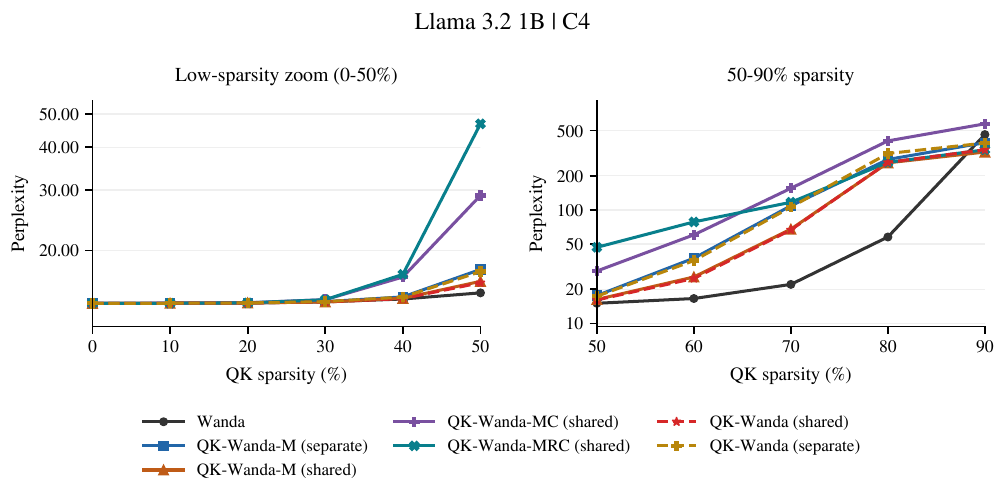}
\caption{Llama-3.2-1B: C4 PPL across QK sparsity.}
\label{fig:qk_budget_1b_c4}
\end{figure}

\FloatBarrier
\subsection{Llama 3.2 3B}
\label{app:budget_llama3_3b}

\begin{figure}[H]
\centering
\includegraphics[width=\textwidth]{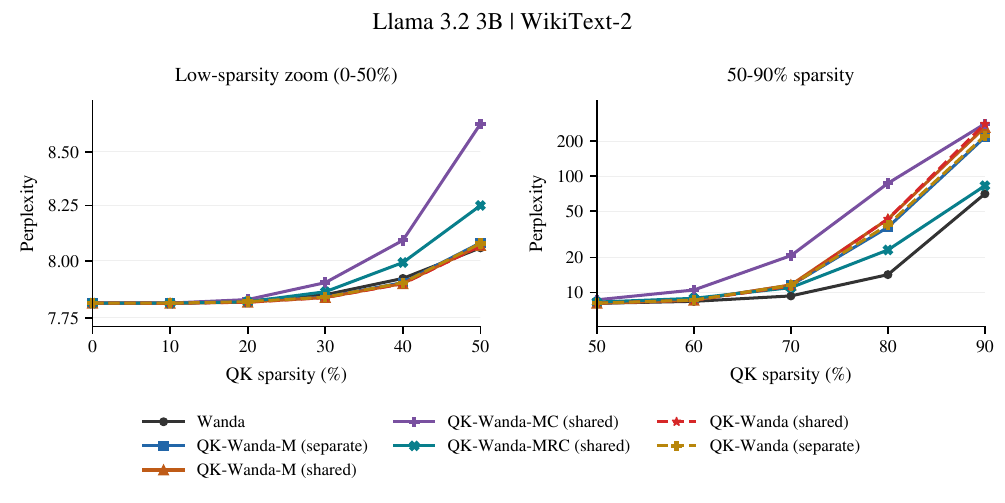}
\caption{Llama-3.2-3B: WikiText-2 PPL across QK sparsity.}
\label{fig:qk_budget_3b}
\end{figure}

\begin{figure}[H]
\centering
\includegraphics[width=\textwidth]{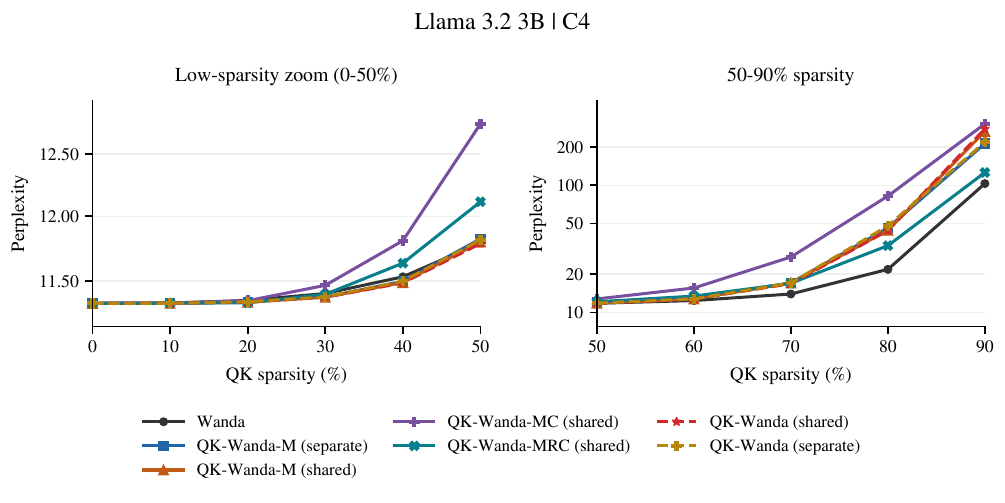}
\caption{Llama-3.2-3B: C4 PPL across QK sparsity.}
\label{fig:qk_budget_3b_c4}
\end{figure}

\FloatBarrier
\subsection{Llama 3.1 8B}
\label{app:budget_llama3_8b}

\begin{figure}[H]
\centering
\includegraphics[width=\textwidth]{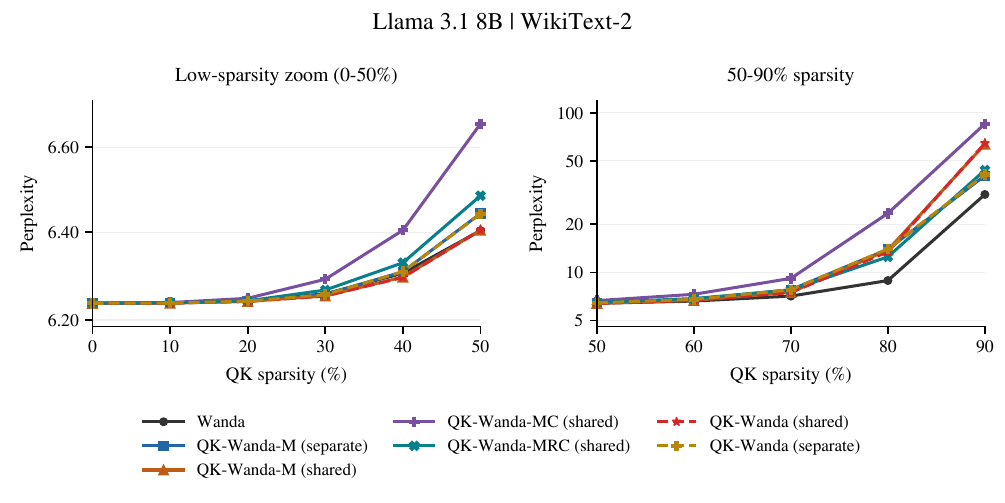}
\caption{Llama-3.1-8B: WikiText-2 PPL across QK sparsity.}
\label{fig:qk_budget}
\end{figure}

\begin{figure}[H]
\centering
\includegraphics[width=\textwidth]{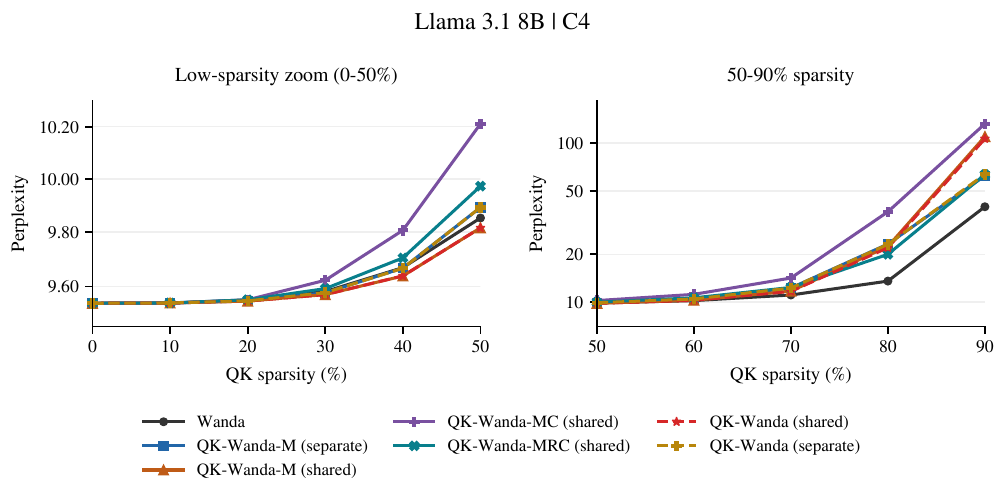}
\caption{Llama-3.1-8B: C4 PPL across QK sparsity.}
\label{fig:qk_budget_c4}
\end{figure}

\FloatBarrier
\subsection{Qwen2.5}
\label{app:budget_qwen25}

The Qwen2.5 figures use the same methods and axes described above. All Qwen2.5-72B configurations, including the dense reference, use BF16.

\IfFileExists{figures/qk_budget_qwen25_0p5b_wikitext2.pdf}{
\begin{figure}[H]
\centering
\includegraphics[width=\textwidth]{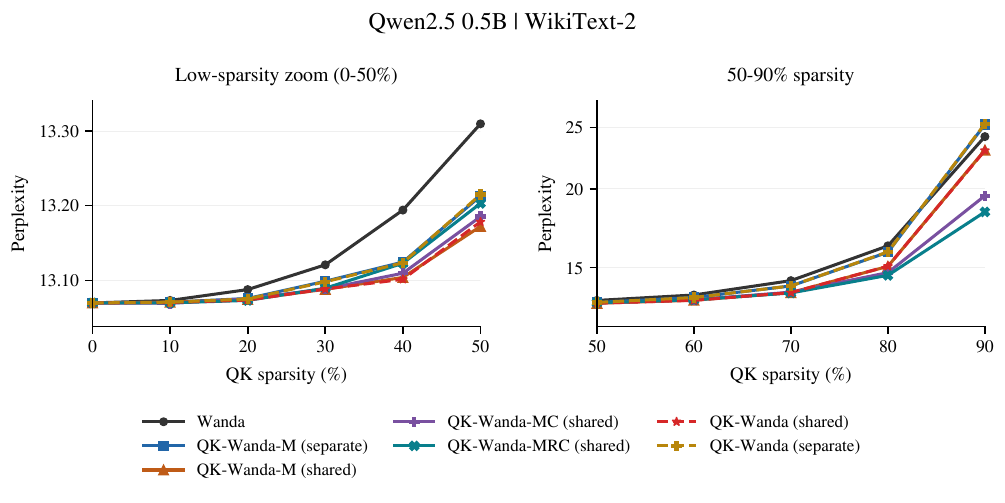}
\caption{Qwen2.5-0.5B: WikiText-2 PPL across QK sparsity.}
\label{fig:qk_budget_qwen25_0p5b_wikitext2}
\end{figure}

\begin{figure}[H]
\centering
\includegraphics[width=\textwidth]{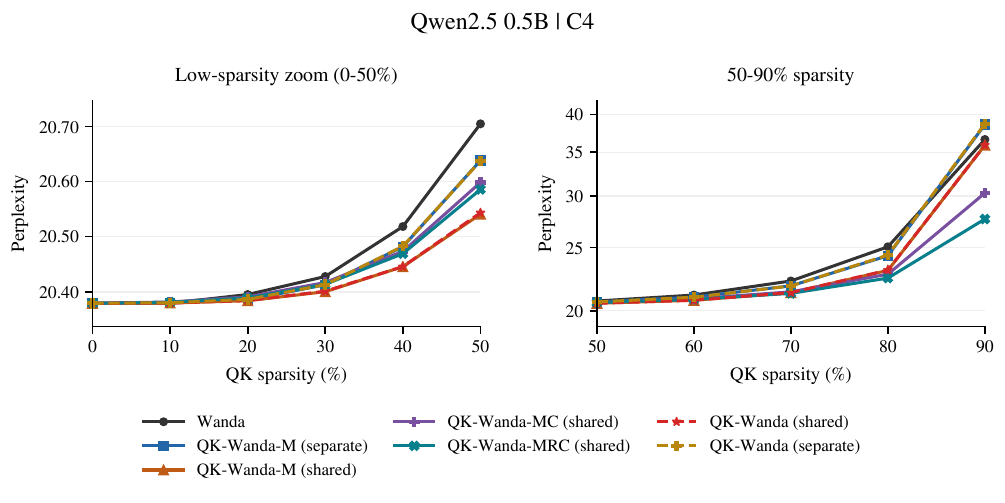}
\caption{Qwen2.5-0.5B: C4 PPL across QK sparsity.}
\label{fig:qk_budget_qwen25_0p5b_c4}
\end{figure}

}{\ivan{Import the complete Qwen2.5-0.5B WikiText-2 and C4 curves.}}

\IfFileExists{figures/qk_budget_qwen25_1p5b_wikitext2.pdf}{
\FloatBarrier
\begin{figure}[H]
\centering
\includegraphics[width=\textwidth]{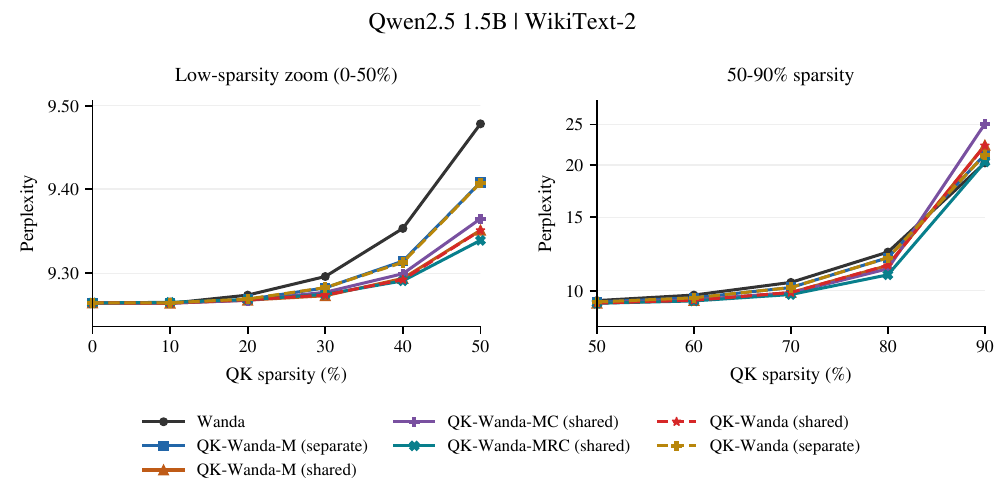}
\caption{Qwen2.5-1.5B: WikiText-2 PPL across QK sparsity.}
\label{fig:qk_budget_qwen25_1p5b_wikitext2}
\end{figure}

\begin{figure}[H]
\centering
\includegraphics[width=\textwidth]{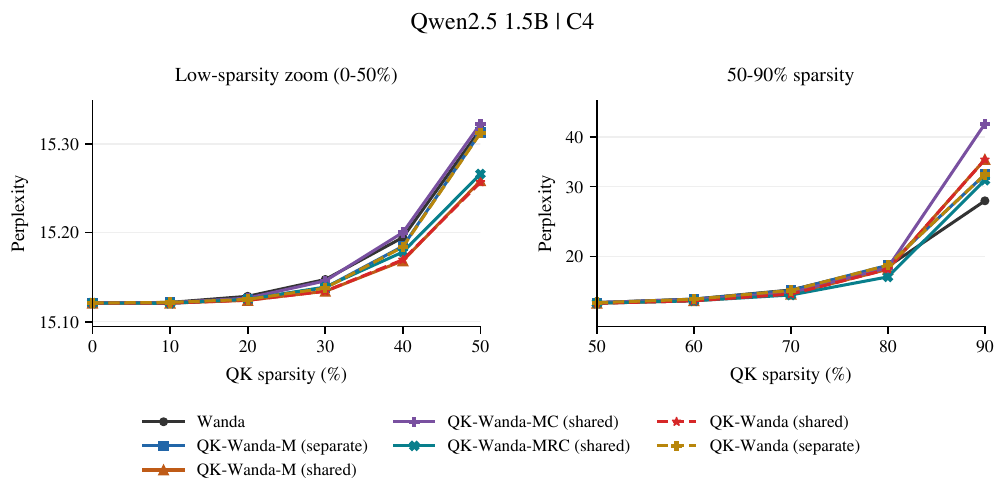}
\caption{Qwen2.5-1.5B: C4 PPL across QK sparsity.}
\label{fig:qk_budget_qwen25_1p5b_c4}
\end{figure}

}{\ivan{Import the complete Qwen2.5-1.5B WikiText-2 and C4 curves.}}

\IfFileExists{figures/qk_budget_qwen25_3b_wikitext2.pdf}{
\FloatBarrier
\begin{figure}[H]
\centering
\includegraphics[width=\textwidth]{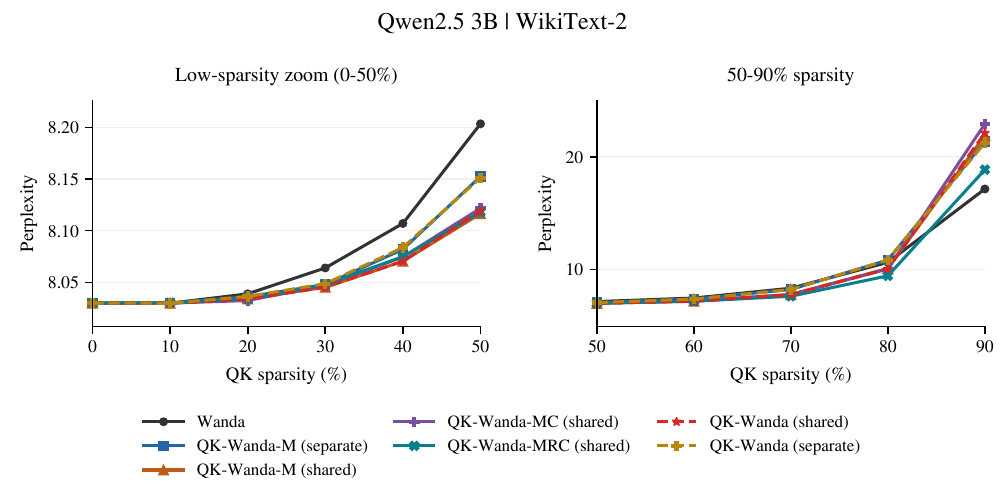}
\caption{Qwen2.5-3B: WikiText-2 PPL across QK sparsity.}
\label{fig:qk_budget_qwen25_3b_wikitext2}
\end{figure}

\begin{figure}[H]
\centering
\includegraphics[width=\textwidth]{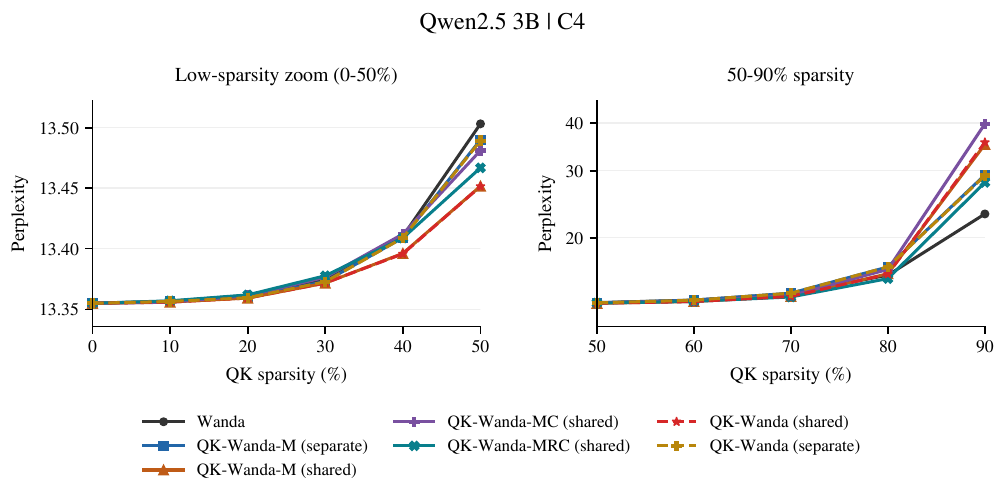}
\caption{Qwen2.5-3B: C4 PPL across QK sparsity.}
\label{fig:qk_budget_qwen25_3b_c4}
\end{figure}

}{\ivan{Import the complete Qwen2.5-3B WikiText-2 and C4 curves.}}

\IfFileExists{figures/qk_budget_qwen25_7b_wikitext2.pdf}{
\FloatBarrier
\begin{figure}[H]
\centering
\includegraphics[width=\textwidth]{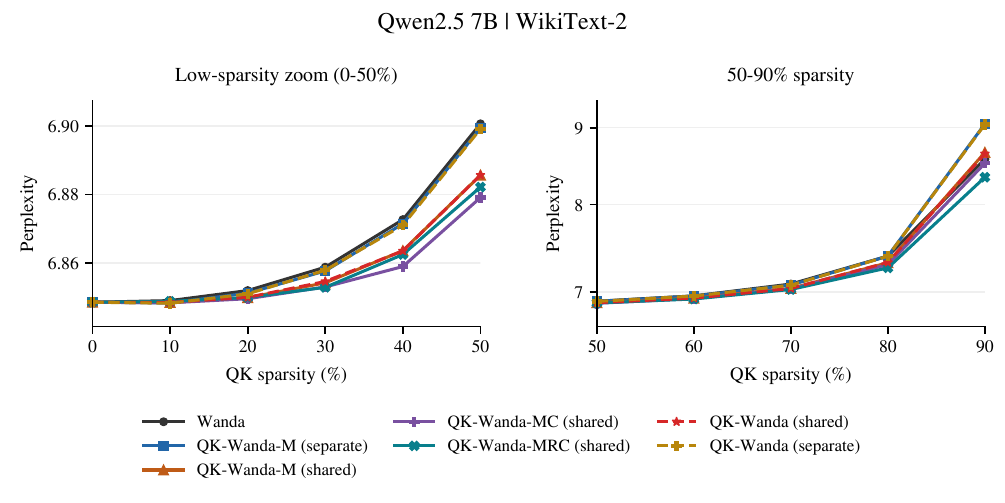}
\caption{Qwen2.5-7B: WikiText-2 PPL across QK sparsity.}
\label{fig:qk_budget_qwen25_7b_wikitext2}
\end{figure}

\begin{figure}[H]
\centering
\includegraphics[width=\textwidth]{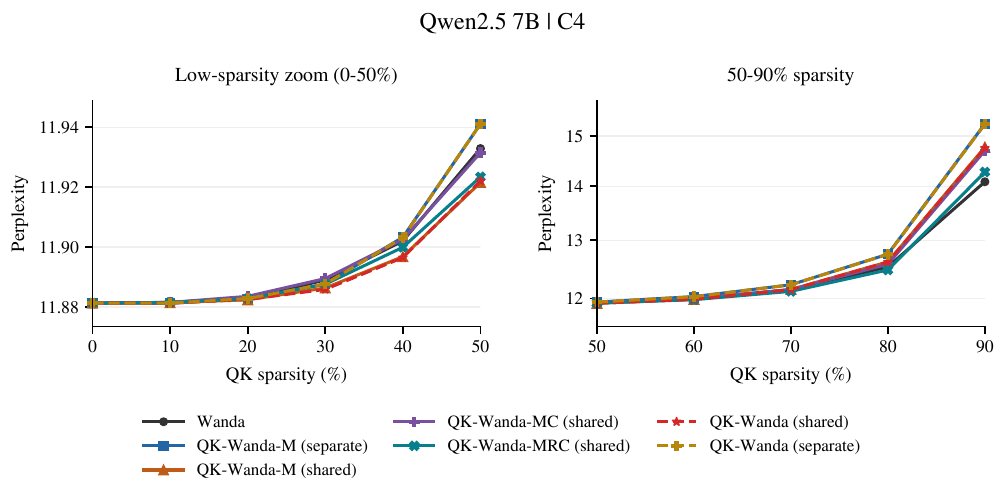}
\caption{Qwen2.5-7B: C4 PPL across QK sparsity.}
\label{fig:qk_budget_qwen25_7b_c4}
\end{figure}

}{\ivan{Import the complete Qwen2.5-7B WikiText-2 and C4 curves.}}

\IfFileExists{figures/qk_budget_qwen25_14b_wikitext2.pdf}{
\FloatBarrier
\begin{figure}[H]
\centering
\includegraphics[width=\textwidth]{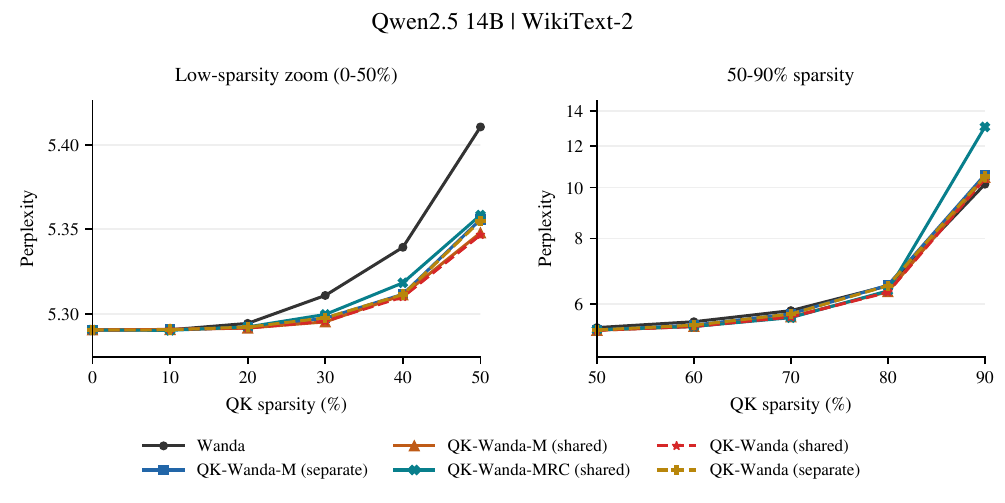}
\caption{Qwen2.5-14B: WikiText-2 PPL across QK sparsity.}
\label{fig:qk_budget_qwen25_14b_wikitext2}
\end{figure}

\begin{figure}[H]
\centering
\includegraphics[width=\textwidth]{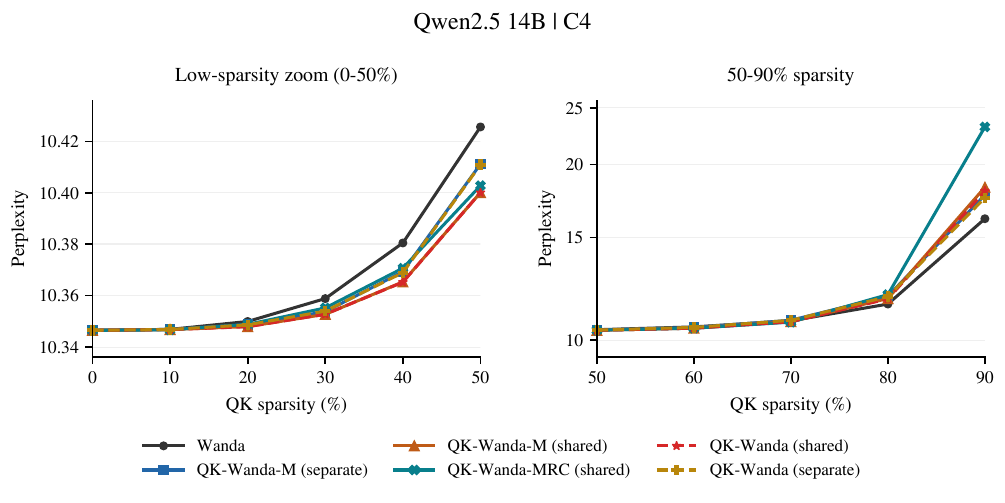}
\caption{Qwen2.5-14B: C4 PPL across QK sparsity.}
\label{fig:qk_budget_qwen25_14b_c4}
\end{figure}

}{\ivan{Import the complete Qwen2.5-14B WikiText-2 and C4 curves.}}

\IfFileExists{figures/qk_budget_qwen25_32b_wikitext2.pdf}{
\FloatBarrier
\begin{figure}[H]
\centering
\includegraphics[width=\textwidth]{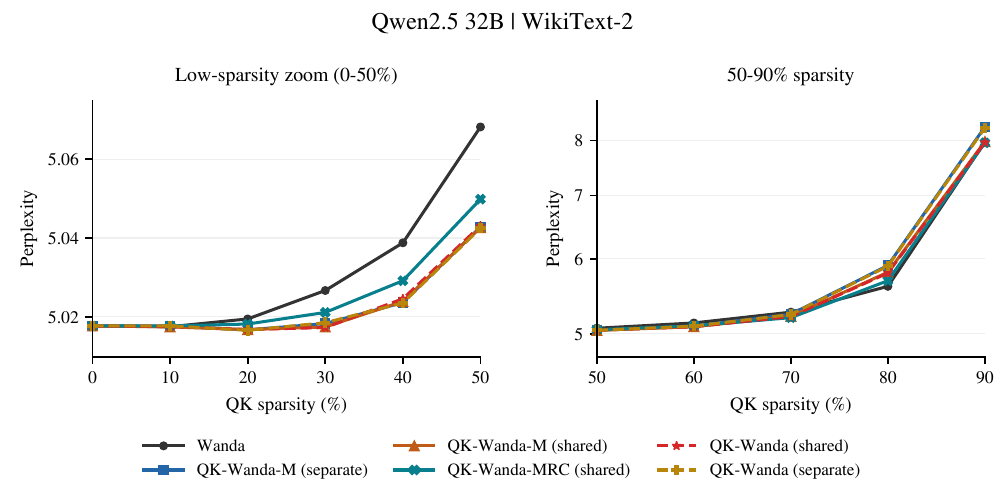}
\caption{Qwen2.5-32B: WikiText-2 PPL across QK sparsity.}
\label{fig:qk_budget_qwen25_32b_wikitext2}
\end{figure}

\begin{figure}[H]
\centering
\includegraphics[width=\textwidth]{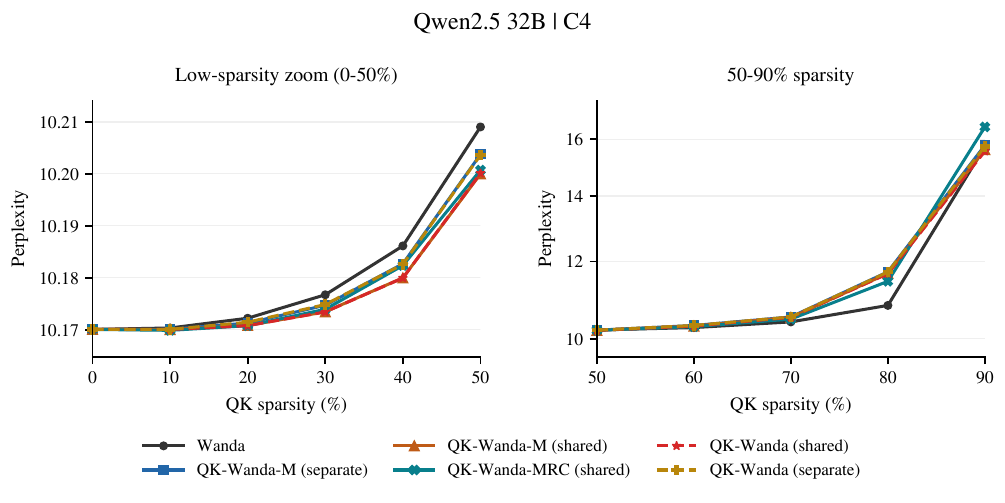}
\caption{Qwen2.5-32B: C4 PPL across QK sparsity.}
\label{fig:qk_budget_qwen25_32b_c4}
\end{figure}

}{\ivan{Import the complete Qwen2.5-32B WikiText-2 and C4 curves.}}

\FloatBarrier
\begin{figure}[H]
\centering
\includegraphics[width=\textwidth]{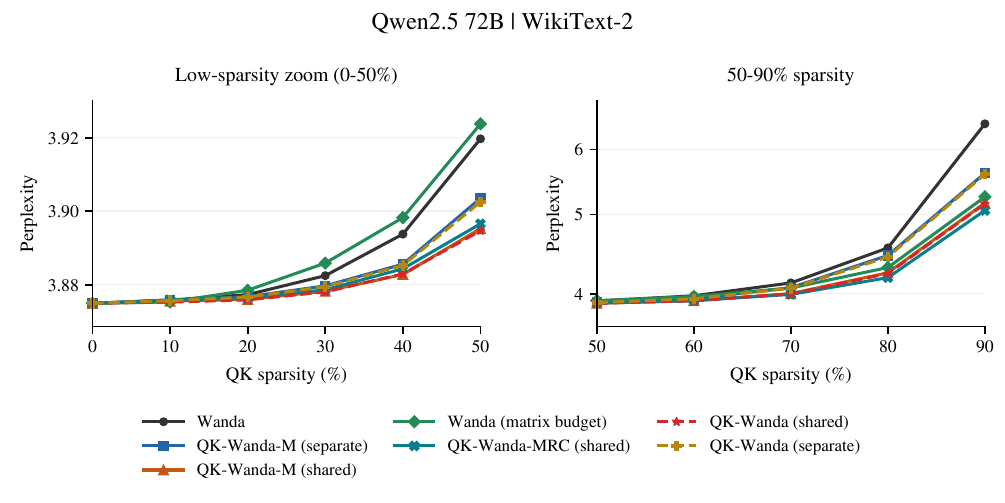}
\caption{Qwen2.5-72B: WikiText-2 PPL across QK sparsity.}
\label{fig:qk_budget_qwen25_72b_wikitext2}
\end{figure}

\begin{figure}[H]
\centering
\includegraphics[width=\textwidth]{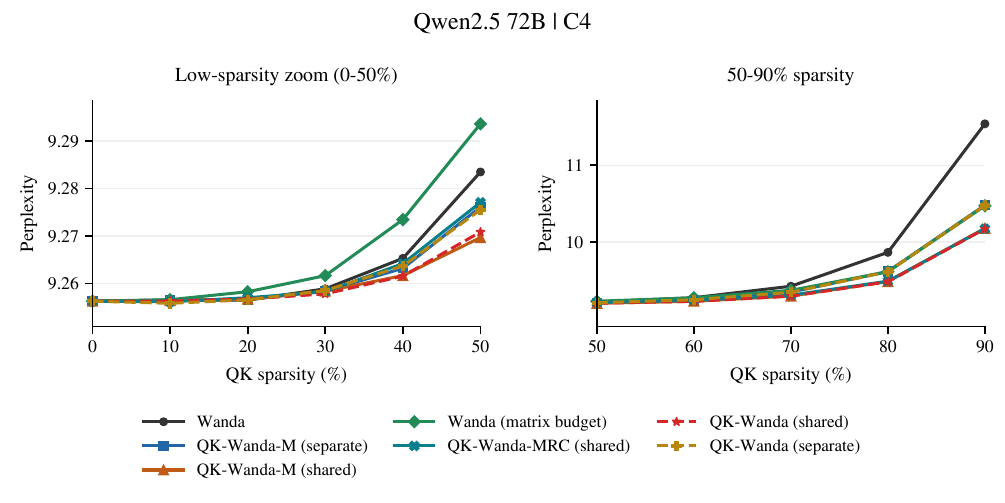}
\caption{Qwen2.5-72B: C4 PPL across QK sparsity.}
\label{fig:qk_budget_qwen25_72b_c4}
\end{figure}

\FloatBarrier
\subsection{Llama 2 70B}
\label{app:budget_llama2_70b}

\begin{figure}[H]
\centering
\includegraphics[width=\textwidth]{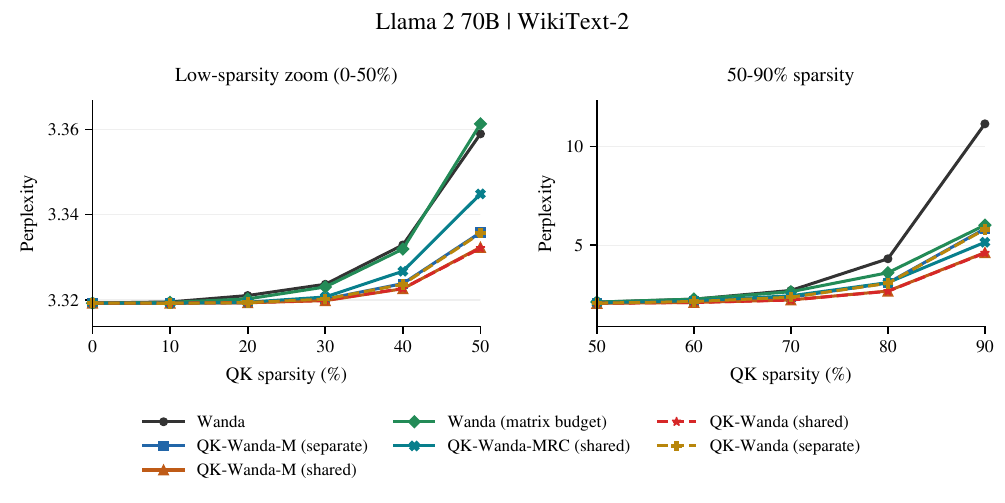}
\caption{Llama 2 70B: WikiText-2 PPL across QK sparsity.}
\label{fig:qk_budget_llama2_70b_wikitext2}
\end{figure}

\begin{figure}[H]
\centering
\includegraphics[width=\textwidth]{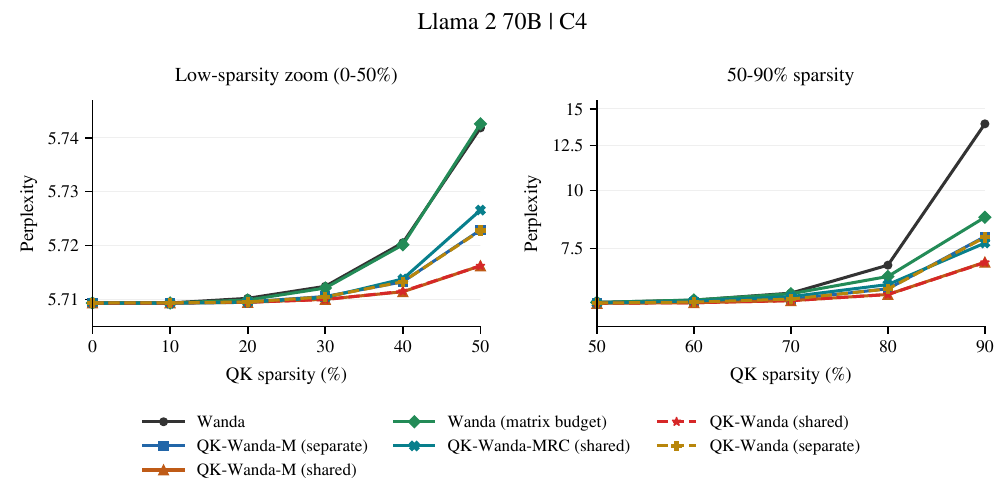}
\caption{Llama 2 70B: C4 PPL across QK sparsity.}
\label{fig:qk_budget_llama2_70b_c4}
\end{figure}

\FloatBarrier
\subsection{Llama 3.1 70B}
\label{app:budget_llama31_70b}

\begin{figure}[H]
\centering
\includegraphics[width=\textwidth]{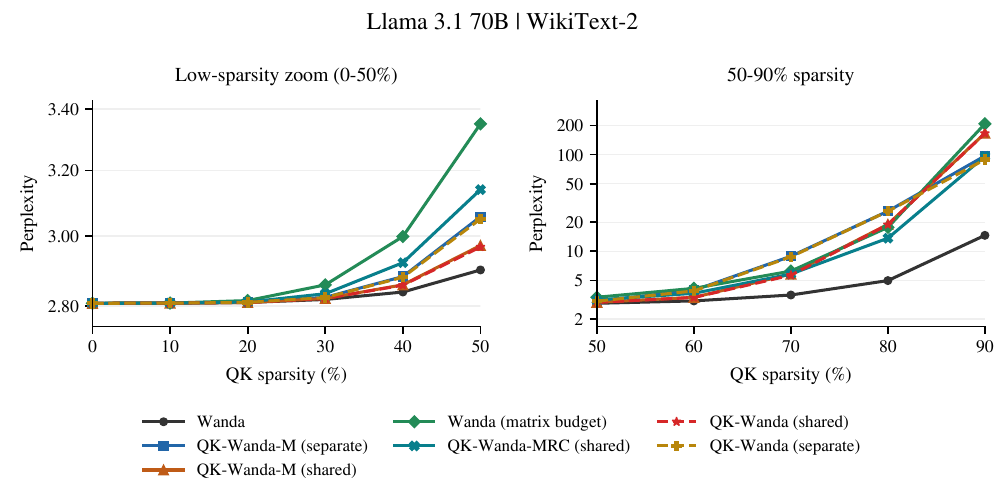}
\caption{Llama 3.1 70B: WikiText-2 PPL across QK sparsity.}
\label{fig:qk_budget_llama31_70b_wikitext2}
\end{figure}

\begin{figure}[H]
\centering
\includegraphics[width=\textwidth]{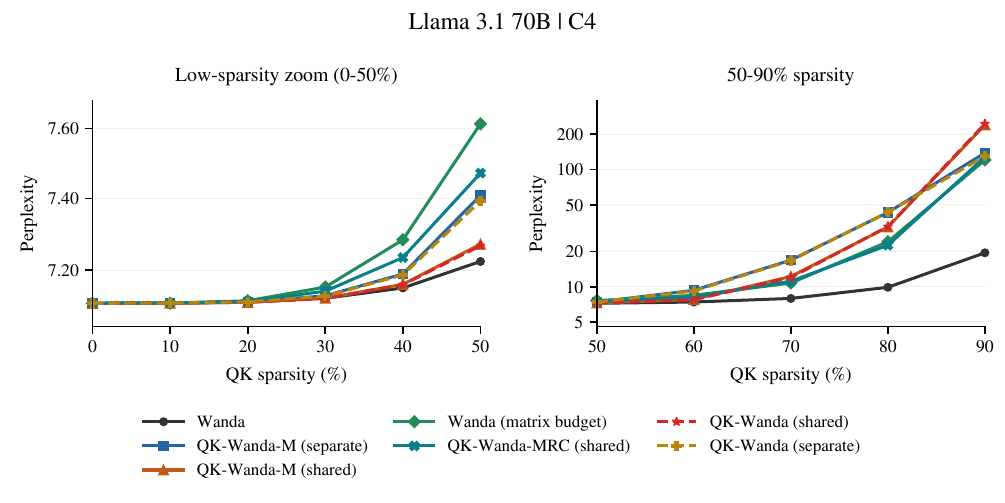}
\caption{Llama 3.1 70B: C4 PPL across QK sparsity.}
\label{fig:qk_budget_llama31_70b_c4}
\end{figure}

\FloatBarrier
\section{QK Reconstruction Across Sparsity Levels}

\label{app:qk_error_curves}

We measure QK reconstruction error (\cref{eq:relative_qk_error}) for all 15 checkpoints at 10\%, 20\%, \ldots, 90\% sparsity. Every method uses its PPL pruning mask and the same 16 held-out C4 sequences of length 512, with common dense block inputs. The diagnostic is the same for all scoring variants; causal and centered losses are evaluated separately in the labeled ablations.

Shared-budget QK-Wanda has lower error than Wanda in 133/135 comparisons, and lower error than separate-budget QK-Wanda in 122/135. Separate-budget QK-Wanda improves over Wanda in 124/135. The two shared-budget exceptions are Llama-3.2-3B at 80\% and 90\%. The corresponding counts for QK-Wanda-M are also 133, 122, and 124; the measured error values are not identical. Lower reconstruction error does not guarantee lower PPL or higher task accuracy.

Each panel uses a logarithmic vertical axis. Markers show measurements and lines join adjacent sparsities. QK-Wanda-MC and QK-Wanda-MRC use the same QK diagnostic as the other methods, rather than their centered scoring losses. Matrix-budget Wanda is included on the three largest checkpoints.

\subsection{Llama 2 and TinyLlama}

\begin{figure}[H]
\centering
\includegraphics[width=\textwidth]{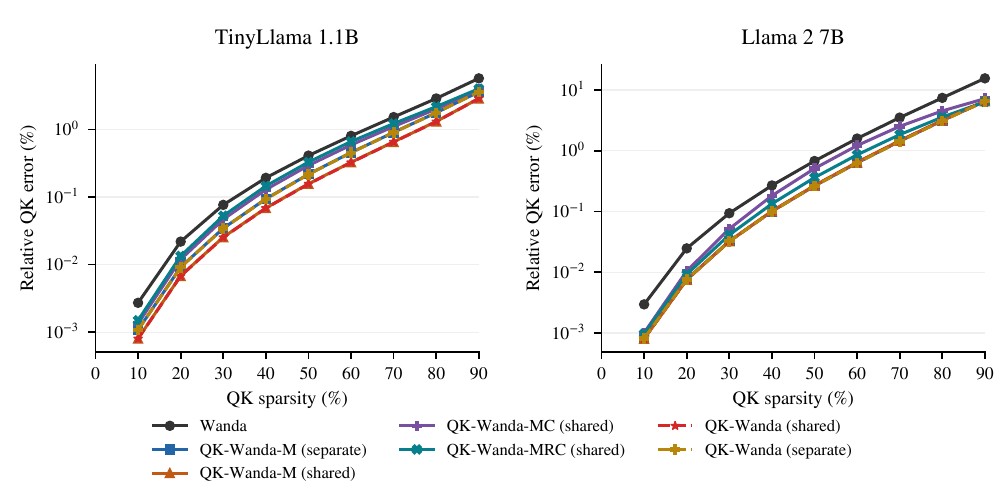}
\caption{TinyLlama 1.1B (left) and Llama 2 7B (right): QK reconstruction error under the common protocol in \cref{eq:relative_qk_error}.}
\label{fig:qk_error_tinyllama_1b}
\label{fig:qk_error_llama2_7b}
\end{figure}

\begin{figure}[H]
\centering
\includegraphics[width=\textwidth]{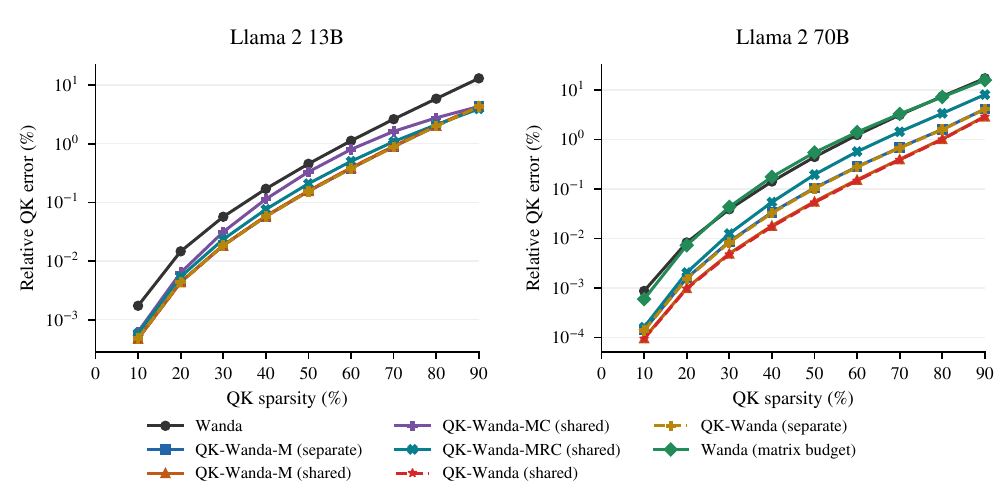}
\caption{Llama 2 13B (left) and Llama 2 70B (right): QK reconstruction error under the common protocol in \cref{eq:relative_qk_error}.}
\label{fig:qk_error_llama2_13b}
\label{fig:qk_error_llama2_70b}
\end{figure}

\FloatBarrier
\subsection{Llama 3}

\begin{figure}[H]
\centering
\includegraphics[width=\textwidth]{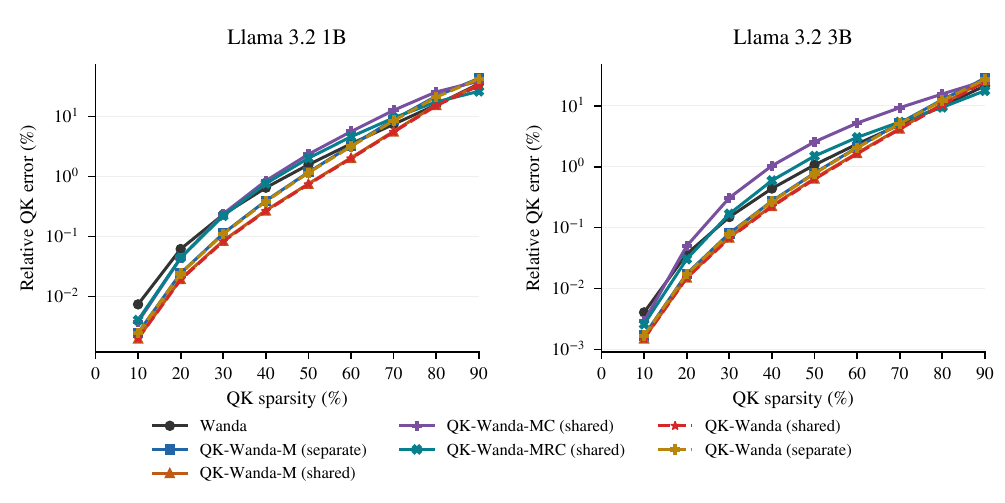}
\caption{Llama 3.2 1B (left) and Llama 3.2 3B (right): QK reconstruction error under the common protocol in \cref{eq:relative_qk_error}.}
\label{fig:qk_error_llama3_1b}
\label{fig:qk_error_llama3_3b}
\end{figure}

\begin{figure}[H]
\centering
\includegraphics[width=\textwidth]{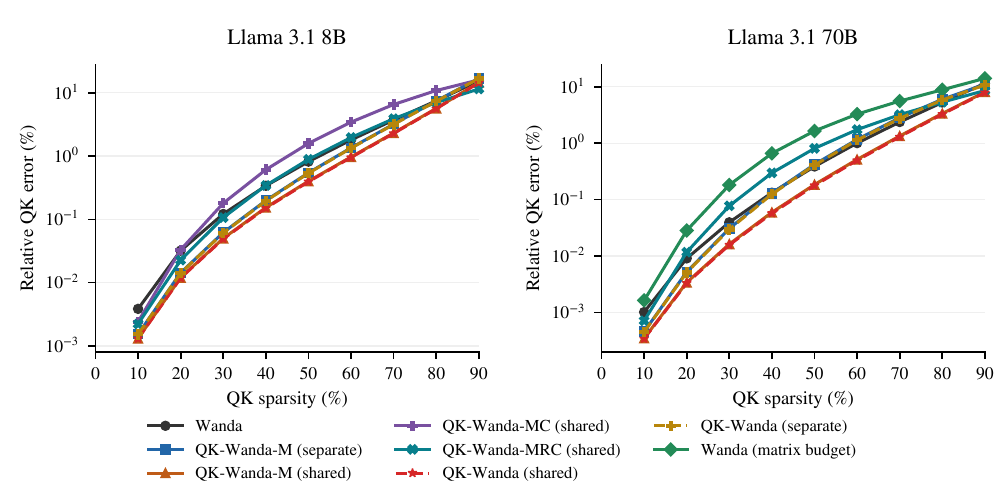}
\caption{Llama 3.1 8B (left) and Llama 3.1 70B (right): QK reconstruction error under the common protocol in \cref{eq:relative_qk_error}.}
\label{fig:qk_error_llama3_8b}
\label{fig:qk_error_llama3_70b}
\end{figure}

\FloatBarrier
\subsection{Qwen2.5}

\begin{figure}[H]
\centering
\includegraphics[width=\textwidth]{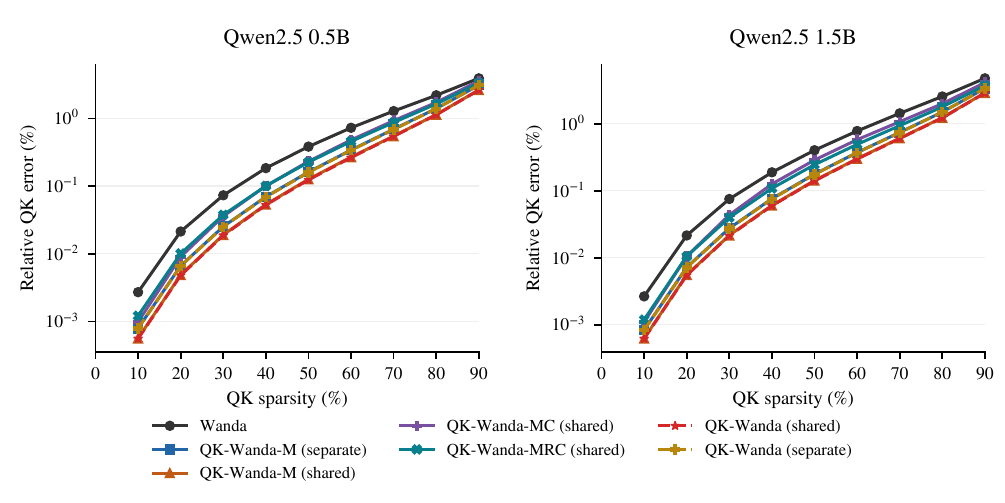}
\caption{Qwen2.5 0.5B (left) and Qwen2.5 1.5B (right): QK reconstruction error under the common protocol in \cref{eq:relative_qk_error}.}
\label{fig:qk_error_qwen25_0p5b}
\label{fig:qk_error_qwen25_1p5b}
\end{figure}

\begin{figure}[H]
\centering
\includegraphics[width=\textwidth]{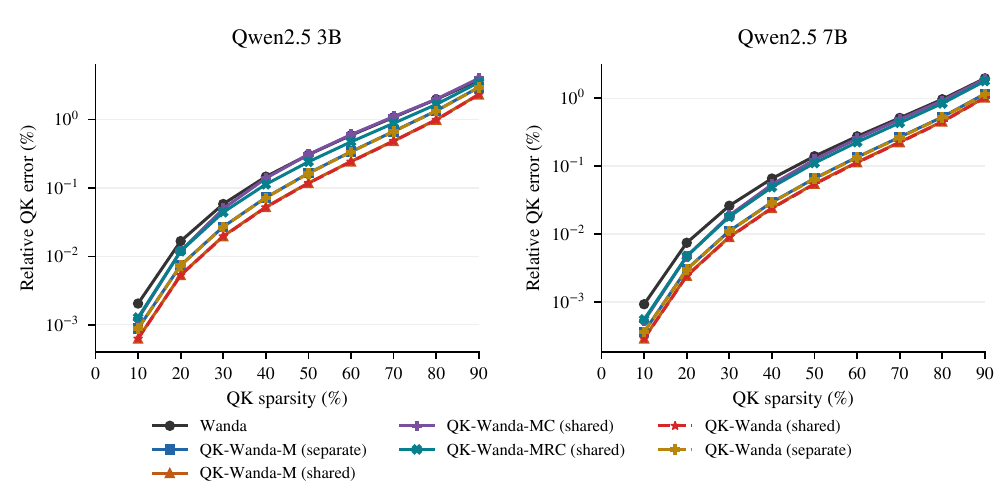}
\caption{Qwen2.5 3B (left) and Qwen2.5 7B (right): QK reconstruction error under the common protocol in \cref{eq:relative_qk_error}.}
\label{fig:qk_error_qwen25_3b}
\label{fig:qk_error_qwen25_7b}
\end{figure}

\FloatBarrier

\begin{figure}[H]
\centering
\includegraphics[width=\textwidth]{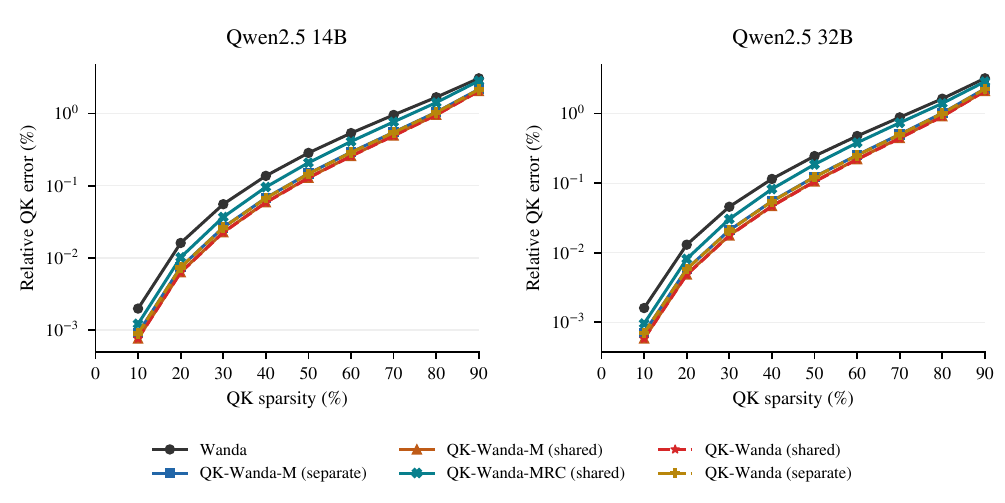}
\caption{Qwen2.5 14B (left) and Qwen2.5 32B (right): QK reconstruction error under the common protocol in \cref{eq:relative_qk_error}.}
\label{fig:qk_error_qwen25_14b}
\label{fig:qk_error_qwen25_32b}
\end{figure}

\begin{figure}[H]
\centering
\includegraphics[width=\textwidth]{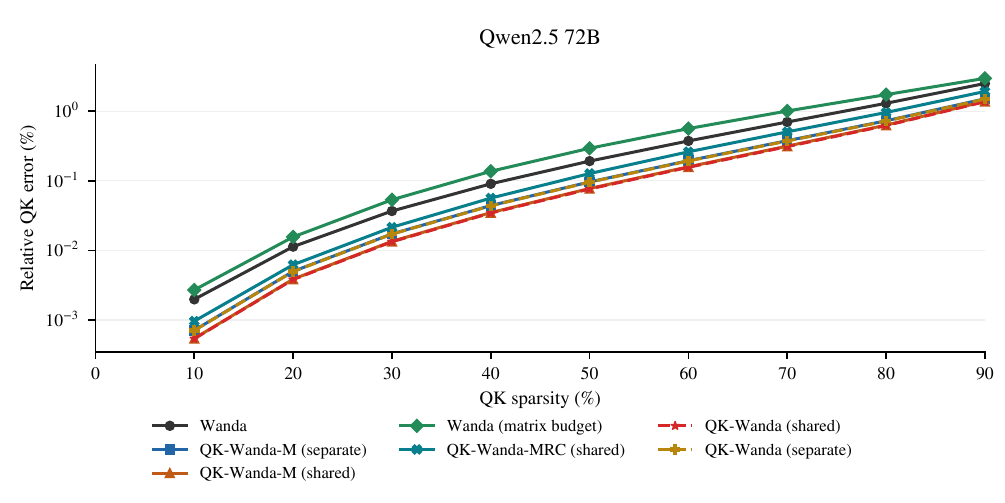}
\caption{Qwen2.5 72B: QK reconstruction error under the common protocol in \cref{eq:relative_qk_error}.}
\label{fig:qk_error_qwen25_72b}
\end{figure}

\FloatBarrier
\section{Progressive QK Re-Scoring}
\label{app:iterative}

We examine whether recomputing QK scores during progressive pruning improves on one-shot mask selection. We specify the tested schedules, explain how their local objective differs from a fixed dense-reference loss, and report the resulting perplexities.

\subsection{Motivation and the executed schedule}

The \algname{QK-Wanda-M} score of a query weight depends on keys, and a key score depends on queries. Partial pruning therefore changes the energies used by subsequent decisions. We tested whether refreshing those energies while gradually growing the masks improves the final model.

Let \(J\) be the number of steps and \(M=\lfloor0.5n\rfloor\) the final removal count in a group of \(n\) original weights. A group is one output row in the row-wise setting, or an entire projection matrix in the matrix-global setting. At step \(r=1,\ldots,J\), both \algname{QK-Wanda-M} score arrays are computed from the current sparse pair \((W^{Q,r-1},W^{K,r-1})\), where the second superscript entry counts completed pruning steps. With initial removal count \(M_0=0\), the cumulative target is
\begin{equation}
 M_r=\left\lfloor\frac{rM}{J}\right\rfloor.
 \label{eq:iterative_budget}
\end{equation}
We add the \(M_r-M_{r-1}\) lowest-scoring remaining weights, separately for Q and K. Previously deleted entries stay deleted. Both next masks are constructed before either is applied. Inputs to the current block remain fixed during these steps; its final sparse output feeds the next block.

This is simultaneous progressive mask growth, rather than an alternating Q-then-K update. For \(J=5\), cumulative sparsity is approximately 10\%, 20\%, 30\%, 40\%, 50\%; \(J=10\) uses approximately 5\% increments. \(J=1\) is one-shot \algname{QK-Wanda-M}. No retained weight is updated by compensation.

\subsection{Why re-scoring is not a greedy teacher-loss optimizer}

The scalar proof applies around the \emph{current} pair. It measures the squared local change caused by the next deletion, rather than its marginal effect on reconstruction of the original reference. Using the unscaled products \(B=Q^\top K\) from \cref{app:proof_c}, for each sequence and head let \(E=B_{\mathrm{current}}-B_{\mathrm{ref}}\) be the accumulated error, and let \(\Delta B=B_{\mathrm{next}}-B_{\mathrm{current}}\) be the next change. The new error is \(E+\Delta B\); we suppress sequence/head labels in these three expressions. For two collections of head matrices \(A=\{A^{(h,s)}\}\) and \(B=\{B^{(h,s)}\}\), define the masked sum of products
\begin{equation}
 \langle A,B\rangle_C=\frac1N\sum_{s=1}^{N}\sum_{h=1}^{H_Q}
 \sum_{t=1}^{T_s}\sum_{u=1}^{t} A^{(h,s)}_{tu}B^{(h,s)}_{tu}.
\end{equation}
The brackets \(\langle\cdot,\cdot\rangle_C\) are shorthand for this expression, and \(\norm{A}_C^2=\langle A,A\rangle_C\) is its associated squared masked error. On unrestricted matrices, \(\norm{\cdot}_C\) is a seminorm because masked-out entries do not contribute; it is a norm on collections restricted to the retained causal entries. Expanding the square gives the exact change in the causal loss \(\mathcal L_C\) against the fixed reference (the ``teacher''):
\begin{equation}
 \norm{E+\Delta B}_C^2-\norm{E}_C^2
 =2\langle E,\Delta B\rangle_C+\norm{\Delta B}_C^2.
 \label{eq:iterative_residual}
\end{equation}
The refreshed scalar score accounts for \(\norm{\Delta B}_C^2\) for one deletion; it omits the residual cross term. Simultaneous batches also contain the interactions in \cref{eq:joint_change}. Thus these experiments test refreshed scalar ranking, not exact greedy minimization, residual-aware selection, or compensated pruning.

\FloatBarrier
\subsection{Results and source provenance}

The ablation used Llama-3.2-1B, QK-only unstructured pruning to 50\%, and calibration seed 0. C4 calibration uses \(128\times2048\) or \(256\times32\) tokens; PPL uses length 2048. Times cover pruning on one A100 40GB, excluding evaluation.

\begin{table}[H]
\centering\small
\caption{Row-wise progressive pruning on Llama-3.2-1B. The final budget is 50\% in every output row. Each row is one seed-0 run; bold selects the lowest unrounded PPL for each calibration setting.}
\label{tab:iterative_row}
\begin{tabular}{llrrrr}
\toprule
Calib. \(N\times T\)&Score&Steps&WikiText-2 PPL&C4 PPL&Time (s)\\
\midrule
128\(\times\)2048&\algname{Wanda}&--&10.6474&15.7481&22.2\\
128\(\times\)2048&\algname{QK-Wanda-M}&1&10.5334&\best{15.5175}&26.2\\
128\(\times\)2048&\algname{QK-Wanda-M}&10&\best{10.5332}&15.5178&94.9\\
\addlinespace
256\(\times\)32&\algname{Wanda}&--&\best{10.3054}&15.0255&15.7\\
256\(\times\)32&\algname{QK-Wanda-M}&1&10.3072&\best{15.0187}&17.2\\
256\(\times\)32&\algname{QK-Wanda-M}&2&10.3077&15.0193&21.9\\
256\(\times\)32&\algname{QK-Wanda-M}&5&10.3066&15.0193&39.7\\
256\(\times\)32&\algname{QK-Wanda-M}&10&10.3067&15.0195&65.7\\
 
\bottomrule
\end{tabular}
\end{table}

\begin{table}[H]
\centering\small
\caption{Progressive QK-Wanda-M pruning with separate matrix budgets. Each Q and K matrix has its own 50\% budget. All other settings match \cref{tab:iterative_row}; this ablation used a separate single-seed calibration protocol. Bold marks the lowest unrounded PPL for each calibration setting.}
\label{tab:iterative_matrix}
\begin{tabular}{llrrrr}
\toprule
Calib. \(N\times T\)&Score&Steps&WikiText-2 PPL&C4 PPL&Time (s)\\
\midrule
128\(\times\)2048&\algname{QK-Wanda-M}&1&\best{11.4888}&\best{18.2040}&25.9\\
128\(\times\)2048&\algname{QK-Wanda-M}&10&12.5458&20.8593&94.9\\
\addlinespace
256\(\times\)32&\algname{QK-Wanda-M}&1&\best{11.1735}&\best{17.7755}&19.8\\
256\(\times\)32&\algname{QK-Wanda-M}&2&11.2330&17.9338&22.4\\
256\(\times\)32&\algname{QK-Wanda-M}&5&11.4152&18.5233&37.0\\
256\(\times\)32&\algname{QK-Wanda-M}&10&11.6609&19.1525&66.3\\
\bottomrule
\end{tabular}
\end{table}

The one-shot rows also compare scoring rules at fixed allocation granularity. With \(128\times2048\) calibration, causal QK-Wanda-M improves both PPL values over Wanda; with \(256\times32\), it slightly worsens WikiText-2 and improves C4. Causal weighting can change rankings within a row (Appendix~\ref{app:row_equivalence}), unlike the single-sequence unmasked row rescaling. These controls show a calibration-dependent scoring effect at one checkpoint and sparsity; they do not isolate how much of the main cross-model gains comes from scoring versus allocation.

Under row budgets, \algname{QK-Wanda-M} schedules differ by less than 0.0012 PPL within each calibration setting, while ten-step pruning takes 3.6--3.8 times as long as one-shot \algname{QK-Wanda-M}. These single-seed results provide no evidence of a useful quality gain.

Under matrix-global budgets, additional steps worsen both perplexities. At length 32, one to ten steps raises WikiText-2/C4 PPL by \(0.4873/1.3771\); at length 2048, by \(1.0570/2.6553\). The data favor one-shot pruning in this ablation. They do not identify whether irreversible mask allocation, omitted residual terms, or other interactions cause the degradation, and do not rule out other adaptive algorithms.

Source paths, SHA256 hashes, and all 16 rows are retained in the paper's
{\small\nolinkurl{results/iterative_archive.json}}.
The exporter checks settings, step coverage, and final sparsity. The single-seed ablation was evaluated separately from the three-seed comparisons.

\FloatBarrier
\section{Relation to Published QK Reconstruction Objectives}
\label{app:prior_geometry}

We relate scalar deletion scores to the QK reconstruction quadratics used in prior quantization work. We then distinguish deletion of an individual weight from deletion of a channel, which changes several coordinates jointly.

\subsection{From a reconstruction quadratic to scalar deletion}

Consider one query head and its associated key head, with the head labels suppressed. The matrix \(\Delta W^{Q}\) is a change to that query projection, \(\delta=\operatorname{vec}(\Delta W^{Q})\) stacks its columns, and \(G_U\) is the unmasked quadratic coefficient from \cref{eq:u_curvature}. With keys fixed, the loss is
\begin{equation}
 \begin{aligned}
 \mathcal L(\Delta W^{Q})
 &=\frac1N\sum_{s=1}^{N}\norm{(\Delta W^{Q}X^{(s)})^\top K^{(s)}}_F^2\\
 &=\frac1N\sum_{s=1}^{N}\norm{K^{(s)\top}\Delta W^{Q}X^{(s)}}_F^2
 =\delta^\top G_U\delta.
 \end{aligned}
 \label{eq:prior_reconstruction}
\end{equation}
Up to a common positive normalization factor, \algname{AESPA} gives the corresponding QK reconstruction targets and vectorized coefficient in Equations (15)--(16) and (19) \citep{kim2024aespa}. Its key orientation differs from ours; transposing the product preserves its Frobenius norm. Evaluating this published type of quadratic at a scalar deletion, with vectorized index \(p\) as in \cref{app:causal_quadratic}, gives our algebraic specialization:
\begin{equation}
 \delta=-W_p e_p
 \quad\Longrightarrow\quad
 \mathcal L=W_p^2[G_U]_{pp}.
\end{equation}
This evaluates the QK reconstruction quadratic for deletion of an original scalar weight. \algname{AESPA} also introduces a product-of-expectations approximation in its Equation (20); our sequence-local accumulators retain the sum of products in \cref{eq:u_curvature}.

\algname{BoA}'s Equations (11)--(12) use related QK Hessians, and (14)--(15) incorporate RoPE through averaged rotated factors \citep{kim2025boa}. That establishes a prior rotary-aware quantization construction. Our \algname{QK-Wanda-MR} proof evaluates the exact finite change of an original scalar coordinate and retains its token-specific rotation, input square, and causal pair restriction. \algname{QK-Wanda-M} uses the corresponding causal restriction without rotations. The resulting statistics retain the causal, token-specific interactions of the scalar perturbation.

\subsection{A channel deletion is a different perturbation}

For one sequence and one unmasked head, suppress sequence/head labels. Using the same unscaled convention, decompose the product as
\begin{equation}
 B=Q^\top K=\sum_i Q_{i:}^\top K_{i:}.
\end{equation}
Deleting a complete matched QK dimension \(i\) removes the summand \(Q_{i:}^\top K_{i:}\). Its uncompensated energy is
\(\norm{Q_{i:}}_2^2\norm{K_{i:}}_2^2\).
\algname{CORP} uses an expected channel-contribution energy to rank such dimensions, then reconstructs removed contributions through closed-form compensation \citep[Sections 3.2--3.4]{zhang2026corp}. In our weight layout, this changes whole matched output rows of Q and K. Deleting just \(W^{Q}_{ij}\) instead produces the perturbation \(-W^{Q}_{ij}X_{j:}^\top K_{i:}\), with a different score for each input coordinate \(j\). The difference is the pruning unit, not the use of the word ``row'' or ``column'' under transposed conventions.

The contribution is the compensation-free elementwise pruning specialization, its exact GQA-aware scalar statistics, and its controlled evaluation, including causal and rotary-aware extensions. The cited prior methods do not instantiate this full procedure. The quadratic connection above identifies which underlying ideas are shared and which compression decision and accumulators are studied here.

\FloatBarrier
\section{Cost of QK Pruning}
\label{app:score_cost}

We measure the full QK pruning computation on NVIDIA A100 PCIe 40GB and H200 GPUs across four Llama~3 model sizes. Tables~\ref{tab:method_cost_masked_history} and~\ref{tab:method_cost_h200_appendix} report medians of five runs for row-wise \algname{Wanda} and shared-budget \algname{QK-Wanda}. The unweighted means of the four checkpoint-specific time ratios are 1.01 on A100 and 1.03 on H200. Individual ratios range from 1.00 to 1.02 on A100 and from 1.01 to 1.08 on H200.

\paragraph{Measurement scope.}
Each timed block includes calibration forward passes, fresh statistics accumulation, score construction, stable GPU sorting and mask selection, weight zeroing, and forward propagation through the pruned block. Thus, both selection and application of the mask are included. Timers exclude checkpoint loading, tokenization and initial calibration-input preparation, initial model/input placement, benchmark-only restoration of the original weights, and evaluation. These measurements cover pruning computation, rather than end-to-end wall time including loading and transfers. The complexity column covers QK projections and scoring only; sorting the shared pool of \(P=d(D_Q+D_K)\) scores adds \(O(P\log P)\) comparison-sort complexity and is included in the measured times. Wanda sorts each output row separately.

\paragraph{Controlled protocol.}
Both GPUs use the same pinned Llama-3.2-1B/3B and Llama-3.1-8B/70B checkpoints in FP16, \(8192\times64=524{,}288\) C4 calibration tokens with seed~0, sequence batch size 32, FP32 statistics, and disabled TF32. These are the saved token manifests, sequence length and batching used in the main pruning experiments. Mask selection uses stable sorting on GPU. We prune only query and key weights at 50\% sparsity: Wanda uses row quotas, and QK-Wanda shares the budget across the two projections. Each method follows its own sequentially pruned calibration trajectory; the last measured output is propagated to its next block. They start from the same dense checkpoint and initial calibration inputs; later blocks receive the outputs of that method's pruned preceding blocks.

The checkpoint remains on CPU with one block moved to GPU. Initial calibration-input placement and block-weight transfers take place outside the measured intervals. Each actual block has two warm-ups and five timed repetitions per method, with method order alternating by block and CUDA synchronization around every measurement. Original query and key weights are restored outside the timer before each repetition. We sum the actual block times for each repetition before taking the median; no times are extrapolated from one layer. The average ratio is the unweighted mean of the four ratios of method-specific medians.

\paragraph{Reproducibility.}
An independent audit reaggregates all five sums over the complete block sequence, checks frozen source and checkpoint revisions, and verifies the saved calibration-token segments. Table~\ref{tab:score_cost} reports repetition ranges; five repeats quantify timing variability. Full repetition sums, calibration receipts, protocol and source hashes are retained in \nolinkurl{results/full_pruning_cost.json}; detailed per-block timing records are preserved in \nolinkurl{output/full_pruning_cost_gpu_20261001/}.

\begin{table}[H]
\centering\small
\setlength{\tabcolsep}{2pt}
\caption{Full QK pruning computation on one A100 PCIe 40GB, at 50\% QK sparsity with the main \(8192\times64=524{,}288\)-token C4 calibration, batch 32. Times include calibration forwards, statistics, scoring, GPU stable sorting and mask selection, weight zeroing and pruned propagation. Loading, setup transfers, initial preparation, benchmark resets are excluded. Medians of five sums over all blocks; ratios are relative to Wanda. Complexity covers QK projections and scoring per block, excluding sorting.}
\label{tab:method_cost_masked_history}
\renewcommand{\arraystretch}{1.15}
\begin{tabular*}{\linewidth}{@{\extracolsep{\fill}}lcccccc@{}}
\toprule
Method&Complexity&Llama-3.2~1B&Llama-3.2~3B&Llama-3.1~8B&Llama-3.1~70B&Avg.\ ratio\\
\midrule
\algname{Wanda}&\(O(NTd(D_Q+D_K))\)&16.88~s&47.92~s&101.98~s&833.06~s&1.00\\
\algname{QK-Wanda}&\(O(NTd(D_Q+D_K))\)&17.30~s&48.66~s&102.80~s&836.27~s&1.01\\
\bottomrule
\end{tabular*}
\end{table}

\begin{table}[H]
\centering\small
\setlength{\tabcolsep}{2pt}
\caption{Full QK pruning computation on one H200, at 50\% QK sparsity with the main \(8192\times64=524{,}288\)-token C4 calibration, batch 32. Times include calibration forwards, statistics, scoring, GPU stable sorting and mask selection, weight zeroing and pruned propagation. Loading, setup transfers, initial preparation, benchmark resets are excluded. Medians of five sums over all blocks; ratios are relative to Wanda. Complexity covers QK projections and scoring per block, excluding sorting.}
\label{tab:method_cost_h200_appendix}
\renewcommand{\arraystretch}{1.15}
\begin{tabular*}{\linewidth}{@{\extracolsep{\fill}}lcccccc@{}}
\toprule
Method&Complexity&Llama-3.2~1B&Llama-3.2~3B&Llama-3.1~8B&Llama-3.1~70B&Avg.\ ratio\\
\midrule
\algname{Wanda}&\(O(NTd(D_Q+D_K))\)&6.59~s&17.02~s&34.57~s&284.37~s&1.00\\
\algname{QK-Wanda}&\(O(NTd(D_Q+D_K))\)&7.12~s&17.45~s&34.98~s&285.98~s&1.03\\
\bottomrule
\end{tabular*}
\end{table}

\begin{table}[H]
\centering\small
\caption{Variability of full QK pruning times with the main calibration protocol, in seconds. Each interval is the minimum--maximum of five full-block sums; ratios use their medians.}
\label{tab:score_cost}
\begin{tabular}{llrrr}
\toprule
GPU&Model&Wanda range&QK-Wanda range&Ratio\\
\midrule
A100&Llama-3.2~1B&16.87--16.88&17.29--17.30&1.025\\
A100&Llama-3.2~3B&47.91--47.93&48.65--48.68&1.015\\
A100&Llama-3.1~8B&101.95--102.00&102.78--102.84&1.008\\
A100&Llama-3.1~70B&832.94--833.41&836.20--836.87&1.004\\
\midrule
H200&Llama-3.2~1B&6.58--6.59&7.11--7.13&1.081\\
H200&Llama-3.2~3B&17.01--17.04&17.43--17.47&1.025\\
H200&Llama-3.1~8B&34.56--34.64&34.97--35.05&1.012\\
H200&Llama-3.1~70B&284.21--284.41&285.82--286.25&1.006\\
\bottomrule
\end{tabular}
\end{table}

These timings describe the current implementation at the stated calibration size and sparsity. Sequence length, batching, sorting implementation, sparsity, and hardware can change the ratios. Earlier score-only, calibration-plus-scoring, smaller-calibration and CPU-sort full-pruning measurements are retained as historical source artifacts; they are not used as full-pruning times in these tables.

\FloatBarrier
\section{Wanda with Matrix-Wide Budgets}
\label{app:matrix_control}

This appendix retains the matched-budget comparison with QK-Wanda-M on the three 70B/72B checkpoints. Wanda uses a separate budget for each complete query and key projection. The corresponding QK-Wanda curves appear in Appendix~\ref{app:budget_curves}, and the main shared-budget comparison in \cref{tab:large_models}. Every projection matched the corresponding separate QK-Wanda-M run's integer deletion count. Checkpoints, tokenizers, seed 0, 524,288 calibration tokens segmented into length-64 sequences, and PPL context 2048 were fixed. H200 runs used FP16 for Llama and BF16 for Qwen. Biases and non-QK weights remained dense.

Pruning was sequential across blocks: sampled tokens matched, while later activations could differ because of earlier masks. QK error used common dense inputs and the original 16 held-out C4 sequences of length 512; each pass reproduced the published 50\% shared-mask anchor. \Cref{tab:matrix_matched_scores} compares the scores under identical budgets. QK-Wanda-M produced lower measured WikiText-2/C4 PPL in nine/nine of nine configurations on Llama~2, seven/four on Llama~3, and five/seven on Qwen; QK error was lower in all 27. At 80\% on Qwen, matrix-Wanda had lower PPL than separate QK-Wanda-M despite larger QK error; the C4 difference is hidden by rounding.

\begin{table}[H]
\centering\fontsize{8}{9}\selectfont\setlength{\tabcolsep}{5pt}
\caption{Scoring at matched matrix budgets: Wanda and causal pre-RoPE QK-Wanda-M with separate QK budgets and identical per-projection deletion counts. Lower is better. Bold marks lower displayed values within each pair, including ties. Errors are block-mean relative QK reconstruction errors (\%).}
\label{tab:matrix_matched_scores}
\begin{qktabular}{7}{cc>{\columncolor{qkrow}}cc>{\columncolor{qkrow}}cc>{\columncolor{qkrow}}c}
\toprule
 & \multicolumn{2}{c}{WikiText-2 PPL} & \multicolumn{2}{c}{C4 PPL} & \multicolumn{2}{c}{Rel. QK error (\%)}\\
\cmidrule(lr){2-3}\cmidrule(lr){4-5}\cmidrule(lr){6-7}
Sparsity & \shortstack{Wanda\\matrix} & \shortstack{QK-Wanda-M\\separate} & \shortstack{Wanda\\matrix} & \shortstack{QK-Wanda-M\\separate} & \shortstack{Wanda\\matrix} & \shortstack{QK-Wanda-M\\separate}\\
\midrule
\multicolumn{7}{c}{\textbf{Llama-2-70B}}\\
\midrule
10\% & \textbf{3.32} & \textbf{3.32} & \textbf{5.71} & \textbf{5.71} & \textbf{0.00} & \textbf{0.00}\\
20\% & \textbf{3.32} & \textbf{3.32} & \textbf{5.71} & \textbf{5.71} & 0.01 & \textbf{0.00}\\
30\% & \textbf{3.32} & \textbf{3.32} & \textbf{5.71} & \textbf{5.71} & 0.04 & \textbf{0.01}\\
40\% & 3.33 & \textbf{3.32} & 5.72 & \textbf{5.71} & 0.18 & \textbf{0.03}\\
50\% & 3.36 & \textbf{3.34} & 5.74 & \textbf{5.72} & 0.54 & \textbf{0.10}\\
60\% & 3.43 & \textbf{3.37} & 5.81 & \textbf{5.75} & 1.41 & \textbf{0.28}\\
70\% & 3.61 & \textbf{3.48} & 5.99 & \textbf{5.84} & 3.28 & \textbf{0.68}\\
80\% & 4.13 & \textbf{3.84} & 6.53 & \textbf{6.14} & 7.19 & \textbf{1.60}\\
90\% & 5.76 & \textbf{5.62} & 8.76 & \textbf{7.96} & 15.78 & \textbf{4.12}\\
\midrule
\multicolumn{7}{c}{\textbf{Llama-3.1-70B}}\\
\midrule
10\% & \textbf{2.81} & \textbf{2.81} & \textbf{7.11} & \textbf{7.11} & \textbf{0.00} & \textbf{0.00}\\
20\% & \textbf{2.81} & \textbf{2.81} & 7.12 & \textbf{7.11} & 0.03 & \textbf{0.01}\\
30\% & 2.86 & \textbf{2.82} & 7.15 & \textbf{7.13} & 0.18 & \textbf{0.03}\\
40\% & 3.00 & \textbf{2.88} & 7.28 & \textbf{7.19} & 0.66 & \textbf{0.13}\\
50\% & 3.35 & \textbf{3.06} & 7.61 & \textbf{7.41} & 1.65 & \textbf{0.42}\\
60\% & 4.14 & \textbf{3.95} & \textbf{8.46} & 9.33 & 3.28 & \textbf{1.17}\\
70\% & \textbf{6.25} & 8.89 & \textbf{10.76} & 16.90 & 5.58 & \textbf{2.83}\\
80\% & \textbf{17.69} & 26.00 & \textbf{24.21} & 43.09 & 8.85 & \textbf{6.00}\\
90\% & 208.08 & \textbf{97.00} & \textbf{120.85} & 138.71 & 14.12 & \textbf{11.05}\\
\midrule
\multicolumn{7}{c}{\textbf{Qwen2.5-72B}}\\
\midrule
10\% & \textbf{3.88} & \textbf{3.88} & \textbf{9.26} & \textbf{9.26} & \textbf{0.00} & \textbf{0.00}\\
20\% & \textbf{3.88} & \textbf{3.88} & \textbf{9.26} & \textbf{9.26} & 0.02 & \textbf{0.00}\\
30\% & 3.89 & \textbf{3.88} & \textbf{9.26} & \textbf{9.26} & 0.05 & \textbf{0.02}\\
40\% & 3.90 & \textbf{3.89} & 9.27 & \textbf{9.26} & 0.14 & \textbf{0.04}\\
50\% & 3.92 & \textbf{3.90} & 9.29 & \textbf{9.28} & 0.29 & \textbf{0.10}\\
60\% & 3.97 & \textbf{3.95} & 9.33 & \textbf{9.31} & 0.56 & \textbf{0.19}\\
70\% & \textbf{4.07} & \textbf{4.07} & 9.42 & \textbf{9.40} & 1.01 & \textbf{0.38}\\
80\% & \textbf{4.31} & 4.45 & \textbf{9.64} & \textbf{9.64} & 1.73 & \textbf{0.73}\\
90\% & \textbf{5.25} & 5.61 & \textbf{10.46} & 10.47 & 2.97 & \textbf{1.51}\\
\bottomrule
\end{qktabular}
\end{table}

\FloatBarrier
\subsection{Zero-shot results}

\Cref{tab:zero_shot_matrix_control} reports the seven full task splits using the exact saved masks at 30\%, 50\%, and 80\%. Evaluation documents and prompts matched the original comparison. Relative to row-wise Wanda, the matrix control improved the mean in three of nine configurations and worsened it in six. Shared QK-Wanda-M improved over the matrix control in six of nine, but lost on Llama~2 70B at 50\% and Qwen2.5-72B at 30\%/50\%. At 80\% on Llama-3.1-70B, the matrix control had lower PPL than shared QK-Wanda-M but much lower mean accuracy (57.74\% versus 66.45\%). Separate-budget QK-Wanda-M was not included in this task table, so the accuracy comparison does not isolate the scoring rule. All results use one calibration seed, and small differences do not establish statistical significance.

\begin{table}[H]
\centering\fontsize{8}{9}\selectfont\setlength{\tabcolsep}{2pt}
\caption{Zero-shot accuracy (\%) for the matrix-budget control and the original comparison. Wanda uses row budgets; Wanda (matrix) ranks weights within each projection; QK-Wanda-M uses shared QK budgets. Each section reuses the same dense reference and evaluation examples. Bold marks the highest displayed sparse score at each sparsity, including ties; Avg. is the unweighted seven-task mean.}
\label{tab:zero_shot_matrix_control}
\begin{qktabular}{10}{lccccccccc}
\toprule
Method & Sparsity & BoolQ & RTE & HellaSwag & WinoGrande & \shortstack{ARC-\\Easy} & \shortstack{ARC-\\Challenge} & \shortstack{OpenBook\\QA} & Avg.\\
\midrule
\multicolumn{10}{c}{\textbf{Llama-2-70B}}\\
\midrule
Dense & 0\% & 83.76 & 67.87 & 83.81 & 77.98 & 80.98 & 57.34 & 48.80 & 71.50\\
\midrule
\algname{Wanda} & 30\% & 83.43 & \textbf{68.95} & 83.72 & 77.82 & 80.72 & \textbf{57.42} & 48.40 & 71.50\\
\algname{Wanda} (matrix) & 30\% & \textbf{83.70} & 68.23 & 83.76 & \textbf{78.06} & 80.89 & 57.17 & \textbf{48.60} & 71.49\\
\algname{QK-Wanda-M} & 30\% & 83.64 & \textbf{68.95} & \textbf{83.79} & 77.90 & \textbf{81.02} & \textbf{57.42} & \textbf{48.60} & \textbf{71.62}\\
\addlinespace[2pt]
\algname{Wanda} & 50\% & 82.63 & 69.31 & 83.38 & \textbf{77.98} & 80.89 & 57.00 & 48.40 & 71.37\\
\algname{Wanda} (matrix) & 50\% & 83.49 & \textbf{71.12} & 83.36 & 77.66 & \textbf{81.06} & \textbf{57.51} & 48.40 & \textbf{71.80}\\
\algname{QK-Wanda-M} & 50\% & \textbf{83.94} & 68.95 & \textbf{83.79} & 77.90 & 81.02 & 57.17 & \textbf{48.60} & 71.62\\
\addlinespace[2pt]
\algname{Wanda} & 80\% & 76.51 & 65.70 & 76.66 & 68.03 & 78.11 & 50.34 & 45.60 & 65.85\\
\algname{Wanda} (matrix) & 80\% & 80.92 & \textbf{73.29} & 78.52 & \textbf{78.06} & 78.75 & 53.16 & 48.00 & 70.10\\
\algname{QK-Wanda-M} & 80\% & \textbf{84.04} & 70.40 & \textbf{82.54} & 77.27 & \textbf{81.44} & \textbf{56.66} & \textbf{50.00} & \textbf{71.76}\\
\midrule
\multicolumn{10}{c}{\textbf{Llama-3.1-70B}}\\
\midrule
Dense & 0\% & 85.44 & 69.68 & 85.03 & 79.95 & 86.66 & 64.93 & 47.80 & 74.21\\
\midrule
\algname{Wanda} & 30\% & 85.35 & 70.04 & \textbf{84.86} & 79.56 & 86.36 & \textbf{64.76} & \textbf{48.20} & 74.16\\
\algname{Wanda} (matrix) & 30\% & \textbf{85.50} & 70.40 & 84.60 & 79.40 & 86.41 & 64.16 & 47.20 & 73.95\\
\algname{QK-Wanda-M} & 30\% & 85.26 & \textbf{70.76} & 84.85 & \textbf{79.64} & \textbf{86.62} & 64.33 & 47.80 & \textbf{74.18}\\
\addlinespace[2pt]
\algname{Wanda} & 50\% & 84.74 & \textbf{70.04} & 84.32 & 79.01 & 85.94 & 62.97 & \textbf{48.20} & 73.60\\
\algname{Wanda} (matrix) & 50\% & 82.75 & 69.68 & 82.56 & 79.08 & 84.64 & 61.26 & 47.80 & 72.54\\
\algname{QK-Wanda-M} & 50\% & \textbf{85.11} & \textbf{70.04} & \textbf{84.54} & \textbf{79.24} & \textbf{86.20} & \textbf{63.82} & 47.60 & \textbf{73.79}\\
\addlinespace[2pt]
\algname{Wanda} & 80\% & 76.06 & 68.59 & \textbf{76.62} & \textbf{69.14} & 80.89 & 54.01 & \textbf{41.20} & \textbf{66.64}\\
\algname{Wanda} (matrix) & 80\% & 71.99 & 57.76 & 67.76 & 59.27 & 69.82 & 46.16 & 31.40 & 57.74\\
\algname{QK-Wanda-M} & 80\% & \textbf{77.80} & \textbf{71.12} & 73.85 & 64.88 & \textbf{81.61} & \textbf{55.72} & 40.20 & 66.45\\
\midrule
\multicolumn{10}{c}{\textbf{Qwen2.5-72B}}\\
\midrule
Dense & 0\% & 89.17 & 77.26 & 86.06 & 78.14 & 83.38 & 62.63 & 46.60 & 74.75\\
\midrule
\algname{Wanda} & 30\% & 89.27 & 77.26 & 86.13 & 77.51 & \textbf{83.33} & 62.46 & 46.80 & 74.68\\
\algname{Wanda} (matrix) & 30\% & \textbf{89.33} & \textbf{77.98} & \textbf{86.16} & \textbf{78.45} & 83.25 & \textbf{62.54} & \textbf{47.40} & \textbf{75.02}\\
\algname{QK-Wanda-M} & 30\% & 89.17 & 77.26 & 86.04 & 77.58 & 83.16 & 62.46 & 46.60 & 74.61\\
\addlinespace[2pt]
\algname{Wanda} & 50\% & 89.11 & \textbf{76.90} & 85.90 & \textbf{77.82} & 83.42 & \textbf{62.12} & \textbf{47.80} & \textbf{74.72}\\
\algname{Wanda} (matrix) & 50\% & 89.24 & \textbf{76.90} & \textbf{86.12} & 77.66 & 83.12 & 61.86 & 47.20 & 74.59\\
\algname{QK-Wanda-M} & 50\% & \textbf{89.42} & 76.53 & 85.96 & 77.58 & \textbf{83.84} & 61.86 & 46.60 & 74.54\\
\addlinespace[2pt]
\algname{Wanda} & 80\% & 87.55 & \textbf{77.26} & 83.90 & \textbf{76.40} & 83.25 & 60.15 & 45.60 & 73.44\\
\algname{Wanda} (matrix) & 80\% & 88.65 & 73.65 & 84.73 & 75.69 & 83.00 & 60.84 & \textbf{46.00} & 73.22\\
\algname{QK-Wanda-M} & 80\% & \textbf{89.27} & 75.09 & \textbf{85.06} & 76.24 & \textbf{85.23} & \textbf{62.29} & 44.60 & \textbf{73.97}\\
\bottomrule
\end{qktabular}
\end{table}

\FloatBarrier
\section{Centering the Causal Reconstruction Error}
\label{app:centered}

Softmax is unchanged when the same constant is added to all allowed logits of a query. This motivates testing whether the pruning objective benefits from ignoring such shifts. We retain the causal, pre-RoPE setting of \cref{eq:main_causal_loss} and denote the centered variant by \algname{QK-Wanda-MC}. Centering is performed separately for each query token, over its allowed key positions; future keys never enter the mean.

\subsection{Objective and single-deletion scores}

Fix one transformer block and one calibration sequence, with the notation of \cref{sec:causal_method}. For a query head \(h\in\mathcal H_g\), define the product error and its causal row mean by
\begin{equation}
 E^{(h)}_{tu}=\bigl(\widehat Q^{(h)\top}\widehat K^{(g)}-Q^{(h)\top}K^{(g)}\bigr)_{tu},
 \qquad \bar E^{(h)}_t=\frac1t\sum_{u=1}^{t}E^{(h)}_{tu}.
\end{equation}
Here \(E^{(h)}\in\R^{T\times T}\); rows index query positions and columns index key positions. The centered loss is
\begin{equation}
 \mathcal L_{\mathrm{cent}}
 =\sum_{g=1}^{H_K}\sum_{h\in\mathcal H_g}\sum_{t=1}^{T}\sum_{u=1}^{t}
 \bigl(E^{(h)}_{tu}-\bar E^{(h)}_t\bigr)^2.
 \label{eq:centered_loss}
\end{equation}
Equivalently, each row's squared error is minimized over an arbitrary constant offset on its allowed positions. For a token vector \(v=(v_1,\ldots,v_T)\), write
\begin{equation}
 \nu_t(v)=\sum_{u=1}^{t}\left(v_u-\frac1t\sum_{r=1}^{t}v_r\right)^2
 =\sum_{u=1}^{t}v_u^2-\frac1t\left(\sum_{u=1}^{t}v_u\right)^2
 \label{eq:centered_prefix}
\end{equation}
for its centered prefix energy. Deleting one scalar weight while keeping every other weight and bias fixed gives
\begin{align}
 \hat S^{Q^{(h)},\mathrm{cent}}_{ij}
 &=\bigl(W^{Q^{(h)}}_{ij}\bigr)^2
   \sum_{t=1}^{T}X_{jt}^2\,\nu_t\!\left(K^{(g)}_{i:}\right), \qquad h\in\mathcal H_g,
 \label{eq:centered_query_score}\\
 \hat S^{K^{(g)},\mathrm{cent}}_{ij}
 &=\bigl(W^{K^{(g)}}_{ij}\bigr)^2
   \sum_{t=1}^{T}\left[\sum_{h\in\mathcal H_g}\bigl(Q^{(h)}_{it}\bigr)^2\right]
   \nu_t\!\left(X_{j:}\right).
 \label{eq:centered_key_score}
\end{align}
For a query deletion, the affected error row is \(-W^{Q^{(h)}}_{ij}X_{jt}K^{(g)}_{i,1:t}\), so centering removes the mean of the key prefix. For a key deletion, its entries are \(-W^{K^{(g)}}_{ij}Q^{(h)}_{it}X_{ju}\), so centering removes the mean of the input prefix and the cost sums over all query heads sharing that key. These identities establish the exact single-deletion costs in \cref{eq:centered_query_score,eq:centered_key_score}.

We average the scores over independent calibration sequences and select both masks under the same shared QK budget as the main method. Prefix sums of \(v\) and \(v^2\) suffice to evaluate \(\nu_t\). The implementation first subtracts \(v_1\), which preserves every prefix variance and reduces cancellation, and clamps small negative numerical residuals to zero. As in the original implementation, both scores are divided by the common head width \(m\); this does not change their ordering. The causal prefix of length one has zero centered energy.

\subsection{Llama 3 results}

The comparison used the same checkpoints, calibration tokens, seed, sparsities, and evaluation data as the original experiments. Centered pruning and evaluation ran on NVIDIA H200 GPUs. \Cref{tab:main_results,tab:zero_shot_llama_all} add the centered results for Llama-3.2-1B/3B and Llama-3.1-8B. Their PPL and reconstruction curves appear in Appendices~\ref{app:budget_curves} and~\ref{app:qk_error_curves}. The reconstruction columns use the common QK diagnostic in \cref{eq:relative_qk_error}, rather than the centered scoring loss.

At 30\%, 50\%, and 80\% sparsity, centering increased both PPL metrics and the common full QK error on all three checkpoints relative to shared-budget \algname{QK-Wanda-M}. For example, at 50\% on 1B, WikiText-2/C4 PPL increased from 10.65/16.24 to 16.46/28.93. Zero-shot changes were mixed at 30\%: mean accuracy increased from 54.28\% to 54.43\% on 1B and from 69.40\% to 69.60\% on 8B, but decreased from 61.55\% to 61.35\% on 3B. At 50\% and 80\%, all three checkpoints had lower mean accuracy with centering.

Thus, removing row-constant error from this proxy did not improve the tabulated PPL results. Softmax invariance alone does not guarantee a better pruning criterion: the proxy remains pre-RoPE and describes individual deletions, while the final mask removes many weights jointly. Within this causal comparison, centering alone did not provide a PPL advantage on these Llama~3 checkpoints.

\subsection{TinyLlama, Llama 2, and Qwen2.5 results}

We repeated the centered comparison on TinyLlama-1.1B, Llama-2-7B/13B, and Qwen2.5-0.5B/1.5B/3B/7B, preserving each checkpoint's original calibration and evaluation protocol. \Cref{tab:llama2_results,tab:qwen25_results} include PPL, mean accuracy, and reconstruction error at 30\%, 50\%, and 80\% sparsity; \cref{tab:zero_shot_llama_all,tab:zero_shot_qwen25_all} give all seven task scores. The corresponding appendix curves cover 0--90\% sparsity, reusing the dense references. The reconstruction metric is the same as above.

Compared with shared-budget \algname{QK-Wanda-M}, centering increased both PPL metrics at every nonzero sparsity on TinyLlama and Llama-2-13B. Outcomes on Qwen2.5 were mixed. At 80\% on 0.5B, WikiText-2/C4 PPL improved from 15.07/23.07 to 14.73/22.76, and mean accuracy increased from 47.01\% to 47.37\%. On 7B, WikiText-2 PPL improved at all nine nonzero sparsities, often only slightly, whereas C4 improved at three. Across these seven checkpoints and nine sparsities, centering lowered WikiText-2 PPL in 19 of 63 comparisons and C4 PPL in 8 of 63.

Mean zero-shot accuracy improved in 9 of the 21 tabulated comparisons and declined in 12. For example, Llama-2-13B had slightly higher accuracy at all three tabulated sparsities despite higher PPL; Qwen2.5-3B at 80\% improved from 61.50\% to 62.76\% accuracy while C4 PPL worsened from 16.08 to 16.51. The common full QK error increased in 62 of 63 comparisons. Thus centering sometimes helped downstream metrics, but did not provide a consistent improvement across checkpoints, sparsities, or datasets. These are single-calibration-seed comparisons.

\FloatBarrier
\subsection{Centering after rotary position embeddings}
\label{app:rope_centered}

The pre-RoPE experiment does not center the logits that enter softmax. We therefore also evaluated \algname{QK-Wanda-MRC}, which centers the causal \emph{post-RoPE} product error. Fix one block and one calibration sequence, suppressing its sequence index. We use the rotary maps and token columns from Appendix~\ref{app:rope}: \(\widetilde q_t^{(h)}=R_tq_t^{(h)}\), \(\widetilde k_u^{(g)}=R_uk_u^{(g)}\), and \(Z_{tu}^{(h)}=\widetilde q_t^{(h)\top}\widetilde k_u^{(g)}\) for \(h\in\mathcal H_g\). Hatted products use the pruned weights with the same inputs, retained biases, and rotary maps. Define
\begin{equation}
 E^{(h)}_{tu,R}=\widehat Z^{(h)}_{tu}-Z^{(h)}_{tu},
 \qquad \bar E^{(h)}_{t,R}=\frac1t\sum_{u=1}^{t}E^{(h)}_{tu,R}.
\end{equation}
The loss for this sequence is
\begin{equation}
 \mathcal L_{R,\mathrm{cent}}
 =\sum_{g=1}^{H_K}\sum_{h\in\mathcal H_g}\sum_{t=1}^{T}\sum_{u=1}^{t}
 \bigl(E^{(h)}_{tu,R}-\bar E^{(h)}_{t,R}\bigr)^2.
 \label{eq:rope_centered_loss}
\end{equation}
Only allowed keys enter each mean. Unlike centering the pre-RoPE product, this objective removes row-constant changes from the actual rotary-transformed logits, up to the common attention scale.

For a sequence of column vectors \(v=(v_1,\ldots,v_T)\), \(v_u\in\R^{m\times1}\), define its unnormalized prefix covariance
\begin{equation}
 \bar v_t=\frac1t\sum_{u=1}^{t}v_u,
 \qquad
 \mathcal V_t(v)=\sum_{u=1}^{t}(v_u-\bar v_t)(v_u-\bar v_t)^\top
 \in\R^{m\times m}.
 \label{eq:rope_centered_cov}
\end{equation}
For a candidate weight at coordinates \((i,j)\), let \(e_i\) be the coordinate basis vector and set \(y_u=X_{ju}R_ue_i\). With all other weights and biases fixed, the exact single-deletion costs are
\begin{align}
 \hat S^{Q^{(h)},\mathrm{cent}}_{ij,R}
 &=\bigl(W^{Q^{(h)}}_{ij}\bigr)^2
   \sum_{t=1}^{T}X_{jt}^2(R_te_i)^\top
   \mathcal V_t\!\left(\widetilde k^{(g)}\right)(R_te_i),
 \qquad h\in\mathcal H_g,
 \label{eq:rope_centered_query}\\
 \hat S^{K^{(g)},\mathrm{cent}}_{ij,R}
 &=\bigl(W^{K^{(g)}}_{ij}\bigr)^2
   \sum_{t=1}^{T}\sum_{h\in\mathcal H_g}
   \widetilde q_t^{(h)\top}\mathcal V_t(y)\widetilde q_t^{(h)}.
 \label{eq:rope_centered_key}
\end{align}
Here \(\widetilde k^{(g)}\) denotes the sequence of rotated key columns; \(y\) depends on the candidate coordinates \((i,j)\).

To obtain \cref{eq:rope_centered_query}, a query deletion gives \(E^{(h)}_{tu,R}=-W^{Q^{(h)}}_{ij}X_{jt}(R_te_i)^\top\widetilde k_u^{(g)}\). For fixed \(t\), subtracting its mean replaces \(\widetilde k_u^{(g)}\) by its deviation from the prefix mean. Squaring and summing over \(u\le t\) gives the covariance contraction. A key deletion gives \(E^{(h)}_{tu,R}=-W^{K^{(g)}}_{ij}\widetilde q_t^{(h)\top}y_u\); the same argument centers the sequence \(y\), then sums over all query heads sharing the key. Centering rotated keys or inputs alone would not give this key-deletion cost, because the direction \(R_ue_i\) changes with position.

We average these costs over calibration sequences and apply the same shared QK budget, with the common \(1/m\) scaling used by the other implementations. Setting all rotary maps to the identity recovers \cref{eq:centered_query_score,eq:centered_key_score}. For RoPE, each \(R_te_i\) is supported on one rotary coordinate pair, so only \(2\times2\) prefix covariances are needed in these contractions. Prefix first and second moments compute them without forming a \(T\times T\) logit matrix. We use the model's actual rotary maps, including its frequency scaling, and the same stabilized prefix accumulation as above. Independent tests compare every scalar score against explicit deletion and centering of the full causal product error.

\FloatBarrier
\subsection{Post-RoPE centered results on Llama 3}

We evaluated \algname{QK-Wanda-MRC} on the same Llama-3.2-1B/3B and Llama-3.1-8B checkpoints, calibration sequences, sparsity grid, and evaluation data. \Cref{tab:main_results,tab:zero_shot_llama_all} report the summary and seven-task results; Appendices~\ref{app:budget_curves} and~\ref{app:qk_error_curves} include its full curves. The reconstruction columns again use the common QK diagnostic, rather than the new centered objective.

At 80\% sparsity, the combined RoPE-centered variant improved both PPL metrics and mean zero-shot accuracy over ordinary shared-budget \algname{QK-Wanda-M} on 3B and 8B. On 3B, WikiText-2/C4 PPL decreased from 42.70/44.30 to 23.18/33.50, and accuracy increased from 50.16\% to 52.88\%. On 8B, PPL decreased from 13.81/22.36 to 12.51/19.95, and accuracy increased from 60.40\% to 61.81\%. Both checkpoints still had higher PPL than original row-wise \algname{Wanda} at this sparsity. Compared with pre-RoPE centering, the 3B variant improved both PPL metrics at all nine nonzero sparsities.

The gains were not uniform. At 30\% and 50\%, RoPE-centered PPL was higher than ordinary \algname{QK-Wanda-M} on both datasets for all three checkpoints. The 1B regression at 50\% was substantial: WikiText-2/C4 PPL rose from 10.65/16.24 to 22.63/46.81, and mean accuracy fell from 53.51\% to 51.07\%. Across the three checkpoints and nine nonzero sparsities, RoPE-centering lowered WikiText-2 PPL in 8 of 27 comparisons, C4 PPL in 7 of 27, and the common full QK error in 4 of 27. Mean zero-shot accuracy improved in 4 of the 9 tabulated comparisons. Thus the combination recovered some high-sparsity quality on 3B/8B, but these results do not show a consistent advantage over QK-Wanda-M across the tested settings.

\subsection{Post-RoPE centering across all model families}
\label{app:rope_centered_all}

We extended the same post-RoPE centered score to all 15 checkpoints, including Llama-2-70B, Llama-3.1-70B, and Qwen2.5-72B. The existing appendix summary and task tables include these measurements, and the PPL and QK-error figures include the full sparsity sweeps. Checkpoint revisions, calibration tokens, precision, evaluation data, and shared QK budgets match the original controls. QK error follows the same diagnostic in \cref{eq:relative_qk_error}.

The downstream effect depends strongly on sparsity. Relative to the ordinary causal score, post-RoPE centering lowered WikiText-2/C4 PPL on 12/10 of the 15 checkpoints at 80\% sparsity, but on only 3/0 at 50\% and 2/1 at 30\%. Across all 135 nonzero model--sparsity pairs, the corresponding counts were 53 and 41. These counts use unrounded values and include small differences.

The largest models illustrate the mixed outcome. At 80\% sparsity, Llama-3.1-70B improved from 18.98/32.37 to 13.67/22.61 WikiText-2/C4 PPL. Llama-2-70B instead worsened from 3.63/5.97 to 3.85/6.27. Qwen2.5-72B improved WikiText-2 from 4.244 to 4.186 while C4 changed from 9.523 to 9.527. Thus the combined variant can recover high-sparsity quality, especially on Llama~3, but is not uniformly better. Appendix~\ref{app:factorial_selection} separates the contributions of masking, rotation, and centering and compares them with QK-Wanda.

\FloatBarrier
\section{Choice of scoring variant}
\label{app:factorial_selection}

We compare all eight combinations of three choices in the reconstruction objective: causal versus unrestricted query--key pairs, pre- versus post-RoPE products, and uncentered versus centered error. Centering subtracts the mean error over the eligible key positions separately for each query. With RoPE, this operation is applied after rotation. The switches are M (causal mask), R (RoPE), and C (centering), with the canonical method names listed in the tables. These choices change the pruning scores; the model's attention computation is unchanged.

Appendices~\ref{app:causal_details},~\ref{app:rope}, and~\ref{app:centered} derive the masked variants. Their unrestricted counterparts include all key positions: set the pair mask to one in the RoPE formulas, and replace each causal prefix by the full key sequence in the centered formulas. This gives R, C, and RC under the same individual-deletion principle as QK-Wanda.

The initial comparison uses Llama-2-7B, Llama-3.1-8B, and Qwen2.5-7B at 50\% and 80\% sparsity. All variants share the same QK budget, checkpoint revision, calibration tokens, numerical precision, and WikiText-2/C4 evaluation protocol. All eight variants are complete. \Cref{tab:factorial_ppl} reports PPL and seven-task mean accuracy.

\begin{table}[H]
\centering\small
\setlength{\tabcolsep}{3pt}
\caption{Eight scoring combinations on 7--8B models: causal mask (M), RoPE (R), and centering (C). QK variants share a QK budget; Wanda uses row-wise budgets. Dense (0\%) is repeated as a reference. Avg. is the mean seven-task zero-shot accuracy (\%). Bold marks the lowest available PPL among QK variants (displayed ties included). Attention itself is unchanged.}
\label{tab:factorial_ppl}
\begin{qktabular}{10}{lcccrrrrrr}
\toprule
& & & & \multicolumn{3}{c}{50\% sparsity} & \multicolumn{3}{c}{80\% sparsity}\\
\cmidrule(lr){5-7}\cmidrule(lr){8-10}
Method / reference & M & R & C & WikiText-2 & C4 & Avg. & WikiText-2 & C4 & Avg.\\
\midrule
\multicolumn{10}{l}{\textbf{Llama-2-7B}}\\
Dense & \multicolumn{3}{c}{---} & 5.472 & 7.264 & 64.38 & 5.472 & 7.264 & 64.38\\
Wanda & \multicolumn{3}{c}{---} & 5.618 & 7.419 & 63.43 & 12.540 & 13.940 & 52.56\\
\cmidrule(lr){1-10}
QK-Wanda & 0 & 0 & 0 & \textbf{5.529} & \textbf{7.357} & 63.98 & 8.313 & 11.197 & 57.84\\
QK-Wanda-M & 1 & 0 & 0 & 5.530 & 7.358 & 63.92 & 8.328 & 11.213 & 57.81\\
QK-Wanda-C & 0 & 0 & 1 & 5.679 & 7.478 & 63.95 & 9.710 & 12.985 & 57.81\\
QK-Wanda-MC & 1 & 0 & 1 & 5.636 & 7.438 & 63.98 & 9.722 & 13.007 & 57.88\\
QK-Wanda-R & 0 & 1 & 0 & 5.535 & 7.360 & 64.15 & 8.241 & 11.033 & 57.81\\
QK-Wanda-MR & 1 & 1 & 0 & 5.537 & 7.365 & 64.02 & 8.391 & 11.269 & 57.82\\
QK-Wanda-RC & 0 & 1 & 1 & 5.601 & 7.413 & 63.80 & \textbf{8.038} & \textbf{10.821} & 58.82\\
QK-Wanda-MRC & 1 & 1 & 1 & 5.579 & 7.393 & 64.00 & 8.254 & 11.112 & 58.81\\
\midrule
\multicolumn{10}{l}{\textbf{Llama-3.1-8B}}\\
Dense & \multicolumn{3}{c}{---} & 6.238 & 9.538 & 69.10 & 6.238 & 9.538 & 69.10\\
Wanda & \multicolumn{3}{c}{---} & 6.404 & 9.853 & 68.43 & 8.890 & 13.540 & 59.64\\
\cmidrule(lr){1-10}
QK-Wanda & 0 & 0 & 0 & \textbf{6.405} & \textbf{9.816} & 69.42 & 13.470 & 21.967 & 60.61\\
QK-Wanda-M & 1 & 0 & 0 & \textbf{6.405} & \textbf{9.816} & 69.39 & 13.814 & 22.361 & 60.40\\
QK-Wanda-C & 0 & 0 & 1 & 6.686 & 10.241 & 69.25 & 27.932 & 40.953 & 58.98\\
QK-Wanda-MC & 1 & 0 & 1 & 6.656 & 10.212 & 69.01 & 23.440 & 36.955 & 59.13\\
QK-Wanda-R & 0 & 1 & 0 & 6.431 & 9.858 & 69.45 & 16.533 & 27.372 & 59.77\\
QK-Wanda-MR & 1 & 1 & 0 & 6.433 & 9.864 & 69.27 & 15.938 & 26.736 & 59.85\\
QK-Wanda-RC & 0 & 1 & 1 & 6.580 & 10.124 & 69.33 & 13.327 & 21.250 & 62.23\\
QK-Wanda-MRC & 1 & 1 & 1 & 6.485 & 9.972 & 69.61 & \textbf{12.515} & \textbf{19.954} & 61.81\\
\midrule
\multicolumn{10}{l}{\textbf{Qwen2.5-7B}}\\
Dense & \multicolumn{3}{c}{---} & 6.849 & 11.881 & 70.56 & 6.849 & 11.881 & 70.56\\
Wanda & \multicolumn{3}{c}{---} & 6.901 & 11.933 & 70.03 & 7.395 & 12.531 & 67.36\\
\cmidrule(lr){1-10}
QK-Wanda & 0 & 0 & 0 & 6.886 & 11.922 & 70.07 & 7.323 & 12.617 & 67.98\\
QK-Wanda-M & 1 & 0 & 0 & 6.886 & 11.922 & 70.02 & 7.320 & 12.619 & 68.03\\
QK-Wanda-C & 0 & 0 & 1 & 6.881 & 11.936 & 69.78 & 7.316 & 12.611 & 67.84\\
QK-Wanda-MC & 1 & 0 & 1 & \textbf{6.879} & 11.931 & 69.78 & 7.298 & 12.579 & 67.96\\
QK-Wanda-R & 0 & 1 & 0 & 6.883 & \textbf{11.921} & 70.15 & 7.307 & 12.587 & 67.69\\
QK-Wanda-MR & 1 & 1 & 0 & 6.883 & 11.922 & 70.13 & 7.313 & 12.605 & 67.94\\
QK-Wanda-RC & 0 & 1 & 1 & 6.884 & 11.926 & 70.00 & 7.274 & 12.493 & 68.26\\
QK-Wanda-MRC & 1 & 1 & 1 & 6.882 & 11.923 & 70.00 & \textbf{7.262} & \textbf{12.480} & 68.46\\
\bottomrule
\end{qktabular}
\end{table}

At 80\% sparsity, \algname{QK-Wanda-RC}, post-RoPE centering without a causal mask, achieves the lowest PPL among QK variants on both datasets for Llama-2-7B (8.038/10.821). For Llama-3.1-8B and Qwen2.5-7B, the lowest values instead come from \algname{QK-Wanda-MRC}, post-RoPE centering with a causal mask (12.515/19.954 and 7.262/12.480, respectively). Thus the preferred masking choice depends on the checkpoint even when rotation and centering are fixed.

The ranking changes at 50\% sparsity: QK-Wanda gives the lowest PPL among QK variants on both datasets for Llama-2-7B, while the differences among several Qwen variants are small. Pre-RoPE centering substantially worsens the 80\% results on both Llama checkpoints. These comparisons show that the three choices interact, rather than providing independent, uniformly beneficial corrections.

% BEGIN SCALE UPDATE
\paragraph{Extension to model scale and zero-shot evaluation.}
The extended comparison adds Llama-2-70B, Llama-3.1-70B, and Qwen2.5-72B at the same two sparsities and evaluates all six checkpoints on the same seven zero-shot tasks. All 96 PPL configurations and 96 zero-shot configurations are independently validated, including compatible reused measurements. \Cref{tab:factorial_ppl,tab:factorial_large} report all results. The lowest-PPL and highest zero-shot variants differ across checkpoints and sparsity levels, so no single combination dominates this screen. We retain QK-Wanda as the main method; the extensions expose checkpoint-specific tradeoffs rather than a uniformly superior replacement.

\begin{table}[H]
\centering\small
\setlength{\tabcolsep}{3pt}
\caption{Eight scoring combinations on 70--72B models: causal mask (M), RoPE (R), and centering (C). QK variants share a QK budget; Wanda uses row-wise budgets. Dense (0\%) is repeated as a reference. Avg. is the mean seven-task zero-shot accuracy (\%). Bold marks the lowest available PPL among QK variants (displayed ties included). Attention itself is unchanged.}
\label{tab:factorial_large}
\begin{qktabular}{10}{lcccrrrrrr}
\toprule
& & & & \multicolumn{3}{c}{50\% sparsity} & \multicolumn{3}{c}{80\% sparsity}\\
\cmidrule(lr){5-7}\cmidrule(lr){8-10}
Method / reference & M & R & C & WikiText-2 & C4 & Avg. & WikiText-2 & C4 & Avg.\\
\midrule
\multicolumn{10}{l}{\textbf{Llama-2-70B}}\\
Dense & \multicolumn{3}{c}{---} & 3.319 & 5.709 & 71.50 & 3.319 & 5.709 & 71.50\\
Wanda & \multicolumn{3}{c}{---} & 3.359 & 5.742 & 71.37 & 4.550 & 6.906 & 65.85\\
\cmidrule(lr){1-10}
QK-Wanda & 0 & 0 & 0 & \textbf{3.332} & \textbf{5.716} & 71.60 & \textbf{3.625} & \textbf{5.971} & 71.79\\
QK-Wanda-M & 1 & 0 & 0 & \textbf{3.332} & \textbf{5.716} & 71.62 & 3.629 & 5.972 & 71.76\\
QK-Wanda-C & 0 & 0 & 1 & 3.349 & 5.730 & 72.20 & 4.067 & 6.509 & 69.85\\
QK-Wanda-MC & 1 & 0 & 1 & 3.343 & 5.725 & 71.98 & 3.918 & 6.326 & 71.07\\
QK-Wanda-R & 0 & 1 & 0 & \textbf{3.332} & \textbf{5.716} & 71.69 & 3.635 & 5.973 & 71.80\\
QK-Wanda-MR & 1 & 1 & 0 & 3.333 & 5.717 & 71.76 & 3.657 & 5.994 & 71.99\\
QK-Wanda-RC & 0 & 1 & 1 & 3.359 & 5.738 & 72.26 & 4.021 & 6.484 & 69.65\\
QK-Wanda-MRC & 1 & 1 & 1 & 3.345 & 5.727 & 71.99 & 3.851 & 6.273 & 70.84\\
\midrule
\multicolumn{10}{l}{\textbf{Llama-3.1-70B}}\\
Dense & \multicolumn{3}{c}{---} & 2.807 & 7.110 & 74.21 & 2.807 & 7.110 & 74.21\\
Wanda & \multicolumn{3}{c}{---} & 2.901 & 7.224 & 73.60 & 4.981 & 9.893 & 66.64\\
\cmidrule(lr){1-10}
QK-Wanda & 0 & 0 & 0 & \textbf{2.968} & \textbf{7.268} & 73.77 & 18.890 & 32.589 & 65.93\\
QK-Wanda-M & 1 & 0 & 0 & 2.972 & 7.272 & 73.79 & 18.984 & 32.372 & 66.45\\
QK-Wanda-C & 0 & 0 & 1 & 3.356 & 7.786 & 73.43 & 21.631 & 38.091 & 59.30\\
QK-Wanda-MC & 1 & 0 & 1 & 3.241 & 7.672 & 73.68 & 21.567 & 39.547 & 61.56\\
QK-Wanda-R & 0 & 1 & 0 & 2.980 & 7.286 & 73.81 & 20.670 & 36.702 & 65.92\\
QK-Wanda-MR & 1 & 1 & 0 & 2.987 & 7.301 & 73.96 & 22.056 & 38.278 & 65.54\\
QK-Wanda-RC & 0 & 1 & 1 & 3.283 & 7.595 & 72.92 & \textbf{12.333} & \textbf{19.618} & 58.36\\
QK-Wanda-MRC & 1 & 1 & 1 & 3.140 & 7.471 & 73.66 & 13.673 & 22.614 & 62.61\\
\midrule
\multicolumn{10}{l}{\textbf{Qwen2.5-72B}}\\
Dense & \multicolumn{3}{c}{---} & 3.875 & 9.256 & 74.75 & 3.875 & 9.256 & 74.75\\
Wanda & \multicolumn{3}{c}{---} & 3.920 & 9.283 & 74.72 & 4.550 & 9.873 & 73.44\\
\cmidrule(lr){1-10}
QK-Wanda & 0 & 0 & 0 & 3.895 & 9.271 & 74.44 & 4.234 & 9.520 & 74.00\\
QK-Wanda-M & 1 & 0 & 0 & 3.895 & \textbf{9.270} & 74.54 & 4.244 & 9.523 & 73.97\\
QK-Wanda-C & 0 & 0 & 1 & 3.903 & 9.281 & 74.63 & 4.262 & 9.593 & 74.02\\
QK-Wanda-MC & 1 & 0 & 1 & 3.901 & 9.279 & 74.49 & 4.275 & 9.586 & 74.17\\
QK-Wanda-R & 0 & 1 & 0 & \textbf{3.893} & \textbf{9.270} & 74.50 & 4.211 & \textbf{9.512} & 74.19\\
QK-Wanda-MR & 1 & 1 & 0 & 3.894 & \textbf{9.270} & 74.47 & 4.225 & 9.521 & 74.20\\
QK-Wanda-RC & 0 & 1 & 1 & 3.899 & 9.278 & 74.36 & 4.187 & 9.533 & 74.13\\
QK-Wanda-MRC & 1 & 1 & 1 & 3.897 & 9.277 & 74.53 & \textbf{4.186} & 9.527 & 74.08\\
\bottomrule
\end{qktabular}
\end{table}

\end{document}